\documentclass[letterpaper]{article} 
\usepackage{aaai2026}  
\usepackage{times}  
\usepackage{helvet}  
\usepackage{courier}  
\usepackage[hyphens]{url}  
\usepackage{graphicx} 
\usepackage{natbib}  
\usepackage{caption} 
\usepackage{algorithm}
\usepackage{algorithmic}

\usepackage{newfloat}
\usepackage{listings}
\DeclareCaptionStyle{ruled}{labelfont=normalfont,labelsep=colon,strut=off} 
\floatstyle{ruled}
\newfloat{listing}{tb}{lst}{}
\floatname{listing}{Listing}
	\usepackage[bibliography=common]{apxproof-old}
	\renewcommand{\appendixbibliographystyle}{aaai2026}
	\newtheoremrep{example}{Example}
	\newtheoremrep{theorem}{Theorem}
	\newtheoremrep{proposition}{Proposition}
	\newtheoremrep{lemma}{Lemma}
	\newtheoremrep{claim}{Claim}
	\newtheoremrep{corollary}{Corollary}
	\newtheoremrep{definition}{Definition}
	\newtheoremrep{remark}{Remark}
	\newtheoremrep{observation}{Observation}
	
	\usepackage{amsfonts}
	\usepackage{xspace}
	\usepackage{pifont}
	\usepackage{color}
	\usepackage{amsmath, upgreek, stmaryrd}
	\usepackage{amssymb}
	\usepackage{xargs}
	\usepackage[pdftex,dvipsnames]{xcolor}
	\usepackage{graphicx}
	\usepackage{mathtools}
	\usepackage[capitalise]{cleveref}
	\usepackage{enumitem}
	\setlist{leftmargin=2.3ex}
	\usepackage{multirow}
	\usepackage{booktabs}

	\usepackage{mathdots}
	\usepackage{tikz}
	\usetikzlibrary{tikzmark}
	\usetikzlibrary{calc}
	\usetikzlibrary{positioning,shapes.geometric}
	\usetikzlibrary{patterns}
	
	\newcommand{\mi}[1]{\ensuremath{\mathit{#1}}}

	\newcommand{\Amc}{\ensuremath{\mathcal{A}}}
	\newcommand{\Emc}{\ensuremath{\mathcal{E}}}
	\newcommand{\Gmc}{\ensuremath{\mathcal{G}}}
	\newcommand{\Hmc}{\ensuremath{\mathcal{H}}}
	\newcommand{\Imc}{\ensuremath{\mathcal{I}}}
	\newcommand{\Jmc}{\ensuremath{\mathcal{J}}}
	\newcommand{\Kmc}{\ensuremath{\mathcal{K}}}
	\newcommand{\Mmc}{\ensuremath{\mathcal{M}}}
	\newcommand{\Omc}{\ensuremath{\mathcal{O}}}
	\newcommand{\Rmc}{\ensuremath{\mathcal{R}}}
	\newcommand{\Smc}{\ensuremath{\mathcal{S}}}
	\newcommand{\Tmc}{\ensuremath{\mathcal{T}}}
	\newcommand{\Vmc}{\ensuremath{\mathcal{V}}}

	\newcommand{\tbox}{\Tmc}
	\newcommand{\abox}{\Amc}
	\newcommand{\kb}{\Kmc}
	\newcommand{\cost}{\ensuremath{\omega}}

	\newcommand{\dllite}{\textup{DL-Lite}}
	\newcommand{\dllitebool}{\ensuremath{\dllite_\textup{bool}}}
	
	\newcommand{\dlliteboolh}{\ensuremath{\dllitebool^\Hmc}}
	\newcommand{\dlliter}{\ensuremath{\dllite_\Rmc}}
	\newcommand{\dllitecore}{\ensuremath{\dllite_\textup{core}}}
	\newcommand{\dllitecoreh}{\ensuremath{\dllitecore^\Hmc}}
	
	\newcommand{\ALC}{\ensuremath{\mathcal{ALC}}}
	\newcommand{\ALCO}{\ensuremath{\mathcal{ALCO}}}
	\newcommand{\ALCI}{\ensuremath{\mathcal{ALCI}}}
	\newcommand{\ALCIO}{\ensuremath{\mathcal{ALCIO}}}
	\newcommand{\ALCHI}{\ensuremath{\mathcal{ALCHI}}}
	\newcommand{\ALCHIO}{\ensuremath{\mathcal{ALCHIO}}}
	\newcommand{\EL}{\ensuremath{\mathcal{EL}}}
	\newcommand{\ELbot}{\ensuremath{\mathcal{EL}_\bot}}

	\newcommand{\NP}{\ensuremath{\textup{NP}}}
	\newcommand{\coNP}{\ensuremath{\textup{co}\NP}}
	\newcommand{\aczero}{\ensuremath{\textup{AC}^0}}
	\newcommand{\tczero}{\aczero}
	
	\def\np{\NP}
	\def\conp{\coNP}

	\def\deltaptwo{\ensuremath{\Delta^{p}_{2}}}
	
	\newcommand{\fo}{\text{FO}}
	
	\newcommand{\tsat}{\ensuremath{\text{3-\textsc{Sat}}}}
	\newcommand{\tcol}{\ensuremath{\text{3-\textsc{Col}}}}
	
	\newcommand{\sizeof}[1]{\ensuremath{{\left|#1\right|}}}
	\newcommand{\istyle}[1]{\ensuremath{\mathsf{#1}}}
	\newcommand{\cstyle}[1]{\ensuremath{\mathsf{#1}}}
	\newcommand{\rstyle}[1]{\ensuremath{\mathsf{#1}}}
	\newcommand{\type}{\ensuremath{\mathsf{tp}}}
	
	\newcommand{\WKB}{\K_\omega = ( \T,\A)_\omega}
	\newcommand{\wkb}{\K_\omega}
	
	\newcommand{\A}{\mathcal{A}}
	\newcommand{\I}{\mathcal{I}}
	\newcommand{\J}{\mathcal{J}}
	\newcommand{\T}{\mathcal{T}}
	\newcommand{\K}{\mathcal{K}}

	\newcommand{\NC}{\ensuremath{\mathsf{N_{C}}\xspace}}
	\newcommand{\NI}{\ensuremath{\mathsf{N_{I}}\xspace}}
	\newcommand{\NR}{\ensuremath{\mathsf{N_{R}}\xspace}}
	\newcommand{\NRpm}{\ensuremath{\mathsf{N_{R}^\pm}\xspace}}
	\newcommand{\existsmany}{\ensuremath{\exists^{> {k}}}}
	\newcommand{\individuals}{\ensuremath{\mathsf{Ind}}}
	\newcommand{\raretypes}{\ensuremath{\mathsf{RT}}}
	\newcommand{\precore}{\ensuremath{\mathsf{pc}}}
	\newcommand{\core}{\ensuremath{\mathsf{core}}}
	\newcommand{\optc}[1]{\mi{optc}(#1)}
	\newcommand{\vio}[2]{vio_{#1}(#2)}
	\newcommand{\sat}[2]{\models_{#1}^{#2}}
	\newcommand{\dec}[3]{#1QA$^{#2}_{#3}$\xspace}
	\newcommand{\BCS}{BCS\xspace}
	\newcommand{\tuplev}{\mathbf{v}_\Gamma}
	
	\newcommand{\bsem}[1]{$k$-cost bounded #1 semantics}
	\newcommand{\optsem}[1]{opt-cost #1 semantics}

	\newcommand{\Idag}{{\I^\dagger}}

	\newcommand{\Knew}{\K^\dagger}
	\newcommand{\Anew}{\A^\dagger}
	\newcommand{\Tnew}{\T^\dagger}
	
	\newcommand{\signature}{\mathsf{sig}}	
	
	\newcommand{\signatureof}[1][\tbox]{\signature(#1)}
	
	\newcommand{\axiom}[2]{\rolestyle{#1} \sqsubseteq \rolestyle{#2}}

	\newcommand{\axand}{\ensuremath{\axiom{A_1 \sqcap A_2}{A}}}
	
	\newcommand{\axexistsleft}{\ensuremath{\axiom{\existsrole{r.B}}{A}}}
	
	\newcommand{\axexistsright}{\ensuremath{\axiom{A}{\existsrole{r.B}}}}
	
	\newcommand{\axnotleft}{\ensuremath{\axiom{\lnot B}{A}}}
	\newcommand{\axnotright}{\ensuremath{\axiom{A}{\lnot B}}}

	\newcommand{\unfolding}[1][\kb]{\circ}
	\newcommand{\unfoldingof}[2][\kb]{\unfolding[#1]}
	
	\newcommand{\deltastar}{\Delta^*}
	
	\newcommand{\domain}[1]{\Delta^{#1}} 
	
	\newcommand{\rolestyle}[1]{\rstyle{#1}}
	
	\newcommand{\existsrole}[1]{\exists \rolestyle{#1}}

	\newcommand{\successor}[1][\kb]{\mathsf{succ}^{#1}}
	
	\newcommand{\successorof}[3][\kb]{\successor[#1]_\rolestyle{#3}(#2)}
	
	\newcommand{\inds}{\individuals}
	\newcommand{\indsof}[1]{\inds(#1)}
	
	\newcommand{\exexof}[1]{\exex}
	\newcommand{\exex}{\Delta^{\circ}}
	\newcommand{\interlace}{f^*}
	\newcommand{\interlaceof}[1]{\interlace(#1)}
	\newcommand{\interlacing}{\I'}
	\newcommand{\interlacingof}[1]{{#1}'}
	\newcommand{\interlacingofstar}[1]{{#1}^*}
	\newcommand{\interleavingof}[1]{\interlacing}
	\newcommand{\exextomod}{f}
	\newcommand{\exextomodof}[1]{\exextomod(#1)}
	
	\newcommand{\intertomod}{\sigma}

\title{Data Complexity of Querying Description Logic \\ Knowledge Bases under Cost-Based Semantics}
\author {
	Meghyn Bienvenu\textsuperscript{\rm 1},
	Quentin Mani\`ere\textsuperscript{\rm 2, 3}
}
\affiliations {
	\textsuperscript{\rm 1}Universit\'{e} de Bordeaux, CNRS, Bordeaux INP, LaBRI, UMR 5800, Talence, France\\
	\textsuperscript{\rm 2}Department of Computer Science, Leipzig University, Germany\\
	\textsuperscript{\rm 3}Center for Scalable Data Analytics and Artificial Intelligence (ScaDS.AI), Dresden/Leipzig, Germany\\
	meghyn.bienvenu@labri.fr, quentin.maniere@uni-leipzig.de
}

\newcommand{\citex}[2][]{(BBJ 2024\ifx&#1&%
	\else%
	, #1%
	\fi)\xspace}

\begin{document}

\maketitle

\begin{abstract}
	In this paper, we study the data complexity of querying inconsistent weighted description logic (DL) knowledge bases 
	under recently-introduced cost-based semantics. In a nutshell, the idea is to assign each interpretation
	a cost based upon the weights of the violated axioms and assertions, and certain and possible
	query answers are determined by considering all (resp.\ some) interpretations having optimal or bounded cost. 
	Whereas the initial study of cost-based semantics focused on DLs between $\mathcal{EL}_\bot$ and $\mathcal{ALCO}$,
	we consider DLs that may contain inverse roles and role inclusions, thus covering prominent DL-Lite dialects.  
	Our data complexity analysis goes significantly beyond existing results by sharpening several lower bounds
	and pinpointing the precise complexity of optimal-cost certain answer semantics (no non-trivial upper bound was known). 
	Moreover, while all existing results show the intractability of cost-based semantics, 
	our most challenging and surprising result establishes that if we consider $\dlliteboolh$ ontologies and a fixed cost bound, 
	certain answers for instance queries and possible answers for conjunctive queries
	can be computed using first-order rewriting and thus enjoy the lowest possible data complexity ($\tczero$). 
\end{abstract}


\section{Introduction}	
Ontology-mediated query answering (OMQA) has been extensively studied within the KR and database 
communities as a means of improving data access by exploiting semantic information provided by an ontology
\cite{DBLP:journals/jods/PoggiLCGLR08,DBLP:conf/rweb/BienvenuO15,DBLP:conf/ijcai/XiaoCKLPRZ18}. 
Ontologies are typically formulated in decidable fragments of first-order logic (FO),
with description logics (DLs) being a popular choice \cite{Baader_Horrocks_Lutz_Sattler_2017}. 
Given an ontology (or TBox in DL parlance) $\T$, a dataset (or ABox) $\A$, and a query $q(\vec{x})$,
the OMQA task boils down to finding the certain answers, i.e. tuples of constants $\vec{a}$ 
for which the instantiated query $q(\vec{a})$ is entailed from the knowledge base (KB) $(\T, \A)$. 
Observe that if the input KB is inconsistent,
every 
answer tuple is trivially a certain answer, so OMQA trivializes.

A prominent approach to tackling this issue is to adopt alternative 
inconsistency-tolerant semantics in order to be able to extract
meaningful information from inconsistent KBs, cf.\ \cite{Lele} and surveys \cite{biebou,DBLP:journals/ki/Bienvenu20}. 
Many of these semantics are based upon repairs, defined as inclusion-maximal consistent 
subsets of the ABox. For example, the AR semantics considers 
the query answers that hold in all repairs, while the brave semantics
returns those answers holding in at least one repair. 
Note that the line of work on repair-based semantics targets scenarios in which the 
TBox axioms are deemed fully reliable, so inconsistencies derive solely from 
errors in the ABox. However, in practice, it can be useful to allow for TBox axioms 
which typically hold but may admit rare exceptions. 
Such `soft' ontology axioms can be addressed qualitatively, using generalized notions of repair 
that have been proposed for existential rule ontologies \cite{eit}, or employing non-monotonic extensions of DLs
that support defeasible axioms cf.\ \cite{DBLP:journals/jair/BonattiLW09,DBLP:journals/ai/GiordanoGOP13,DBLP:journals/tocl/BritzCMMSV21}.
Another option is to adopt a quantitative approach, using the recently proposed cost-based semantics for DL KBs
\cite{BBJKR24}, henceforth referred to as \citex{BBJ} for succinctness.

In a nutshell, the idea is to annotate axioms and assertions with (possibly infinite) weights, 
which are used to assign a cost to each interpretation based upon the weights of the violated axioms and assertions 
(and taking into account also the number of violations of each TBox axiom). 
To query the KB, we may choose either to consider the set of interpretations achieving the optimal cost,
or we may fix a cost bound $k$ and consider all interpretations having cost at most $k$. 
We can then define the sets of certain and possible answers as those answers that hold respectively in all or some 
interpretation of optimal cost or bounded cost. As noted in \citex{BBJKR24}, the optimal-cost certain answer semantics generalizes both the 
classical certain answer semantics and the AR semantics based upon weighted ABox repairs. 
Increasing the cost bound $k$ beyond the optimal cost 
allows one to identify answers that are robust in the sense that they hold not only for the optimal-cost interpretations. 
Optimal- and bounded-cost possible answers generalize query satisfiability and can serve to compare 
candidate answers based upon their incompatibility with the KB. 

The computational complexity of querying inconsistent weighted KBs under cost-based semantics 
was investigated in \citex{BBJKR24}. Five central decision problems were considered: 
bounded-cost satisfiability (does there exist an interpretation with cost at most $k$?) plus 
query entailment under the four cost-based semantics. 
The complexity analysis was fairly comprehensive, considering both combined and data complexity, 
conjunctive and instance queries, and DLs ranging from the lightweight DL $\ELbot$ to 
the expressive DL $\mathcal{ALCO}$. One important question that was left open, however, was
the data complexity of optimal-cost certain semantics (arguably the most useful of the semantics),
for which no non-trivial upper bound was provided. Moreover, as the considered DLs allow neither inverse 
roles nor role inclusions, they do not yield any results for DLs of the DL-Lite family \cite{calvaneseetal:dllite}, which are the 
most commonly utilized in the context of OMQA.

The preceding considerations motivate us to embark on a more detailed data complexity analysis of cost-based semantics, 
considering various DL-Lite dialects and expressive DLs up to $\mathcal{ALCHIO}$. The results of our study are
summarized in Table~\ref{table:results}.
A first major contribution, detailed in Section \ref{upper}, is to provide a $\deltaptwo$ upper bound for the optimal-cost certain and possible semantics, 
matching an existing lower bound for $\ELbot$ and a new lower bound we show for \dllitecore. 
This result is obtained by using an intricate quotient construction to establish a small interpretation property, which crucially does not depend on the considered cost. 
In Section \ref{lower}, we strengthen a number of existing lower bounds for the bounded-cost semantics by showing 
that they hold even when for cost bound $k=1$, as well as providing some new lower bounds for \dllitecore. 
Finally, our most challenging and surprising technical result (presented in Section \ref{section:dl-lite}) 
is to show that if we consider $\dlliteboolh$ ontologies and a fixed cost bound, 
then certain answers for instance queries and possible answers for conjunctive queries
can be computed using first-order rewriting and thus enjoy the lowest possible data complexity ($\tczero$). 
Detailed proofs can be found in the appendix. 

\begin{table*}
	\centering
	\begin{tabular}{l@{\hspace{.0cm}}cccccc}
		& BCS$^k$, IQA$_p^{k}$, CQA$_p^k$ & IQA$_c^k$ & CQA$_c^k$ & BCS, IQA$_p$, CQA$_p$ & IQA$_c$, CQA$_c$ & IQA$_{p, c}^{opt}$, CQA$_{p, c}^{opt}$
		\\\midrule
		$\ELbot$ / $\ALC${$\Hmc\Imc$}$\Omc$
		& $\displaystyle{\mathop{\NP\ ^\text{$\ddagger$}}_{\text{Thm.\ \ref{theorem:upper-bound-alchio-cq-varying}, \ref{theorem:lower-bound-elbot-bcs-fixed-cost-1}}}}$
		& $\displaystyle{\mathop{\coNP\ ^\text{$\ddagger$}}_{\text{Thm.\ \ref{theorem:upper-bound-alchio-cq-varying}, \ref{theorem:lower-bound-el-certain-iq-fixed-cost-1}}}}$
		& $\displaystyle{\mathop{\coNP\ ^\text{$\ddagger$}}_{\text{Thm.\ \ref{theorem:upper-bound-alchio-cq-varying}, \ref{theorem:lower-bound-el-certain-iq-fixed-cost-1}}}}$
		& $\displaystyle{\mathop{\NP\ ^\text{$\dagger$}}_{\text{Thm.\ \ref{theorem:upper-bound-alchio-cq-varying}}}}$ 
		& $\displaystyle{\mathop{\coNP\ ^\text{$\dagger$}}_{\text{Thm.\ \ref{theorem:upper-bound-alchio-cq-varying}}}}$
		& $\displaystyle{\mathop{\deltaptwo\ ^\text{$\dagger$}}_{\text{Thm.\ \ref{theorem:upper-bound-alchio-cq-optimal-cost}}}}$
		\\[10pt]
		$\dllite_{\textup{core}}$ / $\dlliteboolh$
		& $\displaystyle{\mathop{\text{in }\tczero}_\text{Thm.\ \ref{theorem:upper-bound-dlliteboolh-cq-fixed}}}$
		& $\displaystyle{\mathop{\text{in }\tczero}_\text{Thm.\ \ref{theorem:upper-bound-dlliteboolh-cq-fixed}}}$
		& $\displaystyle{\mathop{\coNP\phantom{\ ^\text{$\ddagger$}}}_\text{Thm.\ \ref{theorem:upper-bound-alchio-cq-varying}, \ref{theorem:lower-bound-dllitepos-certain-cq-fixed-cost-1}}}$
		& $\displaystyle{\mathop{\NP\phantom{\ ^\text{$\dagger$}}}_\text{Thm.\ \ref{theorem:upper-bound-alchio-cq-varying}, \ref{theorem:lower-bound-dllitecore-bcs-varying-k}}}$ 
		& $\displaystyle{\mathop{\coNP\phantom{\ ^\text{$\dagger$}}}_\text{Thm.\ \ref{theorem:upper-bound-alchio-cq-varying}, \ref{theorem:lower-bound-dllitecore-bcs-varying-k}}}$
		& $\displaystyle{\mathop{\deltaptwo\phantom{\ ^\text{$\dagger$}}}_\text{Thm.\ \ref{theorem:upper-bound-alchio-cq-optimal-cost}, \ref{theorem:lower-bound-dllitecore-optcost}}}$
		\\\bottomrule
	\end{tabular}
	\caption{
		All results are completeness results, unless stated otherwise.
		$^\dagger$: lower bound from \protect\citex{BBJKR24}. 
		$^{\ddagger}$: lower bound for $k \geq 3$ from \protect\citex{BBJKR24}, improved to $k \geq 1$ in the present paper.
		For CQA, lower bounds already hold for connected acyclic BCQs. Results hold both for binary and unary encoding of weights, 
		except for $\displaystyle{\mathop{\deltaptwo}}$-hardness (only for binary encoding).
	}
	\label{table:results}
\end{table*}

\section{Preliminaries}
We recall the syntax \& semantics of description logics (DLs) and
refer readers to \cite{Baader_Horrocks_Lutz_Sattler_2017}  for further details. 

\subsubsection*{Description logic knowledge bases} 
We consider countably infinite sets 
\NC, \NR, and \NI\ of \emph{concept names}, \emph{role names}, and \emph{individual names}.
An \emph{inverse role} has the form $\rstyle{r}^-$, with $\rstyle{r} \in \NR$. 
A \emph{role} is either a role name or inverse role. We use $\NRpm= \NR \cup \{\rstyle{r}^- \mid \rstyle{r} \in \NR\}$ for the set of roles. 
If $\rstyle{r}=\rstyle{s}^-$ is an inverse role, then $\rstyle{r}^-$ denotes $\rstyle{s}$.

An \emph{\ALCIO\ concept} $\cstyle{C}$ is built according to the grammar
$
\cstyle{C},\cstyle{D} ::= \top \mid \bot \mid \cstyle{A} \mid \{ \istyle{a} \} \mid \neg \cstyle{C} \mid \cstyle{C} \sqcap \cstyle{D} \mid \exists \rstyle{r} . \cstyle{D}
$
where $\cstyle{A} \in \NC$, $\rstyle{r} \in \NRpm$,
and $\istyle{a} \in \NI$.  
A concept 
$\{\istyle{a}\}$ is called a \emph{nominal}. 
An \emph{\ALCI\ concept} is a nominal-free \ALCIO\ concept.
An \emph{\EL\ concept} is an \ALCI\ concept that uses neither negation, nor $\bot$, nor inverse roles
($\EL_\bot$ concepts may additionally use $\bot$). 

An \emph{\ALCHIO\ TBox} is a finite set of \emph{concept inclusions
	(CIs)} $\cstyle{C} \sqsubseteq \cstyle{D}$, where $\cstyle{C},\cstyle{D}$ are \ALCIO\ concepts, and
\emph{role inclusions (RIs)} $\rstyle{r} \sqsubseteq \rstyle{s}$, where $\rstyle{r},\rstyle{s} \in \NRpm$.
An \emph{\EL\ TBox} consists only of CIs between \EL\ concepts. 
An \emph{ABox} is a finite set of
\emph{concept assertions} $\cstyle{A}(\istyle{a})$ and \emph{role assertions} $\rstyle{r}(\istyle{a},\istyle{b})$
where $\cstyle{A} \in \NC$, $\rstyle{r} \in \NR$,
and $\istyle{a}, \istyle{b} \in \NI$.  
We use $\individuals(\Amc)$ for the set of individual names used
in \Amc, and $\NC(\tbox)$ (resp.\ $\NR(\tbox)$) for the set of concept (resp.\ role) names used in $\tbox$. An \emph{\ALCHIO\ knowledge base (KB)} takes the form
$\Kmc=(\Tmc,\Amc)$ with \Tmc\ an \ALCHIO\ TBox and \Amc\ an ABox.

We next introduce the syntax of some DLs of the DL-Lite family.
A \emph{basic concept} has the form $\cstyle{A}$ or $\exists \rstyle{r}$, with $\cstyle{A} \in \NC$ and $\rstyle{r} \in \NRpm$.  A
\dllitecoreh\ TBox is a finite set of CIs of the forms
$\cstyle{C} \sqsubseteq \cstyle{D}$ and $\cstyle{C} \sqcap \cstyle{D} \sqsubseteq \bot$, with $\cstyle{C},\cstyle{D}$
basic concepts, and RIs $\rstyle{r}\sqsubseteq \rstyle{s}$, with $\rstyle{r},\rstyle{s} \in \NRpm$. 
We drop superscript $\cdot^\Hmc$ if no role inclusions are admitted,
and replace $\cdot_{\textup{core}}$ by 
$\cdot_{\textup{bool}}$ to indicate that $\cstyle{C},\cstyle{D}$ may be
built from basic
concepts using $\neg$, $\sqcap$, $\sqcup$.

The semantics of DL KBs is defined as usual in terms of interpretations
$\Imc=(\Delta^\Imc,\cdot^\Imc)$ with $\Delta^\Imc$ the non-empty
\emph{domain} and $\cdot^\Imc$ the \emph{interpretation function}. 
An interpretation
satisfies a CI $\cstyle{C} \sqsubseteq \cstyle{D}$ if $\cstyle{C}^\Imc \subseteq \cstyle{D}^\Imc$ and
likewise for RIs. 
It satisfies an assertion $\cstyle{A}(\istyle{a})$ if $\istyle{a} \in \cstyle{A}^\Imc$
and $\rstyle{r}(\istyle{a},\istyle{b})$ if $(\istyle{a},\istyle{b}) \in \istyle{r}^\Imc$.
This notably requires ABox individuals to be interpreted as themselves, thus enforcing the \emph{standard names assumption} (SNA) for the ABox individuals. 
The interpretation $\Imc_\abox$ associated with an ABox $\abox$ has domain 
$\individuals(\Amc)$ and interprets concept and role names according to the assertions of $\abox$:
\begin{align*}
	\cstyle{A}^{\Imc_\abox} := ~ & \{ \istyle{a} \mid \cstyle{A}(\istyle{a}) \in \abox \}
	&
	\rstyle{r}^{\Imc_\abox} := ~ & \{ (\istyle{a}, \istyle{b}) \mid \rstyle{r}(\istyle{a}, \istyle{b}) \in \abox \}.
\end{align*}
Given any interpretation $\Imc$, we use
$\Imc|_\Delta$ to denote the restriction of \Imc\ to a subdomain $\Delta \subseteq \Delta^\Imc$.

\subsubsection*{Queries} First-order queries are given by formulas in first-order logic with equality. Since we wish to query DL KBs, we consider queries whose 
relational atoms can be either concept atoms $\cstyle{A}(t)$ or role atoms $\rstyle{r}(t, t')$, where $\cstyle{A} \in \NC$, $\rstyle{r}\in \NR$, and $t, t'$ are \emph{terms} (variables or individuals). 
We mostly focus on \emph{conjunctive queries (CQs)} which have the form $\exists \vec{y} \psi$, where $\psi$ is a conjunction of concept and role atoms,
and $\vec{y}$ a tuple of variables from $\psi$. 
\emph{Instances queries (IQs)} are CQs
with a single atom. We also consider \emph{acyclic} and \emph{connected} CQs, meaning that their associated undirected graph (containing an edge $\{t,t'\}$ for each role atom $\rstyle{r}(t, t')$) has these properties. 
A Boolean CQ (BCQ) is a CQ that has no free variables. 
%
We write $\Imc \models q$ to indicate that an interpretation 
$\Imc$ satisfies a BCQ $q$.

\subsubsection*{Complexity classes} 	For any syntactic object $O$ such as a TBox, ABox, or query, we use 
$|O|$ to denote the \emph{size} of $O$, meaning the encoding of $O$
as a word over a suitable alphabet.
Our complexity results concern the well-known complexity classes
$\np$ 
and $\conp$ 
as well as 
$\deltaptwo$
(deterministic polynomial time with access to an $\np$ oracle) and 
$\tczero$. 
We omit the formal definition of $\tczero$, which is based upon circuits, as to understand our results, it suffices to know that it is in $\tczero$ (in data complexity) to test whether a Boolean first-order query is satisfied in a finite interpretation. 

\section{Cost-Based Semantics for Weighted KBs}\label{costsem}
In this section, we recall the definition of cost-based semantics from \citex{BBJKR24}
and introduce the associated reasoning tasks. 
We also recall and prove some basic properties used in later sections. 

Throughout the paper, we work with weighted KBs, whose assertions and axioms are annotated with weights: 
\begin{definition}
	A \emph{weighted knowledge base} (WKB) $\WKB$ consists of 
	a knowledge base  $(\tbox, \abox)$
	and a weight function $\omega: \T \cup \A \mapsto \mathbb{N}_{> 0} \cup \{ \infty \}$. 
	We can similarly define \emph{weighted TBoxes} ($\tbox_\omega$) and \emph{weighted ABoxes} ($\abox_\omega$). 
\end{definition}

Intuitively, these weights can be viewed as the penalties incurred for violating assertions and
axioms: those
having cost $1$ are the least reliable, while 
those assigned maximal weight $\infty$ should definitely be satisfied.

Interpretations will then be assigned costs based upon the sets of violations 
of the TBox axioms and ABox assertions. Note that differently from prior work, 
we also define violations of role inclusions, given by pairs of domain elements. 

\begin{definition}\label{defvio1} 
	Given an interpretation $\I$,  the \emph{set of violations of a concept inclusion $\cstyle{B} \sqsubseteq  \cstyle{C}$} in $\I$ 
	is $\vio{\cstyle{B} \sqsubseteq \cstyle{C}}{\I} 
	= \cstyle{B}^\I \setminus \cstyle{C}^{\I}$, 
	the \emph{set of violations of a role inclusion $\rstyle{r} \sqsubseteq \rstyle{s}$} in $\I$ 
	is $\vio{\rstyle{r} \sqsubseteq \rstyle{s}}{\I} = \rstyle{r}^\I \setminus \rstyle{s}^\I$, and
	the \emph{violations of an ABox $\A$} in $\I$ are  
	$\vio{\A}{\I} = \{ \alpha \in \A \mid \I \not \models \alpha \}$. 
\end{definition}

\begin{definition}
	Let $\WKB$ be a WKB. The \emph{cost of an interpretation} $\I$ w.r.t.\ $\wkb$ is defined by:         
	$$\cost(\I)  = \sum_{\tau \in \T}\, \omega(\tau) |\vio{\tau}{\I}| + \sum_{\alpha \in \vio{\A}{\I}}\!\!\omega(\alpha)$$
	The \emph{optimal cost} of $\wkb$ is $\optc{\wkb}  = \mathsf{min}_{\I} (\cost(\I))$. 
	A WKB $\wkb$ is \emph{$k$-satisfiable} if $\cost(\I)\leq k$ for some interpretation~$\I$.
\end{definition}
We recall next the four cost-based semantics proposed in \citex{BBJKR24},
which 
depend
on whether one considers interpretations whose cost is less than a provided bound,
or the interpretations having optimal cost, and whether the query is required to hold in all or at least one such interpretation.
\begin{definition}
	Let $q$ be a BCQ, $\WKB$ a WKB, and $k$ an integer.
	We say that $q$ is entailed by $\wkb$ under 
	\begin{itemize}
		\item \emph{\bsem{certain}}, written $\wkb \sat{c}{k} q$, if $\I \models q$ for every interpretation $\I$ with $\cost(\I) \leq k$;
		
		\item \emph{\bsem{possible}}, written $\wkb \sat{p}{k} q$, if $\I \models q$ for some interpretation $\I$ with $\cost(\I) \leq k$;
		
		\item \emph{\optsem{certain}}, written $\wkb \sat{c}{opt} q$, if $\I \models q$ for every interpretation $\I$ with $\cost(\I) = \optc{\wkb}$;	
		
		\item \emph{\optsem{possible}}, written $\wkb \sat{p}{opt} q$, if $\I \models q$ for some interpretation $\I$ with $\cost(\I) = \optc{\wkb}$.
	\end{itemize}
	These semantics extend to non-Boolean CQs in the expected way, e.g.\ the opt-cost certain answers to a CQ $q(\vec{x})$ w.r.t.\ $(\T,\A)_\omega$
	are the tuples $\vec{a}$ from $\individuals(\Amc)$ s.t.\ $\wkb \sat{c}{opt} q(\vec{a})$. 
\end{definition}

If the underlying KB is satisfiable, then 
the certain and possible optimal-cost semantics 
coincide with query entailment and query satisfiability
(or with classical notions of certain and possible answers, in the case of non-Boolean queries). 
These semantics are thus intended to be used when the underlying KB is inconsistent. 
The opt-cost certain answers 
identifies those answers
that hold in the interpretations deemed most likely
and have been shown to generalize previously considered weight-based repair semantics \citex{BBJKR24}.
By considering values of $k$ beyond $\optc{\wkb}$, 
we can use the \bsem{certain} to identify `robust' answers which hold not only in the optimal-cost interpretations but also in those with close-to-optimal cost. By contrast, the opt-cost and $k$-cost bounded possible answers can serve to rank candidate answers based upon their degree of incompatibility with the WKB.

We now formalize the decision problems for cost-based semantics investigated in this paper, which differ depending on which cost bound is used and whether it is given as input:
\begin{itemize}
	\item \emph{Bounded cost satisfiability}  (\BCS) takes as input a WKB $\WKB$ and an integer $k$ and decides 
	whether there exists an interpretation $\I$ with $\cost(\I) \leq k$.
	\item \emph{$k$-cost satisfiability}  (BCS$^k$) takes as input a WKB $\WKB$ 
	and decides whether there exists an interpretation $\I$ with $\cost(\I) \leq k$.
	\item \emph{Bounded-cost certain (resp.\ possible) BCQ entailment} (\dec{C}{}{c} / \dec{C}{}{p}) takes as input a WKB $\WKB$, a BCQ $q$ and an integer $k$ and decides whether $\wkb \sat{c}{k} q$ (resp.\  $\wkb \sat{p}{k} q$).
	\item \emph{$k$-cost certain (resp.\ possible) BCQ entailment} (\dec{C}{k}{c} / \dec{C}{k}{p}) takes as input a WKB $\WKB$ and a BCQ $q$ and decides whether $\wkb \sat{c}{k} q$ (resp.\  $\wkb \sat{p}{k} q$).
	\item \emph{Optimal-cost certain (resp.\ possible) BCQ entailment} (\dec{C}{opt}{c} / \dec{C}{opt}{p}) 
	takes as input a WKB $\WKB$ and a BCQ $q$ and decides if $\wkb \sat{c}{opt} q$ (resp.\ $\wkb \sat{p}{opt} q$). 
\end{itemize}
We will also consider the restrictions of the BCQ entailment problems to the case of instance queries, 
denoted by \dec{I}{}{c}, \dec{I}{}{p}, \dec{I}{k}{c}, \dec{I}{k}{p}, \dec{I}{opt}{c} and \dec{I}{opt}{p} respectively.

We shall study the \emph{data complexity} of the preceding reasoning tasks. 
For the fixed-cost and optimal-cost decision problems, data complexity is measured with respect to the size of the input weighted ABox, 
while the size of the weighted TBox and query (if present)
are treated as constants. For the bounded-cost problems, we measure complexity w.r.t.\ the weighted ABox and the input integer.  
Both the ABox weights and the input integer (if present) are assumed to be encoded in binary. 

We conclude the section with some easy lemmas, which establish useful reductions between the decision problems. 

\begin{lemmarep}
	\label{lemma:bcs-to-possible-iq}
	BCS$^k$ for $\dllitecore$ (resp.\ $\ELbot$) reduces to IQA$_p^k$ for $\dllitecore$ (resp.\ $\ELbot$).
	In particular, the same holds for BCS and IQA$_p$.
\end{lemmarep}

\begin{proof}
		Trivial: $\kb_\omega$ is $k$-satisfiable iff $\kb_\omega \models_p^k \cstyle{A}(\istyle{a})$ where $\cstyle{A}$ is a fresh concept name and $\istyle{a}$ is any individual name.
\end{proof}

\begin{lemmarep}
	\label{lemma:possible-iq-to-certain-iq}
	IQA$^k_p$ for $\dllitecore$ (resp.\ $\ELbot$) reduces to the complement of IQA$_c^{k+1}$ for $\dllitecore$ (resp.\ $\ELbot$).
	In particular, IQA$_p$ reduces to the complement of IQA$_c$.
\end{lemmarep}


\begin{proof}
	We first treat the case of $\dllitecore$.		
	Let $(\tbox, \abox)_\omega$ be our $\dllitecore$ WKB, $q$ our Boolean IQ and $k$ our cost.
	Our reduction differs depending on the shape of $q$.
	We treat the three cases of (i) $q = \cstyle{A}(\istyle{a})$, (ii) $q = \exists y\ \cstyle{A}(y)$ and (iii) $q = \rstyle{r}(\istyle{a}, \istyle{b})$, and argue that the other shapes are equivalent to an IQ of shape (i) or (ii) up to adding a few axioms in $\tbox$.
	Indeed, $q = \exists y\ \rstyle{r}(\istyle{a}, y)$ is equivalent to $q' := \cstyle{B}(\istyle{a})$ where $\cstyle{B}$ is a fresh concept name and we add the axiom $\cstyle{B} \equiv \exists \rstyle{r}$ to $\tbox$, with infinite weight.
	Similarly, $q = \exists y\ \rstyle{r}(y, \istyle{a})$ is equivalent to $q' := \cstyle{B}(\istyle{a})$ where $\cstyle{B}$ is a fresh concept name and we add the axiom $\cstyle{B} \equiv \exists \rstyle{r}^-$ to $\tbox$, with infinite weight.
	Finally, $q = \exists y_1\ \exists y_2\ \rstyle{r}(y_1, y_2)$ is equivalent to $q' := \exists y\ \cstyle{B}(y)$ where $\cstyle{B}$ is a fresh concept name and we add the axiom $\cstyle{B} \equiv \exists \rstyle{r}$ to $\tbox$, with infinite weight.
	
	Now, for shape (i), that is $q = \cstyle{A}(\istyle{a})$, the reduction proceeds as follows.		
	We extend $\tbox$ into $\tbox'$ by adding axiom $\cstyle{A} \sqcap \bar{\cstyle{A}} \sqsubseteq \bot$ and $\abox$ into $\abox'$ by adding assertions $\cstyle{A}(\istyle{a})$ and $\cstyle{\bar A}(\istyle{a})$.
	The cost function $\omega$ is extended into $\omega'$ by assigning weight $\infty$ to fresh axiom $\cstyle{A} \sqcap \bar{\cstyle{A}} \sqsubseteq \bot$ and to assertion $\cstyle{A}(a)$, and weight $1$ to $\cstyle{\bar A}(\istyle{a})$.\smallskip
	
	\noindent\textbf{Claim}: $(\tbox, \abox)_\omega \models_p^k \cstyle{A}(\istyle{a})$ iff $(\tbox', \abox')_{\omega'} \not\models_c^{k+1} \cstyle{\bar A}(\istyle{a})$.\smallskip
	
	\noindent $(\Rightarrow).$
	Assume there exists an interpretation $\Imc$ that interprets the symbols in  $(\tbox, \abox)$ 
	in such a way that  $\omega(\Imc) \leq k$ and $\Imc \models \cstyle{A}(\istyle{a})$.
	Observe that if we define an interpretation $\Jmc$ that is the same as $\Imc$ except that the concept $\cstyle{\bar A}$ is interpreted as ${\cstyle{\bar A}}^\Jmc := \emptyset$, we obtain $\omega'(\Jmc) \leq k+1$. 
	Moreover, since $\Jmc$ does not satisfy $\cstyle{\bar A}(\istyle{a})$,  we get $(\tbox', \abox')_{\omega'} \not\models_c^{k+1} \cstyle{\bar A}(\istyle{a})$.
	
	\noindent $(\Leftarrow).$
	Assume there exists an interpretation $\Imc$ 
	such that $\omega'(\Imc) \leq k+1$ and $\Imc \not\models \cstyle{\bar A}(\istyle{a})$.
	Both $\cstyle{A} \sqcap \bar{\cstyle{A}} \sqsubseteq \bot$ and $\cstyle{A}(\istyle{a})$ have infinite weight w.r.t.\ $\omega'$, and thus are always satisfied in $\Imc$, as $\omega'(\Imc)$ is finite.
	In particular $\Imc \models \cstyle{A}(a)$ and $\Imc \not\models \cstyle{\bar A}(\istyle{a})$.
	It follows that $\omega'(\Imc) \geq \omega(\Imc) + \omega'(\cstyle{\bar A}(\istyle{a}))$, thus $k + 1 - 1 \geq \omega(\Imc)$, which yields $k \geq \omega(\Imc)$.
	Therefore $(\tbox, \abox)_\omega \models_p^k \cstyle{A}(\istyle{a})$. \medskip
	
	We next consider the case of queries having shape (ii), where  $q = \exists y\ \cstyle{A}(y)$.
	We extend $\tbox$ into $\tbox'$, $\abox$ into $\abox'$ and $\omega$ into $\omega'$ by adding the following axioms and assertion, with corresponding weights being indicated below:
	\[
	\begin{array}{c@{\qquad\qquad}c@{\qquad\qquad}c@{\qquad\qquad}c@{\qquad\qquad}c}
		\cstyle{A}_0 \sqsubseteq \exists \rstyle{s} 
		&
		\exists \rstyle{s}^- \sqsubseteq \cstyle{A}
		&
		\cstyle{A} \sqsubseteq \exists \rstyle{r} 
		&
		\exists \rstyle{r}^- \sqsubseteq {\cstyle{\bar A}}
		& 
		\cstyle{A}_0(\istyle{a})
		\smallskip\\
		\infty
		& \infty
		& \infty
		& 1
		& \infty
	\end{array}
	\]
	where $\cstyle{\bar A}$, $\rstyle{r}$, $\rstyle{s}$ and $\cstyle{a}$ are fresh.\smallskip
	
	\noindent\textbf{Claim}: $(\tbox, \abox)_\omega \models_p^k \exists y\ \cstyle{A}(y)$ iff $(\tbox', \abox')_{\omega'} \not\models_c^{k+1} \exists y\ \cstyle{\bar A}(y)$.\smallskip
	
	\noindent $(\Rightarrow).$
	Assume there exists an interpretation $\Imc$ which interprets the symbols in $(\tbox, \abox)_\omega$ in such a way that $\omega(\Imc) \leq k$ and $\cstyle{A}^\Imc \neq \emptyset$.
	We transform $\Imc$ into $\Jmc$ by interpreting the fresh predicates $\cstyle{\bar A}$, $\rstyle{r}$ and $\rstyle{s}$ as follows:
	\begin{align*}
		\cstyle{\bar A}^\Jmc :=~  & \emptyset
		&
		\rstyle{r}^\Jmc := ~& \cstyle{A}^\Imc \times \{ \istyle{a} \}
		&
		\rstyle{s}^\Jmc := ~& \{ \istyle{a}\} \times \cstyle{A}^\Imc
	\end{align*}
	It is readily verified that $\omega'(\Imc) \leq k+1$ (the extra cost coming from $\exists \rstyle{r}^- \sqsubseteq \cstyle{\bar A}$ being violated at $\istyle{a}$).
	Furthermore, as $\cstyle{\bar A}^\Jmc := \emptyset$, we obtain $(\tbox', \abox')_{\omega'} \not\models_c^{k+1} \exists y\ \cstyle{\bar A}(y)$.
	
	\noindent $(\Leftarrow).$
	Assume there exists an interpretation $\Imc$ 
	such that $\omega'(\Imc) \leq k+1$ and $\cstyle{\bar A}^\Imc = \emptyset$.
	By the three fresh axioms added to $\tbox'$ and the fresh assertion that have infinite cost, we can derive that $\cstyle{A}^\Imc \neq \emptyset$ and that $\exists \rstyle{r}^- \sqsubseteq \cstyle{\bar A}$ is violated at least once in $\Imc$.
	Observing that the cost of $\Imc$ w.r.t.\ the WKB $(\tbox, \abox)_\omega$ is at most $k$, we obtain $(\tbox, \abox)_\omega \models_p^k \exists y\ \cstyle{A}(y)$. \medskip
	
	For shape (iii), that is $q = \rstyle{r}(\istyle{a}, \istyle{b})$, we extend $\tbox$ into $\tbox'$, $\abox$ into $\abox'$ and $\omega$ into $\omega'$ by adding the following axioms and assertion, with their corresponding weights being indicated below:
	\[
	\begin{array}{c@{\qquad\qquad}c@{\qquad\qquad}c}
		\cstyle{\bar A} \sqcap \exists \rstyle{r} \sqsubseteq \bot 
		& \rstyle{r}(\istyle{a}, \istyle{b}) 
		& \cstyle{\bar A}(\istyle{a})
		\smallskip \\
		\infty
		& \infty
		& 1
	\end{array}
	\]
	We omit the proof of the following claim, which is essentially the same as the one used for shape (i):\smallskip
	
	\noindent\textbf{Claim}:
	$(\tbox, \abox)_\omega \models_p^k \rstyle{r}(\istyle{a}, \istyle{b})$ iff $(\tbox', \abox')_{\omega'} \not\models_c^{k+1} \cstyle{\bar A}(\istyle{a})$. \medskip
	
	
	Let us now turn to the case of $\ELbot$.
	We can argue as before that it is sufficient to focus on the three same cases.
	Note that we need to change the argument for IQs with shape $q = \exists y\ \rstyle{r}(y, \istyle{a})$, as we can no longer rely on inverse roles in the TBox.
	However, this lack of inverse roles in the ontology also means that the above $q$ is actually equivalent to $q' = \exists y_1\ \exists y_2\ \rstyle{r}(y_1, y_2)$, which is turn equivalent to an IQ of shape (ii) as previously seen.
	
	It is clear that we can reuse the same construction as before to handle queries of shapes (i) and (iii).
	For shape (ii), we simply replace the use of inverse roles by qualified existential restrictions. 
	The added axioms in $\tbox$ become instead:
	\[
	\begin{array}{c@{\qquad\qquad}c@{\qquad\qquad}c}
		\cstyle{A}_0 \sqsubseteq \exists \rstyle{s} .\cstyle{A}
		&
		\cstyle{A} \sqsubseteq \exists \rstyle{r} . \cstyle{A}_1
		&
		\cstyle{A}_1 \sqsubseteq {\cstyle{\bar A}}
		\smallskip \\
		\infty
		& \infty
		& 1 
	\end{array} \qedhere
	\] 
\end{proof}

\begin{lemmarep}
	\label{remark:k-to-k+1}
	For every $k \geq 0$, IQA$^k_p$ for $\dllitecore$ (resp.\ $\ELbot$) reduces to IQA$_p^{k+1}$ for $\dllitecore$ (resp.\ $\ELbot$) WKBs.
\end{lemmarep}

\begin{proof}
	Simply add an assertion $\cstyle{A}(\istyle{a})$ with weight $1$ and an axiom $\cstyle{A} \sqsubseteq \bot$ with weight $\infty$, where $\cstyle{A}$ and $\istyle{a}$ are fresh.
\end{proof}

\section{An Upper Bound for the General Case}\label{upper}

The aim of this section is to establish the next two theorems:
\begin{theoremrep}
	\label{theorem:upper-bound-alchio-cq-varying}
	CQA$_p$ (resp.\ CQA$_c$) for $\ALCHIO$ is in $\NP$ (resp.\ in $\coNP$).
\end{theoremrep}
\begin{proof}
	An algorithm for CQA$_p$ (resp.\ CQA$_c$): (i) non-deterministically guesses an interpretation $\Imc$ whose size is bounded as in Lemma~\ref{lemma:brutal-fitration} (resp.\ Lemma~\ref{lemma:new-prop-8}), (ii) verifies whether $\Imc \models q$ (resp. $\Imc \not\models q$) and whether $\omega(\Imc) \leq k$; and, (iii) accepts iff both tests are successful.
	Note that both tests in step (ii) can be carried in deterministic polynomial time, so that the overall procedure yields an $\NP$ (resp.\ $\coNP$) complexity upper bound.
	Soundness of such an algorithm is trivial, while its completeness is guaranteed by Lemma~\ref{lemma:brutal-fitration} (resp.\ Lemma~\ref{lemma:new-prop-8}).
\end{proof}

\begin{theoremrep}
	\label{theorem:upper-bound-alchio-cq-optimal-cost}
	CQA$_p^{opt}$ and CQA$_c^{opt}$ for $\ALCHIO$ are in $\deltaptwo$.
\end{theoremrep}

\begin{proof}
		Notice that if the optimal cost is infinite, then the answer is trivially `true' (resp.\ `false') for CQA$_p$ (resp.\ CQA$_c$); and this case can easily be detected using an $\NP$ oracle (see \emph{e.g.}\ \cite{DBLP:journals/jar/OrtizCE08}) asking for the (standard) satisfiability of the sub-KB consisting of all assertions and axioms with infinite weight. 
		Otherwise, the optimal cost is finite, and we observe that it is bounded from above by an exponential w.r.t.\ data complexity.
		Indeed, starting from $\Imc$ with finite cost $k$ and $q$ a tautology, we apply Lemma~\ref{lemma:brutal-fitration} and obtain an interpretation $\Jmc$ whose size is bounded by a polynomial $p$ in $\sizeof{\abox}$.
		Notice that $\Jmc$ does not violate any assertion or axiom that has infinite cost since $\omega(\Jmc) \leq k$.
		Let $W := \max_{\tau \in \kb, \omega(\tau) < \infty}(\omega(\tau))$ be the maximal finite cost assigned by $\omega$.
		Even if $\Jmc$ violates all finite-cost assertions and axioms in every possible way, then $\omega(\Jmc) \leq W (\sizeof{\Delta^\Jmc} \sizeof{\abox} + \sizeof{\Delta^\Jmc}^2 \sizeof{\tbox})$, which is at most an exponential w.r.t.\ data complexity (if the weights on $\abox$ were encoded in unary, then it would only be a polynomial).
		Performing a binary search to identify the optimal cost 
		thus only requires a polynomial number of calls to Theorem~\ref{theorem:upper-bound-alchio-cq-varying} (given again a tautology) seen as an $\NP$-oracle.
		Once $\optc{\wkb}$ 
		is identified, we use Theorem~\ref{theorem:upper-bound-alchio-cq-varying} (this time with the query $q$ of interest!) with input 
		$\optc{\wkb}$ 
		as one last $\NP$-oracle (resp.\ $\coNP$-oracle) to decide CQA$_p^{opt}$ (resp.\ CQA$_c^{opt}$).
\end{proof}

Starting from an interpretation $\Imc$ that satisfies (or does not satisfy) the query $q$ of interest, the main technical ingredient is the construction of another interpretation $\Jmc$ that behaves as $\Imc$ w.r.t.\ $q$, whose cost is at most the cost of $\Imc$, and whose domain has a size polynomially bounded by the size of the input ABox $\abox$.
This is fairly easy if query satisfaction is to be preserved, that is, for CQA$_p$: one can brutally collapse elements together according to their type using  
the well-known filtration technique \cite{Baader_Horrocks_Lutz_Sattler_2017}.
\begin{lemmarep}
\label{lemma:brutal-fitration}
Let $\kb = (\tbox , \abox)_\omega$ be an $\ALCHIO$ WKB,
$k$ an integer, 
and $q$ a BCQ.
If there exists an interpretation $\Imc$ such that
$\omega(\Imc) \leq k$ and $\Imc \models q$, then there is an interpretation $\Jmc$ such that
$\omega(\Jmc) \leq k$ and $\Jmc \models q$, and whose domain $\Delta^\Jmc$ has cardinality 
bounded polynomially in $\sizeof{\abox}$, with $\sizeof{\tbox}$ and $\sizeof{q}$ treated as constants, and independently from $k$.
\end{lemmarep}
\begin{proof}
	\newcommand{\ftype}{\type}
	\newcommand{\Cmc}{\mathcal{C}}
	Let $\kb = (\tbox , \abox)_\omega$ be an $\ALCHIO$ WKB,	$k$ an integer, 
	and $q$ a BCQ.
	Assume we have an interpretation $\Imc$ with cost 
	$\omega(\Imc) \leq k$ such that $\Imc \models q$.
	Consider the set of concepts occurring in the TBox $\tbox$ and close it under negation and sub-concepts.
	We denote $\Cmc$ the obtained set of concepts, whose size is clearly bounded as $\sizeof{\Cmc} \leq 2\sizeof{\tbox}^2$.
	The type $\ftype_\Imc(d )$ of an element $d  \in \Delta^\Imc$ is the unique subset of $\Cmc$ such that, for every $C \in \ftype_\Imc(d )$, we have $d \in C^\Imc$ iff $C \in \ftype_\Imc(d )$.
	Note that this notion of type slightly differs from the one used in Section~\ref{section:dl-lite} and will only be used within the present proof.
	Let $W := \{ w_{\ftype_\Imc(d)} \mid d \in \Delta^\Imc \setminus \individuals \}$.
	We define the following mapping:
	\[
	\begin{array}{rcl}
		\rho : \Delta^\Imc & \rightarrow & \individuals(\kb) \cup W
		\\
		d  & \mapsto & \left\{ \begin{array}{ll}
			d & \text{if } d  \in \individuals(\kb)
			\\
			w_{\ftype_\Imc(d )} & \text{otherwise}
		\end{array}\right.
	\end{array}
	\]
	where $\individuals(\kb)$ refers to the individuals occurring in the KB.
	We denote $\Jmc := \rho(\Imc)$ the image of interpretation $\Imc$ via $\rho$.
	Note that $\Jmc$ has the desired size $\sizeof{\individuals(\kb)} + \sizeof{W} \leq \sizeof{\abox} + \sizeof{\tbox} + 2^{2\sizeof{\tbox}^2}$. 
	Since $q$ is a BCQ, it is preserved under homomorphisms, and thus we have $\Jmc \models q$ from $\Imc \models q$.
	It remains to verify that $\omega(\Jmc) \leq k$.
	
	To this end, we first verify, as is standard with filtration, that: 
	\begin{center}
	$(\star)$ \quad for every $d \in \Delta^\Imc$, we have $\type_\Jmc(\rho(d)) = \type_\Imc(d)$
	\end{center}
	We proceed by structural induction and show that for every $d \in \Delta^\Imc$, we have $\cstyle{C} \in \type_\Jmc(\rho(d))$ iff $\cstyle{C} \in \type_\Imc(d)$.
	\begin{itemize}
		\item $\cstyle{C} \in \NC$. Let $d \in \Delta^\Imc$. If $\cstyle{C} \in \type_\Imc(d)$, then we have by definition of $\cstyle{C}^\Jmc$ that $\cstyle{C} \in \type_\Jmc(\rho(d))$ as desired. Conversely, if $\cstyle{C} \in \type_\Jmc(\rho(d))$, then by definition of $\cstyle{C}^\Jmc$ there exists $d' \in \Delta^\Imc$ such that $\cstyle{C} \in \type_\Imc(d')$ and $\rho(d) = \rho(d')$.
		By definition of $\rho$, we get $\type_\Imc(d) = \type_\Imc(d')$ and thus $\cstyle{C} \in \type_\Imc(d)$ as desired.
		\item $\cstyle{C} = \top$. By definition we have $\cstyle{C}^\Imc = \Delta^\Imc$ and $\cstyle{C}^\Jmc = \Delta^\Jmc$, thus $\cstyle{C} \in \type_\Jmc(\rho(d))$ iff $\cstyle{C} \in \type_\Imc(d)$ trivially holds.
		\item $\cstyle{C} = \bot$. By definition we have $\cstyle{C}^\Imc = \emptyset$ and $\cstyle{C}^\Jmc = \emptyset$, thus $\cstyle{C} \in \type_\Jmc(\rho(d))$ iff $\cstyle{C} \in \type_\Imc(d)$ trivially holds.
		\item $\cstyle{C} = \{ \istyle{a} \}$. By definition we have $\cstyle{C}^\Imc = \istyle{a}$ and $\cstyle{C}^\Jmc = \istyle{a}$. Notice that $\rho(\istyle{a}) = \istyle{a}$, thus $\cstyle{C} \in \type_\Jmc(\rho(d))$ if $\cstyle{C} \in \type_\Imc(d)$. Conversely, by definition of $\rho$, the only pre-image of $\istyle{a}$ by $\rho$ is $\istyle{a}$, thus $\cstyle{C} \in \type_\Jmc(\rho(d))$ only if $\cstyle{C} \in \type_\Imc(d)$.
		\item $\cstyle{C} = \neg \cstyle{B}$. Let $d \in \Delta^\Imc$. By the induction hypothesis, and using the fact that types are closed under taking subconcepts, we obtain that $\cstyle{B} \in \type_\Jmc(\rho(d))$ iff $\cstyle{B} \in \type_\Imc(d)$. The semantics of concept negation gives $\cstyle{C}^\Imc = \Delta^\Imc \setminus \cstyle{B}^\Imc$ and $\cstyle{C}^\Jmc = \Delta^\Jmc \setminus \cstyle{B}^\Jmc$, thus immediately $\cstyle{C} \in \type_\Jmc(\rho(d))$ iff $\cstyle{C} \in \type_\Imc(d)$.
		\item $\cstyle{C} = \cstyle{B}_1 \sqcap \cstyle{B}_2$. The argument is similar to the one used for $\cstyle{C} = \neg \cstyle{B}$.
		\item $\cstyle{C} = \exists \rstyle{r}.\cstyle{B}$. Let $d \in \Delta^\Imc$. Assume first that $\cstyle{C} \in \type_\Imc(d)$, that is, there exists $e \in \cstyle{B}^\Imc$ such that $(d, e) \in \rstyle{r}^\Imc$.
		Then, by definition of $\rstyle{r}^\Jmc$, we have $(\rho(d), \rho(e)) \in \rstyle{r}^\Jmc$. Furthermore, by the induction hypothesis and using the fact that types are closed under subconcepts, we obtain $\cstyle{B} \in \type_\Jmc(\rho(e))$, thus $\rho(d) \in (\exists \rstyle{r}.\cstyle{B})^\Jmc$, that is, $\cstyle{C} \in \type_\Jmc(\rho(d))$ as desired.
		Conversely, assume that $\cstyle{C} \in \type_\Jmc(\rho(d))$, that is there exists $e \in \cstyle{B}^\Jmc$ such that $(\rho(d), e) \in \rstyle{r}^\Jmc$ and $e \in \cstyle{B}^\Jmc$.
		By definition of $\rstyle{r}^\Jmc$, there exists $(d', e') \in \rstyle{r}^\Imc$ such that $\rho(d') = \rho(d)$ and $\rho(e') = e$.
		By the induction hypothesis and using the fact that types are closed under subconcepts, we get $\cstyle{B} \in \type_\Imc(e')$.
		Joint with $(d', e') \in \rstyle{r}^\Imc$, this gives $\cstyle{C} \in \type_\Imc(d')$.
		Now using $\rho(d') = \rho(d)$, we obtain that $\type_\Imc(d) = \type_\Imc(d')$ and thus $\cstyle{C} \in \type_\Imc(d)$ as desired.
	\end{itemize}
	
	We now come back to verifying $\omega(\Jmc) \leq k$, and prove that $\omega(\Jmc) \leq \omega(\Imc)$, which suffices since the cost of $\Imc$ is assumed to be at most $k$.
	By $(\star)$, every concept assertion violated in $\Jmc$ already is in $\Imc$.
	Similarly, every concept inclusion violated in $\Jmc$ by an individual is already violated in $\Imc$ by the same individual.
	Still using $(\star)$, a violation of a concept inclusion in $\Jmc$ on an element $w_t$ is already violated by all non-individual elements of $\Imc$ whose type is $t$ (note that by definition of $W$, there exists at least one).
	Since $\rho$ is the identity on individuals, every role assertion violated in $\Jmc$ already is in $\Imc$.
	It remains to treat the case of the violation at $(d, e) \in \rstyle{r}^\Jmc \setminus \rstyle{s}^\Jmc$ of a role inclusion $\rstyle{r} \sqsubseteq \rstyle{s}$.
	By definition of $\rstyle{r}^\Jmc$, we have a pair $(d', e') \in \istyle{r}^\Imc$ such that $\rho(d') = d$ and $\rho(e') = e$.
	By definition of $\istyle{s}^\Jmc$, there is no pair $(d'', e'') \in \rstyle{s}^\Imc$ such that $\rho(d'') = d$ and $\rho(e'') = e$ (as otherwise $(d, e) \in \rstyle{s}^\Jmc$).
	Therefore the pair $(d', e')$ already violates $\rstyle{r} \sqsubseteq \rstyle{s}$ in $\Imc$.
	Unfolding the definition of $\rho$, it is readily verified that the produced pair $(\rstyle{d}', e')$ is different for every different violation of $\rstyle{r} \sqsubseteq \rstyle{s}$.
	We proved that $\omega(\Jmc) \leq \omega(\Imc)$ as desired. 
\end{proof}
It becomes more challenging if query non-satisfaction is the property to preserve, that is, to address CQA$_c$.
In particular, the solution adopted in \citex[Proposition~8]{BBJKR24} for $\ALCO$ WKBs yields a polynomial bound that depends on the bounded cost $k$.
This makes their technique adequate when the cost is fixed, that is, for CQA$_c^k$, or if the encoding of $k$ is given in unary, but otherwise does not provide a polynomial upper bound w.r.t.\ data complexity.
We deeply rework the approach, not only to support $\ALCHIO$ WKBs, but also to obtain a cost-independent 
bound 
as follows.
\begin{lemma}
	\label{lemma:new-prop-8}
	Let $\kb = (\tbox , \abox)_\omega$ be an $\ALCHIO$ WKB,
	$k$ an integer,
	and $q$ a BCQ.
		If there exists an interpretation $\Imc$ such that $\omega(\Imc) \leq k$ and $\Imc \not\models q$, then there is an interpretation $\Jmc$ such that $\omega(\Jmc) \leq k$ and $\Jmc \not\models q$, and whose domain $\Delta^\Jmc$ has cardinality that is bounded polynomially in $\sizeof{\abox}$, with $\sizeof{\tbox}$ and $\sizeof{q}$ treated as constants, and independently from $k$.
\end{lemma}

With Lemmas~\ref{lemma:brutal-fitration} and \ref{lemma:new-prop-8}, we obtain the $\NP$ and $\coNP$ upper bounds in Theorem~\ref{theorem:upper-bound-alchio-cq-varying} with standard guess-and-check procedures.
For Theorem~\ref{theorem:upper-bound-alchio-cq-optimal-cost}, the $\deltaptwo$ algorithms proceed similarly but first identify the optimal cost via a binary search, using an exponential bound on the optimal cost (whenever finite).

Now, to prove Lemma~\ref{lemma:new-prop-8}, we rely on the adaptation of a quotient construction from \cite[Theorem~8]{Maniere} defined to answer \emph{counting conjunctive queries} over $\ALCHI$ KBs.
This construction was already reused by \citex{BBJKR24} for $\ALCO$ WKBs, and it is not too difficult to handle inverse roles and role inclusions by sticking closer to the original version.
From the starting interpretation $\Imc$, our adaptation differs from theirs as it also takes as a parameter a subset $\Vmc \subseteq \tbox$ of the considered TBox $\tbox$.
In the constructed interpretation, we violate axioms of $\Vmc$ exactly as in the original interpretation $\Imc$.
Intuitively, one can think of these axioms in $\Vmc$ as so expensive to violate that one cannot do better than in $\Imc$, while potential violations of axioms from $\tbox \setminus \Vmc$ can be handled in a more systematic and structured manner.
Violations of ABox assertions are easier to control and are preserved exactly as in the original interpretation $\Imc$.
For an interpretation $\Jmc$, we denote $vio_\Vmc(\Jmc) := \bigcup_{\tau \in \Vmc} vio_\tau(\Jmc)$.
Adapting the quotient technique yields the following:
\begin{lemmarep}
	\label{lemma:main-quotient}
	Let $\kb = (\tbox, \abox)$ be an $\ALCHIO$ KB and $q$ a BCQ.
	Let $\Vmc \subseteq \tbox$ be a subset of $\tbox$ and $\Imc$ an interpretation such that $\Imc \not\models q$.
	There exists a polynomial $p$ independent of $\abox$, and an interpretation $\Jmc$ satisfying the following:
	\smallskip
	
	\noindent
	\begin{minipage}{.215\textwidth}
		1. $\Jmc \not\models q$;
		\smallskip
		
		2. $vio_\abox(\Jmc) = vio_\abox(\Imc)$;
	\end{minipage}%
	\begin{minipage}{.3\textwidth}
		3. $\forall \tau \in \Vmc, vio_\tau(\Jmc) \subseteq vio_\tau(\Imc)$;
		\smallskip
		
		4. $\sizeof{\Delta^\Jmc} \leq p(\sizeof{\abox} + \sizeof{vio_\Vmc(\Imc)})$.
	\end{minipage}
\end{lemmarep}
\begin{toappendix}	
	
	\mbox{ }\medskip
	
	Before proceeding to the proof of Lemma \ref{lemma:main-quotient}, we briefly recall some standard terminology and notations that we will need. 
	 An interpretation \Imc\ is a \emph{model} of a KB $\Kmc=(\Tmc,\Amc)$
	if it satisfies all inclusions in $\Tmc$ and assertions in $\Amc$.
	We say that a query
	$q$ is \emph{entailed} from a KB $\Kmc$, written $\Kmc \models q$, if 
	$\Imc \models q$
	for every model $\Imc$ of $\Kmc$.

	\paragraph{A partial normal form.}
	As the constructions in \cite{Maniere} work under the usual semantics of models (not tolerating any violations) of KBs in normal form,
	our first step will be to replace the initial $\ALCHIO$ KB $\K$ with an $\ALCHIO$ KB in normal form that is able to keep track of violations of $\kb$ while being an actual model.
	To do so, we rely on a minor adaptation of the normalization for $\ALCO$ WKBs proposed in \cite{longversion-costbased}.
	Their construction removes all concept inclusions with finite weights and introduces new concept
	names into the TBox to be able to keep track of violations of those concept inclusions; they also prune the initial ABox to keep only those 
	assertions that are satisfied in the interpretation $\I$. 
	It is not hard to verify that inverse roles do not affect their translation into normal form.
	However, we cannot adopt their trick to keep track of violations of role inclusions as $\ALCHIO$ does not support expressive enough role inclusions, \emph{e.g.}\ of the shape $\rstyle{r} \sqcap \lnot \rstyle{s} \sqsubseteq \rstyle{t}$ where $\rstyle{r}, \rstyle{s}, \rstyle{t} \in \NRpm$.
	To circumvent this, we simply distinguish, in the original KB $\kb := (\tbox \cup \tbox_\Hmc, \abox)$, between the TBox $\tbox$ that only contains CIs of $\ALCIO$ concepts, and the remaining part $\tbox_\Hmc$ of the TBox that contains all the RIs.
		
	We now recall the normalization procedure proposed in \cite{longversion-costbased}, phrased with inverse roles.
	Note that we drop their finite/infinite-cost-based distinction between parts of the WKB as we start from a KB without weights and simply work on violations.
	
	We recall that a standard $\ALCIO$ KB is in normal form if all of its TBox axioms are in one of the following forms:
	\begin{align*}
		\axiom{\top}{D} \qquad \axiom{A}{D} \qquad \axiom{\{a\}}{A}  \qquad 
		\axand
		\\
		\axexistsleft
		\qquad
		\axexistsright
		\qquad
		\axnotleft
		\qquad
		\axnotright
	\end{align*}
	where $\cstyle{A}, \cstyle{A}_1, \cstyle{A}_2, \cstyle{B} \in \NC$, $\cstyle{D} \in \NC \cup \{\{\istyle{a}\} \mid \istyle{a} \in \NI\}$, and $\rstyle{r} \in \NRpm$. 
	It is well known that an arbitrary $\ALCIO$ TBox can be transformed (possibly through introduction of new concept names)
	into an $\ALCIO$ TBox in normal form. As the transformation may possibly introduce new concept names, 
	the new TBox may not be equivalent to the original one, but it will be a conservative extension, 
	which is sufficient for our purposes.  
	
We proceed as follows to convert the $\ALCIO$ part $(\tbox, \abox)$ of the original KB $\kb = (\tbox \cup \tbox_\Hmc, \abox)$ into a KB in normal form:	
	\begin{enumerate}
		\item For every 
		$\tau = \cstyle{C} \sqsubseteq \cstyle{D} \in \T$, 
		let $\cstyle{A}_\tau$ be a fresh concept name (i.e. occurring neither in $\K$ nor $q$), and set $$\T_\tau= \{\cstyle{A}_\tau \sqsubseteq \cstyle{C} \sqcap \neg \cstyle{D}, \cstyle{C} \sqcap \neg \cstyle{D} \sqsubseteq \cstyle{A}_\tau\}$$
		
		\item Apply the normalization procedure to the TBox 
		$$\T^{vio}  = \bigcup_{\tau \in \T} \T_\tau$$ 
		and let $\Tnew$ denote the resulting TBox. We can assume w.l.o.g.\ that all concept names in $\signatureof[\Tnew] \setminus \signatureof[\T^{vio} ]$
		occur neither in $q$ nor $\K$. 
	\end{enumerate} 
	We set $\Knew = \langle \Tnew, \Anew \rangle$, where $$\Anew = \{\alpha \in \A \mid \I \models \alpha\} \cup \{A_\top(\istyle{a}) \mid \istyle{a} \in \indsof{\A}\}$$
	with $\cstyle{A}_\top$ a fresh concept, not in $\Tnew$ nor $q$.
	Note that the assertions $\istyle{A}_\top(\istyle{a})$ simply serve to ensure that no individuals from $\K$
	are lost during the translation (which will be convenient for the later constructions and proofs).

	The following lemma summarizes the properties of $\Knew$.  Items (ii)-(iv) basically correspond to 
	$\Tnew$ being a conservative extension of $\T^{vio} $, while (v) is a direct consequence 
	of (iii) and the definition of $\T^{vio} $. 
	
	\begin{lemma}\label{newkb}
		The KB $\Knew$ satisfies the following:
		\begin{enumerate}[label=(\roman*),leftmargin=25pt]
			\item $\inds(\K) = \inds(\Knew)$
			\item 
			$\signatureof[\T] \subseteq \signatureof[\T^{vio} ] \subseteq \signatureof[\Tnew]$;
			\item 
			every model of $\Tnew$ is a model of $\T^{vio} $;
			\item
			for every model $\I_1$ of $\T^{vio} $, there exists a model $\I_2$ of $\Tnew$ 
			with $\Delta^{\I_1}=\Delta^{\I_2} $ such that 
			$\cdot^{\I_2}$ and $\cdot^{\I_1}$ coincide on all concept and role names except those in $\signatureof[\Tnew] \setminus \signatureof[\T^{vio} ]$;
			\item for every $\tau \in \T$ and every model $\I'$ of $\Knew$,
			$\vio{\tau}{\I'} = A_{\tau}^{\I'}$.
		\end{enumerate}
	\end{lemma}

	Using the preceding lemma, we can modify our original interpretation $\I$
	to get a model of the KB $\Knew$, 
	that violates the RIs from $\tbox_\Hmc$ exactly as $\I$ does:
	
	\begin{lemma}\label{norm-model}
		There exists a model $\I^\dagger$ of $\Knew$ with $\Delta^{\I^\dagger} = \Delta^\I$
		such that $\cdot^{\I^\dagger}$ and $\cdot^{\I}$ coincide on $\signatureof[\K] \cup \signatureof[q]$. 
		In particular, this means that:
		\begin{itemize}
			\item $\vio{\A}{\I^\dagger} = \vio{\A}{\I}$
			\item for every  $\tau \in \T$: $\vio{\tau}{\I^\dagger} = \vio{\tau}{\I}$
			\item for every  $\tau \in \T_\Hmc$: $\vio{\tau}{\I^\dagger} = \vio{\tau}{\I}$
			\item $\I^\dagger \not \models q$
		\end{itemize}
	\end{lemma}
	\begin{proof}
		We first create an interpretation $\I^\diamond$ which is the same as $\I$ except for the 
		concept names $\cstyle{A}_\tau \in  \signature[\T^{vio} ] \setminus \signatureof[\T]$, for which we 
		set $(\cstyle{A}_\tau)^{\I^\diamond}=(C \sqcap \neg D)^\I$, and the concept name $\cstyle{A}_\top$
		which we interpret as $\indsof{\A}$ (note that we may assume $\indsof{\A} \subseteq \Delta^\I$ due to the SNA on ABox individuals). 
		This ensures that $\I^\diamond$ satisfies the axioms in $\T_\tau$ and the new assertions in $\Anew$,
		and since we have not modified the interpretation of any other symbols, 
		$\I^\diamond$ 
		is a model of $ \langle \T^{vio} , \Anew \rangle$. 
		Hence, by Lemma \ref{newkb} (iv), there exists a model $\I^\dagger$ of $\Tnew$
		which can be obtained from $\I^\diamond$ solely by changing the interpretation of concept names in 
		$\signatureof[\Tnew] \setminus \signatureof[\T^{vio} ]$. In particular, $\I^\dagger$ is a model of $\Anew$
		and hence of the KB $\Knew$.  It follows from the construction of $\I^\dagger$ that
		$\Delta^{\I^\dagger} = \Delta^\I$ and that $\cdot^{\I^\dagger}$ and $\cdot^{\I}$ coincide on $\signatureof[\K] \cup \signatureof[q]$.
		The first, third, and fourth items follow immediately (for the third item, this is because $\cdot^{\I^\dagger}$ and $\cdot^{\I}$ coincide on
		$\signatureof[\T_\Hmc] \subseteq \signatureof[\K]$ and for the fourth item, recall that $\I \not \models q$). 
		The second item holds due to $(\cstyle{A}_\tau)^{\I^\dagger}=(\cstyle{A}_\tau)^{\I^\diamond}=(\cstyle{C} \sqcap \neg \cstyle{D})^\I$. 
	\end{proof}
	
	As the model $\I^\dagger$ of $\Knew$ given in Lemma \ref{lemma:new-prop-8} 
	satisfies the same key property as $\I$, i.e.\ not satisfying $q$,
	it can thus be used in place of $\I$ in the following constructions.

	\paragraph{Construction of the interlacing}
	With the normal form at hand, we now move towards the interlacing construction from \cite{Maniere}.
	The goal is to unravel the interpretation $\Idag$ to obtain a more well-structured interpretation
	that retains the essential properties of $\Idag$, i.e.\ it does not entail $q$ and satisfies Points~2 and 3 in Lemma~\ref{lemma:main-quotient}.
	The interlacing construction starts with the definition of the existential 
	extraction, which is a tree-shaped domain. 
	The original definition, which we adapt next,
	uses the alphabet $\Omega^\exists$ consisting of all $\rolestyle{r.A}$ such that $\existsrole{r.A}$ is the RHS of an axiom in  
	the considered TBox.
	In the context of classical models of a KB, only keeping track of a role $\rstyle{r}$ of interest is sufficient as the considered interpretations being models, one can always infer that all roles $\rstyle{s}$ such that $\rstyle{r} \sqsubseteq_{\tbox_\Hmc} \rstyle{s}$ will be satisfied whenever $\rstyle{r}$ is.
	To account for potentially violated RIs, we refine this alphabet by setting $\Omega$ to be the set of all $\rstyle{R.A}$ such that $\rstyle{R} \subseteq {\NRpm}$ and $\cstyle{A} \in \NC$.
	Intuitively, the subset of roles $\rstyle{R}$ will keep track of which combination of roles is satisfied.
	If $\rstyle{R} \subseteq \NRpm$, we write $d \in (\exists \rstyle{R.A})^\Idag$ if there exists an element $e \in \Delta^\Idag$, called an $\rolestyle{R.A}$-successor of $d$, such that:
	\begin{itemize}
		\item $e \in \cstyle{A}^\Idag$; and
		\item $\rstyle{R} = \{ \rolestyle{r} \in \NRpm \mid (d,e) \in \rstyle{r}^\Idag \}$.
	\end{itemize}
	Notice that this is in line with very expressive DLs that support more complex role constructors.
	Now, for every $\rolestyle{R.A} \in \Omega$, we assume chosen a function $\successor[\Idag]_{\rolestyle{R.A}}$ that maps every element $d \in \Delta^\Idag$ that has a $\rolestyle{R.A}$-successor to one of its $\rolestyle{R.A}$-successors.
	
	\begin{definition}[Existential extraction]
		\label{def:existential-extraction}
		Build the following mapping inductively over the set $\indsof{\Knew} \cdot \Omega^*$:
		\[
		\begin{array}{rcl}
			\exextomod :
			\indsof{\Knew} \cdot \Omega^* 	& \rightarrow 	& \domain{\Idag} \cup \{ \uparrow \}
			\\
			a 							& \mapsto 		& a \\
		\end{array}
		\]
		\[
		\begin{array}{rcl}
			w \cdot \rolestyle{R.A} 			& \mapsto 		& \left\{ \begin{array}{ll}
				\uparrow 								& \textrm{if }  \exextomodof{w} = \; \uparrow 
				\textrm{ or } \exextomodof{w} \notin (\existsrole{R.A})^{\Idag}
				\\
				\successorof[\Idag]{\exextomodof{w}}{R.A} 	& \textrm{otherwise}
			\end{array} \right.
		\end{array}
		\]
		where $\uparrow$ is a fresh symbol witnessing the absence of a proper image for an element of $\indsof{\Knew} \cdot \Omega^*$.
		The \emph{existential extraction} of $\Idag$ is $\exexof{\I} := \{ w \mid w \in \indsof{\Knew} \cdot \Omega^*, \exextomodof{w} \neq \; \uparrow \}$. 
	\end{definition}
	
	Apart from the extended notion of roles, which are now sets of roles, the preceding definition differs from the one in \cite{Maniere} as it uses the set $\indsof{\Knew}$ of individuals in the considered KB 
	rather than the set of individuals in the considered ABox. This is because we work with $\ALCHIO$ whereas the original 
	construction was formulated for $\mathcal{ALCHI}$. 
	Fortunately, although nominals were not considered in \cite{Maniere},
	they are properly handled by the constructions and do not require any notable modifications.
	
	The original construction allows one to select a subdomain $\Delta^*$ of the starting interpretation $\Imc$ whose elements should not be duplicated when manipulating $\Imc$.
	This is where the parameter $\Vmc$ in Lemma~\ref{lemma:main-quotient} comes into play: to guarantee Points~2 and 3, we should avoid duplicating violations of ABox assertions and of axioms from $\Vmc$.
	We thus use the following definition of $\deltastar$:
	\[
	\deltastar := \individuals(\kb) \cup \bigcup_{\tau \in \Vmc} {vio}^{d}_{\tau}(\Idag),
	\]
	where $vio^d_\tau(\Idag)$ refers to the \emph{domain} where the violations of $\tau$ occurs: if $\tau$ is a CI, then $vio^d_\tau(\Idag) := vio_\tau(\Idag)$, otherwise $\tau$ is an RI and $vio^d_\tau(\Idag) := \bigcup_{(e, e') \in vio_\tau(\Idag)} \{ e, e'\}$.

	%
	We can now introduce the function $f^*$, defined exactly as in \cite{Maniere} but using the modified version of $\deltastar$: 
	\[
	\begin{array}{cccl}
		\interlace : &
		\domain{\unfoldingof{\I}} 	& \rightarrow 	& \deltastar \uplus (\domain{\unfoldingof{\I}} \setminus \deltastar)
		\\[3pt] &
		w 							& \mapsto		& 	\left\{ \begin{array}{ll}
			\exextomodof{w}
			& \textrm{if } \exextomodof{w} \in \deltastar
			\\
			w 
			& \textrm{otherwise}
		\end{array} \right.
	\end{array}
	\]
	With these notions in hand, we are ready to recall the definition of the interlacing of $\Idag$ w.r.t.\ the function $f^*$:
	
	\begin{definition}[$\interlace$-interlacing]
		The \emph{$\interlace$-interlacing} $\interlacingofstar{\I}$ of $\Idag$ is the interpretation whose domain is 
		$\domain{\interlacingofstar{\I}} := \interlaceof{\exexof{\I}}$ and which interprets concept and role names as follows:
		\begin{align*}
			\rolestyle{A}^{\interlacingofstar{\I}} ~ := ~ &
			\{ \interlaceof{u} \mid u \in \exexof{\I}, \exextomodof{u} \in \rolestyle{A}^{\Idag} \}
			\\
			\rolestyle{r}^{\interlacingofstar{\I}} ~ := ~ &
			\{ (a, b) \mid a, b \in \indsof{\Knew} \wedge (a, b) \in \rstyle{r}^\Imc \}
			\\ & \cup ~ 
			\{ (\interlaceof{u}, \interlaceof{u \cdot \rolestyle{R.B}}) \mid u, u \cdot \rolestyle{R.B} \in \exexof{\Idag} \wedge \rstyle{r} \in \rstyle{R} \}
			\\ & \cup ~ 
			\{ (\interlaceof{u \cdot \rolestyle{R.B}}, \interlaceof{u}) \mid u, u \cdot \rolestyle{R.B} \in \exexof{\Idag} \wedge \rstyle{r}^- \in \rstyle{R} \}.
		\end{align*}
	\end{definition}
	The preceding definition is a slightly modified version of Definition 20 from \cite{Maniere} as we read from the sets of roles $\rstyle{R}$ which role should hold between two elements, rather than systematically adding the super-roles of a specific role.
	%
		We also make a slight modification -- using $\indsof{\Knew}$ rather than only individuals from the ABox  -- to make it compatible with nominals.
		
		We have phrased the definition directly in terms of our desired function $f^*$, but other functions $f'$ with domain $\exexof{\I}$ can be used instead. 
		Depending on which $f'$ is used to define the interlacing, 
		the resulting interpretation $\I'$ may or may not be a model of the considered KB. 
		It was shown in \cite{Maniere} that if $f'$ is pseudo-injective, then the $f'$-interleaving is a model
		and moreover maps homomorphically into the starting interpretation. 
		This property (stated in Lemma \ref{thm:meta-interlacing-is-model} below, phrased for our KB $\Knew$) was proven for $\mathcal{ALCHI}$ KBs in normal form, 
		but is easily shown to also hold for
		$\ALCIO$ KBs in normal form. 
		In addition, it can be verified that our more refined interpretation of roles does not affect the homomorphism, but instead strengthens the connections between the interlacing and the starting interpretation $\Idag$, as will be clarified in Lemma~\ref{id-inds}.
		
		\begin{definition}
			A function $f' : \exexof{\I} \rightarrow E$ is \emph{pseudo-injective}\index{pseudo-injective} if: for all $u, v \in \exexof{\I}$, if $f'(u) = f'(v)$, then $\exextomodof{u} = \exextomodof{v}$. 
		\end{definition}
		
		\begin{lemma}
			\label{thm:meta-interlacing-is-model}
			If $f': \exexof{\I} \rightarrow E$ is pseudo-injective, then the $f'$-interlacing $\interlacingof{\I}$ is a model of $\Knew$ and the following mapping is a homomorphism from $\interlacingof{\I}$ to $\Idag$:
			\[
			\begin{array}{cccl}
				\intertomod : &
				\domain{\interlacingof{\I}} 	& \rightarrow 	& \domain{\Idag} 
				\\ &
				f'(u) 				& \mapsto		& f(u) 
			\end{array}
			\]
			As $f'$ is pseudo-injective, $\intertomod$ is well defined.
		\end{lemma}
		
		It was proven in \cite{Maniere} that $f^*$ is pseudo-injective,
		and the same arguments apply to
		our modified $f^*$. 
		
		\begin{lemma}
			$f^*$ is pseudo-injective. 
		\end{lemma}
		
		It follows from the preceding lemmas that $\interlacingofstar{\I}$ is a model of $\Knew$ which maps 
		homomorphically into $\Idag$ via the mapping~$\sigma$. 
		In fact, due to the way $f^*$ is defined, we can be more precise about the homomorphism $\sigma$:
		
		\begin{lemma}\label{id-inds}
			The homomorphism $\sigma$ from $\interlacingofstar{\I}$ to $\Idag$
			is such that $\sigma(\istyle{a}) = \istyle{a}$ for all $\istyle{a} \in \indsof{\Knew}$. 
			Furthermore, for all role $\rstyle{r} \in \NRpm$ and $d, e \in \Delta^{\interlacingofstar{\Imc}}$, we have:
			\[
			(d, e) \in \rstyle{r}^{\interlacingofstar{\Imc}} \text{ iff } (\sigma(d), \sigma(e)) \in \rstyle{r}^{{\Idag}} 
			\]
		\end{lemma}

		We are now ready to show that $\interlacingofstar{\I}$ retains the desired properties of $\Idag$ (and thus of $\I$, by virtue of Lemma~\ref{norm-model}). 
		\begin{lemma} 
			\label{interleaving-cost}
			$\interlacingofstar{\I}$ satisfies Points~1, 2 and 3 of Lemma~\ref{lemma:main-quotient}, that is:
			\begin{enumerate}
				\item $\interlacingofstar{\I} \not \models q$;
				\item $vio_\abox(\interlacingofstar{\I}) = vio_\abox({\I^\dagger})$;
				\item $\forall \tau \in \Vmc$, $vio_\tau(\interlacingofstar{\I}) \subseteq vio_\tau({\I^\dagger})$.
			\end{enumerate} 
		\end{lemma}
		\begin{proof}
			For Point~1, recall that Point 4 from Lemma~\ref{norm-model} guarantees $\Idag \not \models q$.
			The existence of the homomorphism $\sigma$ from $\interlacingofstar{\I}$ 
			to $\Idag$ that is the identity on $\indsof{\Knew}$ (see Lemmas \ref{thm:meta-interlacing-is-model} and \ref{id-inds}), 
			hence also on $\indsof{q}$,  
			means that any match of $q$ in $\interlacingofstar{\I}$ can be reproduced in $\Idag$. 
			Therefore, from $\Idag \not \models q$, we immediately obtain $\interlacingofstar{\I} \not \models q$.
			
			For Point~2, note that $vio_\abox(\interlacingofstar{\I}) = vio_\abox(\Idag)$ is immediate from the definition of $\interlacingofstar{\Imc}$.

			For Point~3, let $\tau \in \Vmc$.
			We first treat the case of $\tau$ being a CI.
			Let $e \in vio_\tau(\interlacingofstar{\I})$.
			Using Point~(v) of Lemma~\ref{newkb} and $\interlacingofstar{\I}$ being a model of $\Knew$, we obtain $e \in \cstyle{A}_\tau^{\interlacingofstar{\I}}$.
			By definition of $\cstyle{A}_\tau^{\interlacingofstar{\I}}$, there exists $d \in \exexof{\Idag}$ such that $\interlaceof{d} = e$ and $f(d) \in \cstyle{A}_\tau^{\Idag}$.
			Again using Point~(v) of Lemma~\ref{newkb} but this time on $\Idag$ being a model of $\Knew$,
			we have $f(d) \in vio_\tau({\Idag})$.
			As $vio_\tau(\Idag) \subseteq \deltastar$, we get $f^*(d) = f(d)$ by definition of $f^*$, thus $e = f^*(d) = f(d) \in vio_\tau({\Idag})$ as desired.
			We turn to $\tau$ being a RI.
			Let $(d, e) \in vio_\tau(\interlacingofstar{\I})$.
			By Lemma~\ref{id-inds}, applied on both roles occurring in the RI $\tau$, we obtain $(\sigma(d), \sigma(e)) \in vio_{\tau}(\Idag)$.
			As $vio_\tau(\Idag) \subseteq \deltastar \times \deltastar$, we get $\sigma(d), \sigma(e) \in \deltastar$, thus $\sigma(d) = d$ and $\sigma(e) = e$.
			We thus obtain $(d, e) \in vio_\tau({\Idag})$ as desired.
		\end{proof}

		\paragraph{Quotient construction.} 
		It now remains to `shrink' the  interpretation $\interlacingofstar{\I}$ to obtain an interpretation 
		with the same properties but having the required size. 
		To do so, we can proceed exactly as in 
		Chapter 3.4 of \cite{Maniere}, by first defining a suitable equivalence relation, 
		then considering the quotient interpretation obtained by merging elements that belong to the same equivalence class.
		
		We will now describe how the equivalence relation $\sim_n$ is defined, but without giving every detail as
		the definition is rather involved and we will not be modifying it except by  
		using our own definition of $\deltastar$.
		The definition of $\sim_n$ involves the notion of the $n$-neighbourhood of an element $d$ in $\I^*$ relative to $\deltastar$,  denoted  $\mathcal{N}^{\I^*,\deltastar}_n(d)$,
		consisting of the  elements in $\Delta^{\I^*}$ that can be reached by taking at most $n$ `steps' along role edges, starting at $d$ 
		and stopping whenever an element of $\deltastar$ is reached. 
		If $d\in \deltastar$, then $d$ itself is the only element in its $n$-neighbourhood (for any $n$). 
		However, when $d \in \Delta^{\I^*} \setminus \Delta^*$,
		we know that $d = a w$ for some $a \in \indsof{\Knew}$ and $w \in \Omega^*$. 
		Using the tree-shaped structure of the domain $\domain{\unfoldingof{\I}}$, 
		we can identify a unique `root' prefix $r_{n,d}$ of $d=a w $
		such that $f^*(r_{n,d}) \in \mathcal{N}^{\I^*,\deltastar}_n(d)$ and 
		for every $d' \in \mathcal{N}^{\I^*,\deltastar}_n(d)$, there is a unique word 
		$w^{d'}_{n,d} \in \Omega^*$ (note that $\Omega$ now refers to a full combination of roles) such that $d'=f^*(r_{n,d} \cdot w_{n,d}^{d'})$ and $|w^{d'}_{n,d}| \leq 2n$. 
		With this uniform way of refer to the elements in neighbourhood of a considered element $d$, 
		we can define a function $\chi_{n,d}$ whose output tells us for each word $w \in \Omega^*$ with $|w| \leq 2n$ 
		whether there is an element in the neighbourhood whose word is $w$, 
		and if so, whether that element belongs to $\deltastar$, and if not, which concept names it satisfies in $\I^*$. 
		The equivalence relation $\sim_n$ then groups together those elements $d,e \in \Delta^{\I^*} \setminus \Delta^*$
		which have the same associated word (i.e.\ $w^{d}_{n,d} = w^{e}_{n,e}$) and same associated function ($\chi_{n,d}=\chi_{n,e}$), 
		plus an additional condition on the length of $d$ and $e$ (namely, $|d| = |e| \textrm{ mod } 2|q|+3$). 
		
		With the equivalence relation $\sim_n$ at hand, it now 
		suffices to merge elements which are equivalent w.r.t.\ $\sim_{|q|+1}$.
		Formally, we consider the quotient interpretation $\J:= \I^* / \sim_{|q|+1}$
		whose domain is $\Delta^\J= \{[e]_{\sim_{|q|+1}} \mid e \in \Delta^{\I^*}\}$ 
		and whose interpretation function 
		$\cdot^\J$ is as follows:
		\[ 
		a^\J= [a^{\I^*}]_{\sim_{|q|+1}}
		\qquad\qquad
		A^\J = \{[e]_{\sim_{|q|+1}} \mid e \in A^{\I^*}\}
		\qquad\qquad
		r^\J = \{([d]_{\sim_{|q|+1}},[e]_{\sim_{|q|+1}}) \mid (d,e) \in r^{\I^*}\}
		\]
		where  $[e]_{\sim_{|q|+1}}$ is the equivalence class of $e$ w.r.t.\ 
		$\sim_{|q|+1}$.

		It has been shown that $\J$ remains a model of the considered KB and does not contain
		any additional query matches. Given that the interpretation $\I^*$ we consider here is built in exactly 
		the same way as the interleaving $\I'$ considered in \cite{Maniere}, except that we work 
		with a larger domain of interest $\deltastar$ (which includes the one used to build $\I'$),
		exactly the same arguments can be used to show Lemma \ref{quotient-model} below. 
		We point out that the presence of nominals in $\Knew$ is not problematic 
		as all of its individuals (whether present in the ABox or TBox) are included in $\deltastar$ and are therefore left
		untouched by the construction. 
		
		\begin{lemma}\label{quotient-model}
			The interpretation $\J$ is a model of $\Knew$ such that $\J \not \models q$; in particular it satisfies Point~1 of Lemma~\ref{lemma:main-quotient}.
		\end{lemma}
		
		It remains to show that the quotient operation preserves the intended violations of assertions in  $\abox$  and of axioms from $\Vmc$. 
		\begin{lemma}
			$\Jmc$ satisfies Points~2 and 3 from Lemma~\ref{lemma:main-quotient}.
		\end{lemma}
		\begin{proof}
			Notice that $\Jmc$ and $\interlacingofstar{\Imc}$ coincide on $\deltastar$.
			Using Point~2 from Lemma~\ref{interleaving-cost} then yields the desired Point~2 for $\Jmc$.

			For Point~3, we distinguish between CIs from $\Vmc$ and RIs from $\Vmc$.
			For a CI $\tau \in \Vmc$, we use $\Jmc$ being a model of $\Knew$ joint with Point~(v) of Lemma~\ref{newkb}, and the observation that $\Jmc$ coincides with $\interlacingofstar{\Imc}$ on $\deltastar$, to obtain that $vio_\tau(\Jmc) = vio_\tau(\Idag)$.
			Now using Point~3 from Lemma~\ref{interleaving-cost} gives the desired result.
			For a RI $\rstyle{r} \sqsubseteq \rstyle{s}$, consider a pair $(d, e) \in \rstyle{r}^\Jmc$ such that $(d, e) \notin \rstyle{s}^\Jmc$.
			The definition of $\rstyle{r}^{\Jmc}$ gives a pair $(d_0, e_0) \in \rstyle{r}^{\interlacingofstar{\Imc}}$ such that $[d_0]_{\sim_{|q|+1}}=d$ and $[e_0]_{\sim_{|q|+1}} = e$, and the definition of $\rstyle{s}^\Jmc$ guarantees that $(d_0, e_0) \notin \rstyle{s}^{\interlacingofstar{\Imc}}$ (as otherwise $(d, e)$ would be in $\rstyle{s}^\Jmc$).
			It then follows from Lemma~\ref{interleaving-cost} that $d_0, e_0 \in \deltastar$.
			Using the observation that $\Jmc$ and $\interlacingofstar{\Imc}$ coincide on $\deltastar$, this yields $(d, e) \in vio_{\rstyle{r} \sqsubseteq \rstyle{s}}(\interlacingofstar{\Imc})$, and we conclude using Point~3 of Lemma~\ref{interleaving-cost}.
		\end{proof}
		
		To complete the proof of Lemma~\ref{lemma:main-quotient}, we now 
		examine the steps in the construction in order to place a bound on 
		the size of the interpretation $\J$. 
		By analyzing the number of equivalence classes, the following 
		upper bound on $|\Delta^\J|$ was shown in \cite{Maniere}:
		\begin{align*}
			(2|q|+3) \times |\T|^{|q|+2} \times (|\deltastar| + 2^{\signatureof[\T]} + 1)^{|\T|^{2|q|+2}}  
		\end{align*}
		Their analysis however relied on the size of the alphabet $\Omega$ being bounded by $\tbox$.
		In our case, $\Omega$ suffers from an exponential blow-up.
		To derive a naive upper bound in our setting, we simply replace each occurrence of $\sizeof{\tbox}$ in the above by $2^\sizeof{\tbox}$.
		Observe that the resulting upper bound remains a polynomial in $\Delta^*$, hence a polynomial in $\sizeof{\individuals(\kb)} + \sizeof{vio^d_\Vmc(\Imc)}$.
		The latter term is upper bounded by $2 \sizeof{vio_\Vmc(\Imc)}$, and using $\sizeof{\individuals(\kb)} \leq \sizeof{\tbox} + \sizeof{\abox}$, one obtains the desired polynomial in $\sizeof{\abox} + \sizeof{vio_\Vmc(\Imc)}$ whose coefficients are positive and independent of $\abox$.
	\end{toappendix}

We now explain how to obtain Lemma~\ref{lemma:new-prop-8},
with an approach inspired from \cite{LMKR24}, where our Lemmas~\ref{lemma:new-prop-8} and \ref{lemma:main-quotient} respectively play the role of their Proposition~2 and Lemma~1.
Let $\kb_\omega = (\tbox, \abox)_\omega$ be an $\ALCHIO$ WKB, $q$ a BCQ, $k$ an integer, and $\Imc$ an interpretation such that $\Imc \not\models q$.
For a given $\Vmc \subseteq \tbox$, we use $\Jmc_\Vmc$ to denote the interpretation obtained by applying Lemma~\ref{lemma:main-quotient} with $\Vmc$ the input set of axioms. 
We prove that there exists a subset $\Vmc \subseteq \tbox$ such that (i) the size of ${vio_\Vmc(\Imc)}$ is bounded by a polynomial in $\sizeof{\abox}$ independent of $k$; and (ii) $\omega(\Jmc_\Vmc) \leq k$.
To do so, we construct a sequence $\Vmc_0 \subsetneq \Vmc_1  \subsetneq \dots \subsetneq \Vmc_n \subseteq \tbox$ of $\Vmc$'s that all satisfy item (i) and with $\Vmc_n$ also satisfying item (ii).
Note that $\Jmc_{\Vmc_n}$ is then the desired interpretation for Lemma~\ref{lemma:new-prop-8}: item (i) plus Point~4 from Lemma~\ref{lemma:main-quotient} gives the polynomial bound on the size of $\Jmc_{n}$, while item (ii) and Point~1 in Lemma~\ref{lemma:main-quotient} ensure the desired properties w.r.t.\ the cost and query.\smallskip\\
\noindent\emph{Initialization}. 
Set $\Vmc_0 := \emptyset$, which trivially satisfies item (i).\smallskip\\
\noindent\emph{Induction step}.
Assume that, for some $i \geq 0$, we have successfully constructed $\Vmc_i$ satisfying item (i); thus we have a polynomial $p_i$ independent of $k$ such that $\sizeof{vio_\Vmc(\Imc)} \leq p_i(\sizeof{\abox})$.
If $\Vmc_i$ also satisfies item (ii), then we are done.
Otherwise $\omega(\Jmc_{\Vmc_i}) > k$, and since $\omega(\Imc) \leq k$, there exists an assertion or axiom $\tau$ from $\kb$ that is violated at least once more in $\Jmc_{\Vmc_i}$ than in $\Imc$, \emph{i.e.}\ $\sizeof{vio_\tau(\Jmc_{\Vmc_i})} > \sizeof{vio_\tau(\Imc)}$.
Note that due to Point~2 in Lemma~\ref{lemma:main-quotient}, it is then clear that $\tau \notin \abox$.
Similarly, due to Point~3 in Lemma~\ref{lemma:main-quotient}, we have $\tau \notin {\Vmc_i}$.
Therefore $\tau \in \tbox \setminus \Vmc_i$.
We set $\Vmc_{i + 1} := \Vmc_i \cup \{ \tau\}$.
It remains to verify that $\Vmc_{i+1}$ satisfies item (i).
Note that $vio_{\Vmc_{i+1}}(\Imc) = vio_{\Vmc_{i}}(\Imc) \cup vio_\tau(\Imc)$.
The size of $vio_{\Vmc_{i}}(\Imc)$ is bounded adequately by $p_i(\sizeof{\abox})$.
For the size of $vio_\tau(\Imc)$, recall that by choice of $\tau$ we have $\sizeof{vio_\tau(\Imc)} < \sizeof{vio_\tau({\Jmc_{\Vmc_i}})}$.
We brutally bound $\sizeof{vio_\tau(\Jmc_{\Vmc_{i}})}$ by $\sizeof{\Delta^{\Jmc_{\Vmc_i}}}^2$.
By Point~4 in Lemma~\ref{lemma:main-quotient}, $\Delta^{\Jmc_{\Vmc_i}}$ has size bounded by $p(\sizeof{\abox} + \sizeof{vio_{\Vmc_i}(\Imc)})$, thus by $p(\sizeof{\abox} + p_i(\sizeof{\abox}))$.
Overall, the size of $vio_{\Vmc_{i+1}}(\Imc)$ is bounded by $p_{i+1}(\sizeof{\abox})$ where $p_{i+1}(x) := p_i(x) + (p(x + p_i(x)))^2$ is the desired polynomial independent of $k$.

Note that this procedure is guaranteed to terminate in at most $\sizeof{\tbox}$ steps, which concludes the proof.

	\section{Lower Bounds}\label{lower}
	We first refine some existing lower bounds for the fixed-cost decision problems in $\ELbot$,
	then prove the lower bounds for DL-Lite listed in Table \ref{table:results}. 
	
	\subsection{Lower Bounds in Extensions of $\ELbot$}
	\label{subsection:lower-el}
	\begin{toappendix}
		\subsection*{Proofs for Section~\ref{subsection:lower-el} (Lower Bounds in Extensions of $\ELbot$)}
	\end{toappendix}
	
	We begin with two lower bounds showing that even if the cost $k$ is fixed to $1$, all considered reasoning tasks are $\NP$-complete (or $\coNP$-complete, depending on the task) already for $\ELbot$ WKBs.
	These results notably improve those from \citex{BBJKR24}, where the cost was fixed to $3$.
	Lemma~\ref{remark:k-to-k+1} lifts our hardness proof to any fixed $k \geq 1$, and since the case of fixed $k = 0$ coincides with the usual semantics of $\EL$ KBs, the complexity w.r.t.\ fixed cost is now well understood.
	\begin{theoremrep}
		\label{theorem:lower-bound-elbot-bcs-fixed-cost-1}
		For every $k \geq 1$, BCS$^k$, IQA$_p^k$ and CQA$_p^k$ for $\ELbot$ are $\NP$-hard.
	\end{theoremrep}

	\begin{proofsketch}
		The reduction is from \tsat.
		Given a 3-CNF formula
		$\phi := \bigwedge_{i = 1}^\ell \bigvee_{j = 1}^3 l_{i, j}$, where each $l_{i, j}$ is a literal over $v_1, \dots, v_n$, 
		we construct a WKB $(\tbox, \abox_\phi)_\omega$. The ABox $\abox_\phi$ contains $\cstyle{False}(\istyle{a})$, 
		$\cstyle{Bool}(\istyle{a})$, and the additional assertions:
		\begin{align*}
			\cstyle{Var}(v_k) & \text{ for } 1 \leq k \leq n ~~~~
			&
			\rstyle{clause}(\istyle{a}, c_i) & \text{ for }  1 \leq i \leq \ell
			\smallskip\\
			\rstyle{pos}_j(c_i, v_k) & \text{ for }   l_{i, j} = v_k
			&
			\rstyle{neg}_j(c_i, v_k) & \text{ for } l_{i, j} = \lnot v_k
		\end{align*}
		The TBox $\tbox$ has the following axioms: 
		\[
		\hspace*{-1ex}\begin{array}{c}
			\cstyle{Bool} \sqsubseteq \cstyle{True}
			\quad 
			\cstyle{True} \sqcap \cstyle{False} \sqsubseteq \bot
			\quad
			\exists \rstyle{clause}.\cstyle{False}  \sqsubseteq \cstyle{True}
			\smallskip\\
			\cstyle{Var}  \sqsubseteq \exists \rstyle{val}.\cstyle{Bool}
			\quad
			\exists \rstyle{val}.\cstyle{True} \sqsubseteq \cstyle{True}
			\quad
			\exists \rstyle{val}.\cstyle{False} \sqsubseteq \cstyle{False} 
			\smallskip\\
			\exists \rstyle{pos}_1.\cstyle{False} \sqcap \exists \rstyle{pos}_2.\cstyle{False} \sqcap \exists \rstyle{pos}_3.\cstyle{False} \sqsubseteq \cstyle{False}
			\smallskip\\
			\textit{(the six others combinations...)}
			\smallskip\\
			\exists \rstyle{neg}_1.\cstyle{True} \sqcap \exists \rstyle{neg}_2.\cstyle{True} \sqcap \exists \rstyle{neg}_3.\cstyle{True} \sqsubseteq \cstyle{False}
		\end{array}
		\]
		The function $\omega$ assigns $\infty$ to all axioms and assertions, except for $\cstyle{Bool} \sqsubseteq \cstyle{True}$, which has weight $1$. 
		One can verify that $\phi$ is satisfiable iff $(\tbox, \abox_\phi)_\omega$ is $1$-satisfiable.
	\end{proofsketch}

	\begin{proof}
		The reduction is from \tsat.  Given a 3-CNF formula 
		$\phi := \bigwedge_{i = 1}^\ell \bigvee_{j = 1}^3 l_{i, j}$, where each $l_{i, j}$ is a literal over $v_1, \dots, v_n$, 
		we construct a WKB $(\tbox, \abox_\phi)_\omega$. The ABox $\abox_\phi$, encoding $\varphi$, contains the following assertions:
		\begin{align*}
			\cstyle{False}(\istyle{a})
			\\
			\cstyle{Bool}(\istyle{a})
			\\
			\cstyle{Var}(v) & \textrm{ for each variable } v
			\\
			\rstyle{clause}(\istyle{a}, c_i) & \textrm{ for each } 1 \leq i \leq \ell
			\\
			\rstyle{pos}_j(c_i, v) & \textrm{ for each } 1 \leq i \leq \ell \textrm{ and } 1 \leq j \leq 3 \textrm{ such that } l_{i, j} = v
			\\
			\rstyle{neg}_j(c_i, v) & \textrm{ for each } 1 \leq i \leq \ell \textrm{ and } 1 \leq j \leq 3 \textrm{ such that } l_{i, j} = \lnot v.
		\end{align*}
		The TBox $\tbox$ has the following axioms: 
		\[
		\begin{array}{c}
			\cstyle{Bool} \sqsubseteq \cstyle{True}
			\qquad
			\cstyle{True} \sqcap \cstyle{False} \sqsubseteq \bot 
			\qquad
			\cstyle{Var}  \sqsubseteq \exists \rstyle{val}.\cstyle{Bool} 
			\\[3pt]
			\exists \rstyle{val}.\cstyle{True}  \sqsubseteq \cstyle{True}
			\qquad
			\exists \rstyle{val}.\cstyle{False} \sqsubseteq \cstyle{False}
			\qquad
			\exists \rstyle{clause}.\cstyle{False}  \sqsubseteq \cstyle{True}
			\\[3pt]
			\exists \rstyle{pos}_1.\cstyle{False} \sqcap \exists \rstyle{pos}_2.\cstyle{False} \sqcap \exists \rstyle{pos}_3.\cstyle{False} \sqsubseteq \cstyle{False}
			\qquad
			\exists \rstyle{pos}_1.\cstyle{False} \sqcap \exists \rstyle{pos}_2.\cstyle{False} \sqcap \exists \rstyle{neg}_3.\cstyle{True} \sqsubseteq \cstyle{False}
			\\[3pt]
			\exists \rstyle{pos}_1.\cstyle{False} \sqcap \exists \rstyle{neg}_2.\cstyle{True} \sqcap \exists \rstyle{pos}_3.\cstyle{False} \sqsubseteq \cstyle{False}
			\qquad
			\exists \rstyle{pos}_1.\cstyle{False} \sqcap \exists \rstyle{neg}_2.\cstyle{True} \sqcap \exists \rstyle{neg}_3.\cstyle{True} \sqsubseteq \cstyle{False}
			\\[3pt]
			\exists \rstyle{neg}_1.\cstyle{True} \sqcap \exists \rstyle{pos}_2.\cstyle{False} \sqcap \exists \rstyle{pos}_3.\cstyle{False} \sqsubseteq \cstyle{False}
			\qquad
			\exists \rstyle{neg}_1.\cstyle{True} \sqcap \exists \rstyle{pos}_2.\cstyle{False} \sqcap \exists \rstyle{neg}_3.\cstyle{True} \sqsubseteq \cstyle{False}
			\\[3pt]
			\exists \rstyle{neg}_1.\cstyle{True} \sqcap \exists \rstyle{neg}_2.\cstyle{True} \sqcap \exists \rstyle{pos}_3.\cstyle{False} \sqsubseteq \cstyle{False}
			\qquad
			\exists \rstyle{neg}_1.\cstyle{True} \sqcap \exists \rstyle{neg}_2.\cstyle{True} \sqcap \exists \rstyle{neg}_3.\cstyle{True} \sqsubseteq \cstyle{False}
		\end{array}
		\]
		The function $\omega$ assigns $\infty$ to all axioms and assertions, except for $\cstyle{Bool} \sqsubseteq \cstyle{True}$, which has weight $1$.\medskip
		
		\noindent\textbf{Claim:}
		$\phi$ is satisfiable iff $(\tbox, \abox_\phi)_\omega$ is $1$-satisfiable.\smallskip
		
		\noindent$(\Leftarrow)$.
		Assume $\phi$ is satisfiable, that is, there exists a valuation $\nu$ of $v_1, \dots, v_n$ that makes $\phi$ true.
		We build an interpretation $\Imc$  of $(\tbox, \abox_\phi)_\omega$ according to $\nu$, with cost $1$.
		The domain of $\Imc$ is $\Delta^\Imc := \individuals(\abox_\phi) \cup \{ b \}$, where $b$ is a fresh domain element intended to represent the Boolean `true', as opposed to individual $\istyle{a}$ representing `false'.
		The interpretation $\Imc$ interprets concept $\cstyle{Var}$ and roles $\rstyle{clause}$, $\rstyle{pos}_j$, $\rstyle{neg}_j$ as specified in the ABox.
		The remaining predicates are interpreted as follows:
		\begin{align*}
			\cstyle{Bool}^\Imc := ~ & 
			\{ \istyle{a}, b \}
			\\
			\cstyle{False}^\Imc := ~ & 
			\{ \istyle{a} \} \cup \{ v_k \mid 1 \leq i \leq n, \nu(v_k) = 0 \} 
			\\
			\cstyle{True}^\Imc := ~ & 
			\{ b \} \cup \{ v_k \mid 1 \leq k \leq n, \nu(v_k) = 1 \} 
			\\
			\cstyle{val}^\Imc := ~ & 
			\{ (v_k, \istyle{a}) \mid 1 \leq k \leq n, \nu(v_k) = 0 \}  \cup
			\{ (v_k, b) \mid 1 \leq k \leq n, \nu(v_k) = 1 \} 
		\end{align*}
		It now suffices to verify that $\omega(\Imc) = 1$.
		All assertions from $\abox_\phi$ are satisfied.
		Regarding axioms from $\tbox$, notice that $\istyle{a}$ violates the concept inclusion $\cstyle{Bool} \sqsubseteq \cstyle{True}$, that has cost $1$.
		The other axioms of $\tbox$ are all satisfied; in particular, since $\nu$ satisfies $\phi$, it satisfies all its clauses, and thus none of the $c_i$'s needs to satisfy the concept $\cstyle{False}$.\smallskip
		
		\noindent$(\Rightarrow)$.
		Assume $(\tbox, \abox_\phi)_\omega$ is $1$-satisfiable.
		Consider an interpretation $\Imc$ with cost at most $1$.
		Since assertion $\cstyle{False}(\istyle{a})$ and axiom $\cstyle{True} \sqcap \cstyle{False}$ both have infinite cost, it is immediate that $\istyle{a} \notin \cstyle{True}^\Imc$.
		From $\cstyle{Bool}(\istyle{a}) \in \abox_\phi$ having infinite cost and $\cstyle{Bool} \sqsubseteq \cstyle{True} \in \tbox$ having cost $1$, the only violation in $\Imc$ is of the latter axiom, due to $\istyle{a} \in \cstyle{Bool}^\Imc \setminus \cstyle{True}^\Imc$. 
		In particular, all other axioms of $\tbox$ are satisfied by $\Imc$, and we also have $\cstyle{Bool}^\Imc \setminus \{ \istyle{a}\} \subseteq \cstyle{True}^\Imc$.
		Therefore $\Imc$ defines a valuation $\nu$ for variables $v_1, \dots, v_n$: for each $1 \leq k \leq n$, we set $\nu(v_k) := 0$ if $(v_k, a) \in \rstyle{val}^\Imc$ (in which case $v_k \in (\exists \rstyle{val}.\cstyle{False})^\Imc$), and $\nu(v_k) = 1$ otherwise (in which case $v_k \in (\exists \rstyle{val}.\cstyle{True})^\Imc$).
		It is now readily verified that $\nu(\phi) = 1$ as otherwise we would have some $1 \leq i \leq \ell$ such that $c_i \in \cstyle{False}^\Imc$ (using one of the eight long axioms, depending on the shape of $c_i$), which would in turn yield $\istyle{a} \in \cstyle{True}^\Imc$ (via axiom $\exists \rstyle{clause}.\cstyle{False} \sqsubseteq \cstyle{True}$), that is a contradiction.
	\end{proof}
	
	For the case of IQA$_c^k$, we could use Lemma~\ref{lemma:possible-iq-to-certain-iq} to directly obtain a $\coNP$-hardness proof for 
	$k \geq 2$.
	We instead re-adapt the above proof to strengthen the result to every $k \geq 1$.
	Note that
	concept disjointness axioms are not even needed here.
	
	\begin{theoremrep}
		\label{theorem:lower-bound-el-certain-iq-fixed-cost-1}
		For every $k \geq 1$, IQA$^k_c$ and CQA$^k_c$ for $\EL$ are $\coNP$-hard.
	\end{theoremrep}
	
	\begin{proof}
		We reduce from the problem of deciding whether a 3-DNF formula $\phi := \bigvee_{i = 1}^\ell \bigwedge_{j = 1}^3 l_{i, j}$, where each $l_{i, j}$ is a variable from $v_1, \dots, v_n$ or its negation, is a tautology.
		The IQ is $q := \cstyle{True}(\istyle{a})$.
		The budget $k$ is fixed to $1$.
		The TBox has axioms:
		\[
		\begin{array}{c}
			\cstyle{Bool} \sqsubseteq \cstyle{True}
			\qquad
			\cstyle{Var}  \sqsubseteq \exists \rstyle{val}.\cstyle{Bool} 
			\\[3pt]
			\exists \rstyle{val}.\cstyle{True}  \sqsubseteq \cstyle{True}
			\qquad
			\exists \rstyle{val}.\cstyle{False} \sqsubseteq \cstyle{False}
			\qquad
			\exists \rstyle{clause}.\cstyle{True}  \sqsubseteq \cstyle{True}
			\\[3pt]
			\exists \rstyle{pos}_1.\cstyle{True} \sqcap \exists \rstyle{pos}_2.\cstyle{True} \sqcap \exists \rstyle{pos}_3.\cstyle{True} \sqsubseteq \cstyle{True}
			\qquad
			\exists \rstyle{pos}_1.\cstyle{True} \sqcap \exists \rstyle{pos}_2.\cstyle{True} \sqcap \exists \rstyle{neg}_3.\cstyle{False} \sqsubseteq \cstyle{True}
			\\[3pt]
			\exists \rstyle{pos}_1.\cstyle{True} \sqcap \exists \rstyle{neg}_2.\cstyle{False} \sqcap \exists \rstyle{pos}_3.\cstyle{True} \sqsubseteq \cstyle{True}
			\qquad
			\exists \rstyle{pos}_1.\cstyle{True} \sqcap \exists \rstyle{neg}_2.\cstyle{False} \sqcap \exists \rstyle{neg}_3.\cstyle{False} \sqsubseteq \cstyle{True}
			\\[3pt]
			\exists \rstyle{neg}_1.\cstyle{False} \sqcap \exists \rstyle{pos}_2.\cstyle{True} \sqcap \exists \rstyle{pos}_3.\cstyle{True} \sqsubseteq \cstyle{True}
			\qquad
			\exists \rstyle{neg}_1.\cstyle{False} \sqcap \exists \rstyle{pos}_2.\cstyle{True} \sqcap \exists \rstyle{neg}_3.\cstyle{False} \sqsubseteq \cstyle{True}
			\\[3pt]
			\exists \rstyle{neg}_1.\cstyle{False} \sqcap \exists \rstyle{neg}_2.\cstyle{False} \sqcap \exists \rstyle{pos}_3.\cstyle{True} \sqsubseteq \cstyle{True}
			\qquad
			\exists \rstyle{neg}_1.\cstyle{False} \sqcap \exists \rstyle{neg}_2.\cstyle{False} \sqcap \exists \rstyle{neg}_3.\cstyle{False} \sqsubseteq \cstyle{True}
		\end{array}
		\]
		The assertions are:
		\begin{align*}
			\cstyle{False}(\istyle{a})
			\\
			\cstyle{Bool}(\istyle{a})
			\\
			\cstyle{Var}(v) & \textrm{ for each variable } v
			\\
			\rstyle{clause}(\istyle{a}, c_i) & \textrm{ for each } 1 \leq i \leq \ell
			\\
			\rstyle{pos}_j(c_i, v) & \textrm{ for each } 1 \leq i \leq \ell \textrm{ and } 1 \leq j \leq 3 \textrm{ s.t. } l_{i, j} = v
			\\
			\rstyle{neg}_j(c_i, v) & \textrm{ for each } 1 \leq i \leq \ell \textrm{ and } 1 \leq j \leq 3 \textrm{ s.t. } l_{i, j} = \lnot v.
		\end{align*}
		The function $\omega$ assigns $\infty$ to all axioms and assertions, except for $\cstyle{Bool} \sqsubseteq \cstyle{True}$, which has weight $1$.\medskip
		
		\noindent\textbf{Claim:}
		$\phi$ is a tautology iff $(\tbox, \abox_\phi)_\omega \models^1_c \cstyle{True}(a)$.\smallskip
		
		\noindent$(\Leftarrow)$.
		Assume $\phi$ is not a tautology, that is, there exists a valuation $\nu$ of $v_1, \dots, v_n$ that makes $\phi$ false.
		We build an interpretation $\Imc$  of $(\tbox, \abox_\phi)_\omega$ according to $\nu$, with cost $1$, and such that $\istyle{a} \notin \cstyle{True}^\Imc$ (that is, $\Imc \not\models q$).
		The domain of $\Imc$ is $\Delta^\Imc := \individuals(\abox_\phi) \cup \{ b \}$, where $b$ is a fresh domain element intended to represent the Boolean `true', as opposed to individual $\istyle{a}$ representing `false'.
		The interpretation $\Imc$ interprets concept $\cstyle{Var}$ and roles $\rstyle{clause}$, $\rstyle{pos}_j$, $\rstyle{neg}_j$ as specified in the ABox.
		The remaining predicates are interpreted as follows:
		\begin{align*}
			\cstyle{Bool}^\Imc := ~ & 
			\{ \istyle{a}, b \}
			\\
			\cstyle{False}^\Imc := ~ & 
			\{ \istyle{a} \} \cup \{ v_k \mid 1 \leq i \leq n, \nu(v_k) = 0 \} 
			\\
			\cstyle{True}^\Imc := ~ & 
			\{ b \} \cup \{ v_k \mid 1 \leq k \leq n, \nu(v_k) = 1 \} 
			\\
			\cstyle{val}^\Imc := ~ & 
			\{ (v_k, \istyle{a}) \mid 1 \leq k \leq n, \nu(v_k) = 0 \}  \cup
			\{ (v_k, b) \mid 1 \leq k \leq n, \nu(v_k) = 1 \} 
		\end{align*}
		It is immediate that $\istyle{a} \notin \cstyle{True}^\Imc$, and it now suffices to verify that $\omega(\Imc) = 1$.
		All assertions from $\abox_\phi$ are satisfied.
		Regarding axioms from $\tbox$, notice that $\istyle{a}$ violates the concept inclusion $\cstyle{Bool} \sqsubseteq \cstyle{True}$, which has cost $1$.
		The other axioms of $\tbox$ are all satisfied; in particular, since $\nu$ falsifies $\phi$, it falsifies all its clauses, and thus none of the $c_i$'s needs to satisfy the concept $\cstyle{True}$.\smallskip
		
		
		\noindent$(\Rightarrow)$.
		Assume $\phi$ is a tautology. We need to prove that every interpretation with cost at most $1$ satisfies $q$.
		Consider an interpretation $\Imc$ with cost at most $1$.
		Note that $\istyle{a} \in \cstyle{Bool}^\Imc$ as assertion $\cstyle{Bool}(\istyle{a})$ has infinite cost.
		If $\istyle{a} \in \cstyle{True}^\Imc$, then we are done.
		Otherwise, the cost of $\Imc$ is $1$ due to the violation of the concept inclusion $\cstyle{Bool} \sqsubseteq \cstyle{True}$ at domain element $\istyle{a}$.
		In particular, all other axioms of $\tbox$ are satisfied by $\Imc$ and we also have $\cstyle{Bool}^\Imc \setminus \{ \istyle{a}\} \subseteq \cstyle{True}^\Imc$.
		From $\Imc$, we define a valuation of variables $v_1, \dots, v_n$: : for each $1 \leq k \leq n$, we set $\nu(v_k) := 0$ if $(v_k, \istyle{a}) \in \rstyle{val}^\Imc$ (in which case $v_k \in (\exists \rstyle{val}.\cstyle{False})^\Imc$), and $\nu(v_k) = 1$ otherwise (in which case $v_k \in (\exists \rstyle{val}.\cstyle{True})^\Imc$).
		By assumption, $\phi$ is a tautology, thus this valuation makes $\phi$ true.
		Therefore, one of its clauses $c_i$ is satisfied by $\nu$, which guarantees $c_i \in \cstyle{True}^\Imc$ due to one of the eight long axioms (depending on the shape of $c_i$).
		Axiom $\exists \rstyle{clause}.\cstyle{True} \sqsubseteq \cstyle{True}$ being satisfied then entails $\istyle{a} \in \cstyle{True}^\Imc$, a contradiction.
	\end{proof}

	\subsection{Lower Bounds in the $\dllite$ Family}
	\label{subsection:lower-dllite}
	\begin{toappendix}
		\subsection*{Proofs for Section~\ref{subsection:lower-dllite} (Lower Bounds in the $\dllite$ Family)}
	\end{toappendix}
	
	We now turn to the $\dllite$ family, which inherits the upper bounds from Theorems~\ref{theorem:upper-bound-alchio-cq-varying} and \ref{theorem:upper-bound-alchio-cq-optimal-cost}.
	We begin with a proof that, when $k$ is allowed to vary, all considered reasoning tasks are $\NP$-hard (resp.\ $\coNP$-hard) already for $\dllitecore$ WKBs.

	\begin{theoremrep}
		\label{theorem:lower-bound-dllitecore-bcs-varying-k}
		BCS, IQA$_p$ and CQA$_p$ for $\dllitecore$ are $\NP$-hard.
		IQA$_c$ and CQA$_c$ for $\dllitecore$ WKBs are $\coNP$-hard.
	\end{theoremrep}
	
	Note that, by virtue of Lemmas~\ref{lemma:bcs-to-possible-iq} and \ref{lemma:possible-iq-to-certain-iq}, it suffices to prove that BCS for $\dllitecore$ is $\NP$-hard.
	
	\begin{proofsketch}
		We reduce from \tcol, that is deciding whether a given graph $\Gmc = (\Vmc, \Emc)$ is 3-colourable.
		All axioms of $\tbox$ are given infinite weight by $\omega$ and are as follows:
		\[
		\hspace*{-1.2ex}\begin{array}{r@{~}c@{~~}l}
			\exists s_i \sqcap \exists t_j & \sqsubseteq \bot & \text{for } s, t \in \{ \rstyle{r}, \rstyle{g}, \rstyle{b} \}, s \neq t, \text{ and } i, j \in \{ 1, 2 \} 
			\\[3pt]
			\exists {s}_1^- \sqcap \exists {s}_2^- & \sqsubseteq  \bot & \text{for } s \in \{ \rstyle{r}, \rstyle{g}, \rstyle{b} \} \\
		\end{array}	
		\]
		We choose an orientation $\Emc'$ of $\Emc$: for each $\{ u, v \} \in \Emc$, we add either $(u, v)$ or $(v, u)$ in $\Emc'$.
		For each $e = (u, v) \in \Emc'$ and each $s \in \{ \rstyle{r}, \rstyle{g}, \rstyle{b} \}$, we add $s_1(u, e)$ and $s_2(v, e)$ in the ABox $\abox_\Gmc$.
		Assertions in $\abox_\Gmc$ are given weight $1$ by $\omega$.
		It can
		be verified that $\Gmc \in \tcol$ iff
		$(\tbox, \abox_\Gmc)_\omega$ is $4\sizeof{\Emc}$-satisfiable.
	\end{proofsketch}
	
	\begin{proof}
		We reduce from \tcol, which is the problem of deciding whether a given graph $\Gmc = (\Vmc, \Emc)$ is 3-colourable.
		All axioms of $\tbox$ are given infinite weight by $\omega$ and are as follows:
		\[
		\begin{array}{r@{~}c@{~~~}l}
			\exists s_i \sqcap \exists t_j & \sqsubseteq \bot & \text{for } s, t \in \{ \rstyle{r}, \rstyle{g}, \rstyle{b} \}, s \neq t, \text{ and } i, j \in \{ 1, 2 \} 
			\\[3pt]
			\exists {s}_1^- \sqcap \exists {s}_2^- & \sqsubseteq  \bot & \text{for } s \in \{ \rstyle{r}, \rstyle{g}, \rstyle{b} \} \\
		\end{array}	
		\]
		We choose an orientation $\Emc'$ of $\Emc$, \emph{i.e.}\ for each $\{ u, v \} \in \Emc$, we add either $(u, v)$ or $(v, u)$ to $\Emc'$.
		For each $e = (u, v)$ in $\Emc'$ and each role name $s \in \{ \rstyle{r}, \rstyle{g}, \rstyle{b} \}$, we add in the ABox $\abox_\Gmc$ the assertions $s_1(u, e)$ and $s_2(v, e)$.
		The assertions in $\abox_\Gmc$ are all given weight $1$ by $\omega$.\medskip
		
		\noindent\textbf{Claim:} $\Gmc \in \tcol$ iff $(\tbox, \abox_\Gmc)_\omega$ is $4\sizeof{\Emc}$-satisfiable.\smallskip
		
		Notice that it is clear that every interpretation has cost at least $4 \sizeof{\Emc}$ as, for each edge $e \in \Emc'$, removing at least $4$ of the $6$ role assertions involving $e$ is necessary to avoid the infinite cost of the TBox axioms.\smallskip
		
		\noindent$(\Rightarrow)$.
		Assume $\Gmc$ is 3-colourable, that is, there exists a 3-colouring $\sigma : \Vmc \rightarrow \{ \rstyle{r}, \rstyle{g}, \rstyle{b} \}$.
		Consider the interpretation $\Imc$ that interprets, for $s \in \{ \rstyle{r}, \rstyle{g}, \rstyle{b} \}$ and $i \in \{ 1, 2 \}$, the role name $s_i$ as:
		\begin{align*}
			s_1^\Imc := ~ & \{ (u, e) \mid e = (u, v) \in \Emc' \text{ and } \sigma(u) = s \}
			\\
			s_2^\Imc := ~ & \{ (v, e) \mid e = (u, v) \in \Emc' \text{ and } \sigma(v) = s \} 
		\end{align*}
		Note that the cost of ABox violations of $\Imc$ is exactly $4 \sizeof{\Emc}$.
		Furthermore, $\sigma$ being a 3-colouring guarantees that $\Imc$ does not violate any TBox axiom.
		Therefore, $(\tbox, \abox_\Gmc)_\omega$ is $4\sizeof{\Emc}$-satisfiable.\smallskip
		
		\noindent$(\Leftarrow)$.
		Assume $(\tbox, \abox_\Gmc)_\omega$ is $4\sizeof{\Emc}$-satisfiable, that is, there exists an interpretation $\Imc$ with cost $\leq 4 \sizeof{\Emc}$, thus exactly $4 \sizeof{\Emc}$.
		We construct the mapping $\sigma : \Vmc \rightarrow \{ \rstyle{r}, \rstyle{g}, \rstyle{b} \}$ by setting $\sigma(v) = s$ iff $v \in (\exists s_1)^\Imc \cup (\exists s_2)^\Imc$.
		Notice that $\sigma$ is a well-defined function as $\Imc$ has cost exactly $4 \sizeof{\Emc}$.
		It is then readily verified that $\sigma$ is a 3-colouring of $\Gmc$, hence $\Gmc \in \tcol$.
	\end{proof}
	
	For optimal cost semantics, we establish a matching $\deltaptwo$ lower bound. 
	The proof adapts an existing construction from \cite[Proposition~6.2.4]{DBLP:phd/hal/Bourgaux16} that 
	establishes $\deltaptwo$-hardness of query entailment for DL-Lite KBs under preferred repair semantics,
	by reduction from deciding if a given variable is true in the lexicographically maximum truth assignment satisfying a given satisfiable CNF. 
	We point out that, unlike the other lower bounds listed in Table \ref{table:results}, this result crucially relies upon a binary encoding of weights, 
	intuitively because exponentially large weights are needed to perform lexicographic comparison of satisfying valuations. 
	
	\begin{theoremrep}
		\label{theorem:lower-bound-dllitecore-optcost}
		IQA$_p^{opt}$, IQA$_c^{opt}$, CQA$_p^{opt}$, and CQA$_c^{opt}$ for $\dllitecore$ are $\deltaptwo$-hard.
	\end{theoremrep}	
	
	\begin{proof}
		The proof is by reduction from the following $\deltaptwo$-hard problem \cite{Krentel88}: given a satisfiable propositional 3-CNF formula $\varphi= c_1 \wedge \ldots \wedge c_m$ over variables $x_1,\dots, x_n$ and given  $k\in\{1,\dots, n\}$, decide whether the lexicographically maximum truth assignment $\nu_{max}$ satisfying $\varphi$ with respect to the ordering $(x_1,\dots, x_n)$ is such that $\nu_{max}(x_k)=1$. The reduction closely follows a similar reduction 
		for querying DL-Lite KBs under preferred repair semantics \cite[Proposition~6.2.4]{DBLP:phd/hal/Bourgaux16}. 
		
		Given such a formula $\varphi$, we define a $\dllitecore$ WKB $\WKB$  as follows:
		\begin{align*}
			\A=&\{\rstyle{p}_\ell(c_j,x_i) \mid x_i \text{ is the $\ell$th literal of } c_j \}\cup\\
			&\{\rstyle{n}_\ell(c_j,x_i)  \mid \neg x_i \text{ is the $\ell$th literal of } c_j \}\cup\\
			& \{\rstyle{T}(x_i)\mid 1\leq i\leq n\}\\ 
			\T=&\{\exists \rstyle{n}_\ell^- \sqsubseteq \neg \rstyle{T} \mid 1 \leq \ell \leq 3\} \cup\\
			&\{ \exists \rstyle{p}_\ell \sqsubseteq \neg \exists \rstyle{n}_{\ell'}, \exists \rstyle{p}_\ell^- \sqsubseteq \neg \exists \rstyle{n}_{\ell'}^- \mid 1\leq \ell, \ell' \leq 3 \} \cup \\
			&\{ \exists \rstyle{p}_\ell \sqsubseteq \neg \exists \rstyle{p}_{\ell'}, \exists \rstyle{n}_\ell \sqsubseteq \neg \exists \rstyle{n}_{\ell'} \mid 1\leq \ell \neq \ell' \leq 3 \} 
		\end{align*}
		We set $\omega(\tau)=\infty$ for every $\tau\in\T$. The weights of the ABox assertions will be defined through the following prioritization $\mathcal{P}$
		which partitions $\Amc$ as follows: 
		\begin{itemize}
			\item $L_1=\Amc \setminus\{\rstyle{T}(x_i)\mid 1\leq i\leq n\} $
			\item  $L_p= \{\rstyle{T}(x_{p-1})\}$ for $1 < p \leq n+1$
		\end{itemize} 
		following the method given in Lemma~6.2.5 in \cite{DBLP:phd/hal/Bourgaux16}, which 
		in our setting yields:  $\omega(\alpha)=u^{n+1-p}$ for every $\alpha\in L_p$, where $u=3m+1$. 
		
		Let $\I$ be an interpretation with optimal cost. As the TBox assertions all have infinite weight, 
		we know that $\I \models \T$. Furthermore, 
		it follows from Proposition~3 of \citex{BBJKR24} 
		that the ABox $\Amc_\I = \{ \alpha \in \Amc \mid \Imc \models \alpha\}$ 
		corresponding to the assertions satisfied in $\I$
		is an $\leq_w$-repair, i.e.\ a subset of the ABox consistent with $\Tmc$ and maximal among $\Tmc$-consistent ABox subsets for the preorder defined by $ \A_1 \leq_\omega \A_2$ if $\sum_{\alpha \in \A_1}\omega_\alpha \leq \sum_{\alpha \in \A_2}\omega_\alpha$. 
		Moreover, as we have defined the weight of assertions following \cite[Lemma~6.2.5]{DBLP:phd/hal/Bourgaux16}
		we know that $\Amc_\I$ is a $\leq_P$-repair of the KB $\langle \Tmc, \Amc \rangle$ w.r.t.\ the prioritization $\mathcal{P}$. 
		Formally, this means that there does not exist another subset $\Amc' \subseteq \Amc$ and priority level $1 \leq h < n+3$
		such that (i) $\Amc'$ is  $\Tmc$-consistent, (ii) $|\Amc_\I \cap L_h| < |\Amc' \cap L_h| $, and (iii)
		$|\Amc_\I \cap L_h |=  |\Amc' \cap L_h |$ for all $1 \leq g <h$.
		
		It has been proven in \cite[Proposition~6.2.4]{DBLP:phd/hal/Bourgaux16}
		that for every $\leq_P$-repair $\Amc^* \subseteq \Amc$  
		it is the case that $\rstyle{T}(x_i) \in \Amc^*$ iff the lexicographically maximum truth assignment 
		$\nu_{max}$ satisfying $\varphi$ fulfils $\nu_{max}(x_i)=1$. 
		In particular, this is the case for any ABox $\Amc_\I$
		induced by an optimal-cost interpretation $\I$. 
		It follows that for every optimal-cost interpretation $\I$,  
		$\I\models \rstyle{T}(x_k)$ iff $\nu_{max}(x_k)=1$. 
		Hence $\wkb \sat{c}{opt} \rstyle{T}(x_k)$ iff $\wkb \sat{p}{opt} \rstyle{T}(x_k)$ iff $\nu_{max}(x_k)=1$. \qedhere

	\end{proof}
	
	We now move to the case in which $k$ is fixed.
	Notice indeed that the lower bound from Theorem~\ref{theorem:lower-bound-dllitecore-bcs-varying-k} strongly relies on a varying $k$.
	For the certain semantics, we show that 
	$\coNP$-hardness holds
	even if $k$ is fixed to $1$, if we consider CQs. 
	The proof is strongly inspired by the one of Theorem~\ref{theorem:lower-bound-el-certain-iq-fixed-cost-1}, especially with how to simulate the truth value assignment.
	The main difference is that we use acyclic CQs to circumvent the lack of nested concepts in $\dllitecore$.
	
	\begin{theoremrep}
		\label{theorem:lower-bound-dllitepos-certain-cq-fixed-cost-1}
		For every $k \geq 1$, CQA$_c^k$ for $\dllitecore$ is $\coNP$-hard.
		This holds already for connected acyclic BCQs and without concept disjointness axioms.
	\end{theoremrep}

	\begin{proof}
		We reduce from deciding whether a 3-DNF formula $\phi := \bigvee_{i = 1}^\ell \bigwedge_{j = 1}^3 l_{i, j}$, where each $l_{i, j}$ is a variable from $v_1, \dots, v_n$ or its negation, is a tautology.
		We set $p_{i, j} := 0$ and $u_{i, j} := v_k$ if $l_{i, j} = \lnot v_k$.
		Similarly, we set $p_{i, j} := 1$ and $u_{i, j} := v_k$ if $l_{i, j} = v_k$. 
		The budget $k$ is fixed to $1$.
		All axioms and assertions have infinite weight except for the following, having weight $1$:
		\[
		\cstyle{Bool} \sqsubseteq \cstyle{True}.
		\]
		The other axioms are:
		\begin{align*}
			\cstyle{Var} & \sqsubseteq \exists \rstyle{val}
			&
			\exists \rstyle{val}^- & \sqsubseteq \cstyle{Bool}
		\end{align*}
		The assertions are $\cstyle{False}(\istyle{a})$, $\cstyle{True}(\istyle{b})$, $\cstyle{Bool}(\istyle{a})$, $\cstyle{Bool}(\istyle{b})$, $\rstyle{val}(\istyle{a}, \istyle{a})$, $\rstyle{val}(\istyle{a}, \istyle{b})$, $\rstyle{val}(\istyle{d}, \istyle{a})$  and:
		\begin{align*}
			\cstyle{Var}(v_k) & \textrm{ for } 1 \leq k \leq n
			\\
			\rstyle{clause}_{s}(\istyle{a}_{s}, c_i) & \textrm{ for } 1 \leq i \leq \ell \text{ where } s := (p_{i, 1}, p_{i, 2}, p_{i, 3})
			\\
			\rstyle{clause}_{s}(\istyle{a}_{s'}, \istyle{a}) & \textrm{ for } s, s' \in \{ 0, 1 \}^3 \text{ with } s \neq s'
			\\
			\rstyle{clause}_{s}(\istyle{d}, \istyle{d}) & \textrm{ for } s \in \{ 0, 1 \}^3 \\
			\rstyle{lit}_{j, p_{i, j}}(c_i, u_{i, j}) & \textrm{ for } 1 \leq i \leq \ell, 1 \leq j \leq 3
			\\
			\rstyle{lit}_{j, s}(\istyle{a}, \istyle{a}) & \textrm{ for } 1 \leq j \leq 3, s \in \{ 0, 1 \}\\
			\rstyle{lit}_{j, s}(\istyle{d}, \istyle{d}) & \textrm{ for } 1 \leq j \leq 3, s \in \{ 0, 1 \}
		\end{align*}
		The connected and acyclic BCQ $q$ is the conjunction of $8$ subqueries $q_{s}(y)$ for $s \in \{ 0, 1 \}^3$ that all share the variable $y$:
		\[
		q := \exists y \bigwedge_{s \in \{ 0, 1\}^3} q_{s}(y).
		\]
		In turn, for each $s := (s_1, s_2, s_3) \in \{ 0, 1 \}^3$, the query $q_{s}(y)$ is defined as:
		\[\begin{array}{r}
			q_s(y) = ~ \displaystyle{\exists y_{s, 0}\ \exists y_{s, 1}\ \exists y_{s, 2}\ \exists y_{s, 3}\ \exists y_{s, 1}'\ \exists y_{s, 2}'\ \exists y_{s, 3}'\ 
				~ \rstyle{clause}_s(y, y_{s, 0}) 
				\land \bigwedge_{j = 1}^3 \big( \rstyle{lit}_{j, s_j}(y_{s, 0}, y_{s, j}) \land \rstyle{val}(y_{s, j}, y_{s, j}') \big)}
			\\[3pt]
			\displaystyle{\land \bigwedge_{\substack{ j \in \{ 1, 2, 3\} \\ s_j = 1}} \cstyle{True}(y_{s, j}') 
				\land \bigwedge_{\substack{ j \in \{ 1, 2, 3\} \\ s_j = 0}} \cstyle{False}(y_{s, j}').}
		\end{array}\]
		
		\noindent\textbf{Claim:}
		$\phi$ is a tautology iff $(\tbox, \abox_\phi)_\omega \models^1_c q$.\smallskip
		
		\noindent$(\Rightarrow)$.
		Assume $\phi$ is a tautology.
		Consider an interpretation $\Imc$ with cost at most $1$. First consider the case where $\istyle{a} \in \cstyle{True}^\Imc$. In this case, we can trivially satisfy $q$ by mapping all variables of the forms $y,y_{s, 0}, y_{s, j}$ to $\istyle{d}$ and all variables of the form $y_{s, j}'$ to~$\istyle{a}$. Indeed, the ABox (which must be satisfied since the cost is at most $1$ and ABox assertions have infinite weight) ensures that the role atoms are satisfied and that $\istyle{a}$ belongs to $\cstyle{False}$, and hence that $\istyle{a}$ can be used to satisfy both $\cstyle{True}$ and $\cstyle{False}$ query atoms. Next consider the case where $\istyle{a} \not \in \cstyle{True}^\Imc$. In this case,
the cost must be exactly $1$, via $\istyle{a}$ not being an instance of $\cstyle{True}$ despite $\istyle{a} \in \cstyle{Bool}^\Imc$ (as the corresponding assertion has infinite weight) and $\cstyle{Bool} \sqsubseteq \cstyle{True}$ (the violation of which costs $1$).
		Therefore, $(\exists \rstyle{val}^-)^\Imc \setminus \{ \istyle{a}\} \subseteq \cstyle{True}^\Imc$ as otherwise we would violate either $\exists \rstyle{val}^- \sqsubseteq \cstyle{Bool}$ or $\cstyle{Bool} \sqsubseteq \cstyle{True}$.
		We define the underlying valuation $\sigma : \{ v_1, \dots, v_n \} \rightarrow \{ 0, 1 \}$ by setting $\sigma(v_k) = 0$ if $(v_k, \istyle{a}) \in \rstyle{val}^\Imc$ and $\sigma(v_k) = 1$ otherwise.
		Notice that, in the second case, that is $(v_k, \istyle{a}) \notin \rstyle{val}^\Imc$, there exists at least an element $e \in \cstyle{True}^\Imc$ such that $(v_k, e) \in \rstyle{val}^\Imc$.
		For each such $v_k$, we choose one such $e$ and denote it $f(v_k)$.
		Since $\phi$ is a tautology we have $\sigma(\phi) = 1$, thus there exists a $i \in \{ 1, \dots, \ell \}$ such that $\sigma(c_i) = 1$.
		Let $s := (p_{i, 1}, p_{i, 2}, p_{i, 3})$.
		We define a mapping $h$ from variables of $q$ to $\Delta^\Imc$ as follows:
		\begin{align*}
			y \mapsto ~ & a_s
			\\
			y_{s, 0} \mapsto ~ & c_i
			\\
			y_{s, j} \mapsto ~ & u_{i, j}
			\\
			y_{s, j}' \mapsto ~ & \istyle{a} & & \text{ if } \sigma(u_{i, j}) = 0
			\\
			y_{s, j}' \mapsto ~ & f(u_{i, j}) & & \text{ if } \sigma(u_{i, j}) = 1
			\\
			y_{s', j} \mapsto ~ & \istyle{a} & & \text{ for } s' \neq s, j \in \{ 0, 1, 2, 3\}
			\\
			y_{s', j}' \mapsto ~ & \istyle{b} & & \text{ for } s' \neq s, j \in \{ 1, 2, 3\} \text{ s.t. } s'_j = 1
			\\
			y_{s', j}' \mapsto ~ & \istyle{a} & & \text{ for } s' \neq s, j \in \{ 1, 2, 3\} \text{ s.t. } s'_j = 0
		\end{align*}
		It is readily verified that $h$ is a homomorphism of $q$, thus $\Imc \models q$ as desired.\smallskip
		
		\noindent$(\Leftarrow)$.
		Assume $\phi$ is not a tautology.
		There exists a valuation $\nu$ of variables $v_1, \dots, v_n$ that makes all clauses false.
		Let $\nu'$ be the ``valuation'' of variables $v_1, \dots, v_n$ obtained by setting $\nu'(v) = \istyle{a}$ if $\nu(v) = 0$ and $\nu'(v) = \istyle{b}$ otherwise (in other words, just replace $0$ by $\istyle{a}$ and $1$ by $\istyle{b}$ in $\nu$).
		We define an interpretation $\Imc$ corresponding to that valuation $\nu$ by interpreting all predicates as in the ABox, except for the role name $\rstyle{val}$, that is interpreted as follows:
		\[
		\rstyle{val}^\Imc :=  \{ (\istyle{a}, \istyle{a}), (\istyle{a}, \istyle{b} ) \} \cup \{ (v_k, \nu'(v_k)) \mid 1 \leq k \leq n \}.
		\]
		It is immediate that the cost of $\Imc$ is $1$, due to $\istyle{a} \in \cstyle{Bool}^\Imc \setminus \cstyle{True}^\Imc$, violating the TBox axiom $\cstyle{Bool} \sqsubseteq \cstyle{True}$.
		It remains to verify that $\Imc \not\models q$.
		Assume by contradiction that there exists a homomorphism $h : q \rightarrow \Imc$.
		Due to the interpretation of the roles of the form $\rstyle{clause}_s$, the variable $y$ can only be mapped by $h$ on one of the individuals ${a}_s$ or on $\istyle{d}$.
		Note that the latter is excluded as $h(y) = \istyle{d}$ would require $h(y'_{(1, 1, 1), 1}) = \istyle{a}$, to respect the interpretation of $\rstyle{val}$, but $\istyle{a} \notin \cstyle{True}^\Imc$.
		Therefore, $h(y) = a_s$ for some $s \in \{ 0, 1 \}^3$.
		Then the variable $y_{s, 0}$ can only be mapped on some individual $c_i$ for some $1 \leq i \leq \ell$ and, in turn, variables $y_{s, j}$ are mapped to individuals $u_{i, j}$.
		However, as $\nu$ makes all clauses false, and in particular $c_i$, it is readily verified that the images for variables $y_{s, j}'$ cannot adhere to the interpretation $\rstyle{val}^\Imc$ defined above.
		Therefore, $\Imc \not\models q$, thus $(\tbox, \abox_\phi)_\omega \not\models^1_c q$.
	\end{proof}

	\section{Positive Results in the DL-Lite Family}
	\label{section:dl-lite}
	
	In light of the lower bounds established in the previous section, we can only hope to achieve tractability for the $\dllite$ family under a fixed cost (see Theorem~\ref{theorem:lower-bound-dllitecore-bcs-varying-k}).
	Furthermore, under the certain semantics, we established that CQA$_c^k$ answering is $\coNP$-hard (Theorem~\ref{theorem:lower-bound-dllitepos-certain-cq-fixed-cost-1}) already for $\dllitecore$ and $k=1$.
	This leaves us with two promising settings to explore: CQA$_p^k$ and IQA$^k_c$ for $\dllitecore$.
	This section establishes that both 
	reasoning tasks enjoy the lowest possible complexity, that is, $\tczero$. 
	Furthermore, we can even push this positive result 
	to one of the most expressive logics of the $\dllite$ family, namely $\dlliteboolh$.
	\begin{theorem}
		\label{theorem:upper-bound-dlliteboolh-cq-fixed}
		For every integer $k \geq 1$, CQA$_p^k$ and IQA$_c^k$ for $\dlliteboolh$ are in $\tczero$.
	\end{theorem}
	
	Our approach is based on first-order ($\fo$) rewriting: given the weighted TBox $\tbox_{\omega_\tbox}$, query $q$ and cost $k$, we construct an $\fo$-query $q'$ such that for every weighted ABox $\abox_{\omega_{\!\abox}}$, the following, here stated for CQA$_p^k$, holds:
	$$
	(\tbox, \abox)_{\omega_\tbox \cup \omega_{\!\abox}} \models^k_p q \quad \text{ iff } \quad \Imc_{\abox_{\omega_{\!\abox}}} \models q'. 
	$$
	To make this formulation fully precise, we need to define the FO-interpretation  $\Imc_{\abox_{\omega_{\!\abox}}}$
	associated with a weighted ABox $\abox_{\omega_{\!\abox}}$, over which the rewritten query $q'$ is evaluated. 	
	We argue that 
	$\abox_{\omega_{\!\abox}}$ can be seen as a usual ABox augmented with extra assertions about the weights. 
	We denote by $\abox^{\omega_{\!\abox}}_k$ 
	the extension of $\abox$ 
	with special concept and role assertions that encapsulate the relevant information about weights w.r.t.\ the fixed cost bound~$k$: if a concept assertion $\omega_{\!\abox}(\cstyle{A}(\istyle{a})) = n$, then we add the assertion $\cstyle{W}_\cstyle{A}^n(\istyle{a})$ if $n \leq k$, or assertion $\cstyle{W}_\cstyle{A}^\infty(\istyle{a})$ if $n > k$.
	We proceed similarly for each role assertion $\rstyle{r}(\istyle{a}, \istyle{b})$, adding respectively $\rstyle{w}_\rstyle{r}^n(\istyle{a}, \istyle{b})$ or $\rstyle{w}_\rstyle{r}^\infty(\istyle{a}, \istyle{b})$.
	Notice that we only need to introduce $(k+1)(\NC(\abox) + \NR(\abox))$ fresh predicates and that computing $\abox^{\omega_{\!\abox}}_k$ from any reasonable representation of $\abox_{\omega_{\!\abox}}$ can be seen as a pre-processing step achieved by an $\tczero$ transducer. 
	As $\abox^{\omega_{\!\abox}}_k$ is a usual ABox, its corresponding interpretation $\Imc_{\abox^{\omega_{\!\abox}}_k}$ is well defined and
	can serve as the desired interpretation 
	$\Imc_{\abox_{\omega_{\!\abox}}}$. 
	
	Before explaining how to construct $q'$, 
	we sketch the main argument that allows for such a rewriting to exist and that ensures completeness of the claim (\emph{i.e.}\ the $\Rightarrow$ direction in the formulation above).
	It relies on a (very!) small interpretation property: if an interpretation witnesses the desired behaviour w.r.t.\ the query and within the fixed cost, then there exists one that is completely trivial except on a small domain whose size is bounded by a \emph{constant} w.r.t.\ the input weighted ABox.
	The different possibilities to interpret such a constant-size domain can thus all be encapsulated in the rewritten query $q'$.
	To facilitate the understanding of the rewriting, we first present this small interpretation property.
	
	\subsection{A (Very) Small Interpretation Property}
	\label{subsection:small}
	\begin{toappendix}
		\subsection*{Proofs for Section~\ref{subsection:small} (A (Very) Small Interpretation Property)}
	\end{toappendix}

	\newcommand{\criticals}{\ensuremath{\mathsf{crit}}}
	
	Consider a WKB $\kb = (\tbox, \abox)_\omega$, a fixed cost $k$ and a BCQ~$q$.
	We introduce two distinct notions of types.
	The first is the usual one in DLs: a \emph{$1$-type} $t$ is a subset of $\NC(\tbox) \cup \{ \exists \rstyle{r} \mid \rstyle{r} \in \NRpm(\tbox) \}$.
	The $1$-type of an element $e \in \Delta^\Imc$ is $\type_\Imc(e) := \{ \cstyle{A} \mid e \in \cstyle{A}^\Imc \} \cup \{ \exists \rstyle{r} \mid e \in (\exists \rstyle{r})^\Imc, \rstyle{r} \in \NRpm \}$.
	This notion of $1$-type captures the basic $\dllite$ concepts, and thus, if two elements $d$ and $e$ have the same $1$-type, then they violate exactly the same $\dllitebool$ CIs.
	
	The second notion of type is intended to capture the concepts that hold due to the ABox assertions. 
	We define the \emph{ABox type} $\type_\abox(\istyle{a})$ of 
	$\istyle{a} \in \mathsf{Ind}(\abox)$ as:
	\[
		\type_\abox(\istyle{a}) := \type_{\Imc_\abox}(\istyle{a}) \cup \{ \existsmany \rstyle{r} \mid \abox \models \existsmany \rstyle{r}, \rstyle{r} \in \NRpm(\tbox) \},
	\]
	where $\abox \models \existsmany \rstyle{r}$ means that there exists (at least) $k+1$ distinct individuals $\istyle{b}_1, \dots, \istyle{b}_{k+1} \in \NI$ such that $\rstyle{r}(a, \istyle{b}_1), \dots, \rstyle{r}(a, \istyle{b}_{k+1}) \in \abox$ (or $\rstyle{r}(\istyle{b}_1, a), \dots, \rstyle{r}(\istyle{b}_{k+1}, a) \in \abox$ if $\rstyle{r}$ is an inverse role).
	Notice that there are at most $2^{\sizeof{\NC(\tbox)}+2\sizeof{\NR(\tbox)}}$ 
	possible $1$-types, and at most $2^{\sizeof{\NC(\tbox)}+4\sizeof{\NR(\tbox)}}$ possible ABox types. 
	
	We write $\rstyle{r} \sqsubseteq_\tbox \rstyle{s}$ if $(\rstyle{r}, \rstyle{s})$ is in the transitive closure of $\{ (\rstyle{p}, \rstyle{p}) \mid \rstyle{p} \in \NRpm \} \cup \{ (\rstyle{p}, \rstyle{q}) \in \NRpm \times \NRpm \mid \rstyle{p} \sqsubseteq \rstyle{q} \in \tbox \}$.
	The following observation motivates the extension of $1$-types into ABox types when a fixed cost $k$ is considered:
	\begin{lemmarep}
		\label{lemma:many-succ-in-abox-mean-at-least-one-in-an-interpretation}
		Consider an interpretation $\Imc$ whose cost is $\leq k$.
		For every individual $\istyle{a}$ and role $\rstyle{r} \in \NRpm$, if $\existsmany \rstyle{r} \in \type_\abox(\istyle{a})$,
		then $\exists \rstyle{s} \in \type_\Imc(\istyle{a})$ for every role $\rstyle{s} \in \NRpm$ such that $\rstyle{r} \sqsubseteq_\tbox \rstyle{s}$.
	\end{lemmarep}
	\begin{proof}
		By contradiction, assume that there exists a role $\rstyle{s}\in \NRpm$ such that $\exists \rstyle{s} \notin \type_\Imc(\istyle{a})$.
		From $\existsmany \rstyle{r} \in \type_\abox(\istyle{a})$, we obtain $k+1$ distinct elements $\istyle{b}_1, \dots, \istyle{b}_{k+1} \in \NI$ such that $\rstyle{r}(a, \istyle{b}_1), \dots, \rstyle{r}(a, \istyle{b}_{k+1}) \in \abox$ (we here treat the case of $\rstyle{r}$ not being an inverse role, the other case is similar).
		Note that for each such pair, we have $(a, \istyle{b}_i) \notin \rstyle{s}^\Imc$ as $\exists \rstyle{s} \notin \type_\Imc(\istyle{a})$.
		We now prove that each pair $(a, \istyle{b}_i)$ is involved in at least one binary violation, which will guarantee $k+1$ distinct violations, thus an overall cost for $\Imc$ that is $\geq k + 1$, which is the desired contradiction.
		We distinguish based on $(a, \istyle{b}_i) \in \rstyle{r}^\Imc$.
		If $(a, \istyle{b}_i) \notin \rstyle{r}^\Imc$, then this is a clear violation of the role assertion $\rstyle{r}(a, \istyle{b}_i)$.
		Else $(a, \istyle{b}_i) \in \rstyle{r}^\Imc$, and from $(a, \istyle{b}_i) \notin \rstyle{s}^\Imc$ we have $\rstyle{r} \neq \rstyle{s}$, thus at least one role inclusion used to witness $\rstyle{r} \sqsubseteq_\tbox \rstyle{s}$ is violated on $(a, \istyle{b}_i)$.
	\end{proof}
	
	We now prove that, if there exists an interpretation $\Imc$ whose cost is $\leq k$, then we find an interpretation $\Jmc$ that behaves as $\Imc$ w.r.t.\ the query $q$, and whose cost is also $\leq k$, with all violations concentrated in a predictable portion of its domain.
	This portion of the domain of $\Jmc$ is of course small, since the maximum number of violations is $k$, thus involving at most $2k$ distinct elements.
	By `predictable', we mean that these violations take place among a small number of special individuals (constant number with respect to $k$) that can easily be identified, and on a small set of additional domain elements. 
	To identify these special individuals, we rely on the following intuition: if an ABox type $t$ is realized more than $2k$ times in $\abox$, then there is a way to complete $t$ without any `local' violation.
	Indeed, if it was impossible to do so, then $t$ being realized more than $2k$ times would always result in more than $k$ violations and 
	thus in a cost exceeding $k$, contradicting the very existence of $\Imc$.
	Therefore, only individuals with a rare ABox type may require a special treatment to keep the cost less than $k$.
	Formally, we say that an ABox type $t$ is \emph{rare} in $\abox$ if $\#\{ \istyle{a} \in \mathsf{Ind}(\abox) \mid \type_\abox(\istyle{a}) = t \} \leq 2k$,
	and we use $\raretypes(\abox)$ 
	for the set of rare ABox types in $\abox$.
	
	Now, when we start from an interpretation $\Imc$ with cost $\leq k$ and attempt to build $\Jmc$, we preserve the interpretation of concepts and roles from $\Imc$ on those special individuals that have a rare ABox type. 
	For $\Jmc$ to behave like $\Imc$ with respect to the query $q$, we also preserve the interpretation on individuals occurring in $q$.
	We define the pre-core $\mathsf{pc}(\abox)$ as the set of individuals from $\abox$ that have a rare ABox type, plus those query-related individuals, that is:
	\[
	\precore(\abox) := \{ \istyle{a} \mid \type_\abox(\istyle{a}) \in \raretypes(\abox) \} \cup \individuals(q).
	\]
	Unfortunately, the pre-core does not contain 
	all the individuals that may be forced to participate in violations.
	As the following example illustrates, elements `close' to the pre-core may also be forced to do so.

	\begin{example}
		Consider the fixed cost $k := 3$ and the ABox
				\(
					\abox := \{ \cstyle{A}(\istyle{a}_0) \} \cup \bigcup_{i = 0}^7 \{ \rstyle{r}(\istyle{a}_i, \istyle{b}_i), \rstyle{t}(\istyle{b}_i, \istyle{c}_i) \}.
				\)
		The ABox type of $\istyle{a}_0$ is rare, 
		others are not.
		Consider the TBox $\tbox$ with axioms: 
		\[
		\cstyle{A} \sqsubseteq \exists \rstyle{u} \sqcap \lnot \exists \rstyle{s} 
		\quad 
		\exists \rstyle{u}^- \sqsubseteq \exists \rstyle{r}^- \sqcap \lnot \exists \rstyle{s}^-
		\quad 
		\exists \rstyle{t} \sqsubseteq \exists \rstyle{s}^- 
		\quad 
		\rstyle{r} \sqsubseteq \rstyle{s}
		\]
		and assign cost $1$ to all $\rstyle{t}$ assertions, cost $2$ to axiom $\rstyle{r} \sqsubseteq \rstyle{s}$, and infinite cost to other assertions and axioms.
		Every interpretation with cost $\leq 3$ violates the assertion $\rstyle{t}(\istyle{b}_0, \istyle{c}_0)$, which is somewhat surprising as both involved individuals have ABox types that can otherwise be instantiated in a way that does not violate anything.
		However, $\istyle{b}_0$ and $\istyle{c}_0$ happen to be `close', \emph{i.e.}\ at distance less than $k = 3$, to $\istyle{a}_0$. 
		The rare type of $\istyle{a}_0$ can then impact $\istyle{b}_0$ and $\istyle{c}_0$ as seen above.
	\end{example}
	
	To capture those individuals that may be affected by elements from the pre-core, we essentially explore the neighbourhood of the latter.
	For $\istyle{a}, \istyle{b} \in \individuals(\abox)$, we write $\istyle{a} \leadsto_1 \istyle{b}$ if there exists a role $\rstyle{r} \in \NRpm$  and an assertion $\rstyle{r}(\istyle{a}, \istyle{b})$ in $\abox$ and $\existsmany \rstyle{r} \notin \type_\abox(\istyle{a})$.
	When exploring neighbours, the reason we exclude roles $\rstyle{r}$ such that $\existsmany \rstyle{r} \in \type_\abox(\istyle{a})$ comes from Lemma~\ref{lemma:many-succ-in-abox-mean-at-least-one-in-an-interpretation}: it guarantees that element $\istyle{a}$ satisfies $\exists \rstyle{r}$, so the potential violations on $\istyle{a}$ do not impact the $\rstyle{r}$-edges to $\istyle{b}$.
	We then denote $\istyle{a} \leadsto_{i+1} \istyle{c}$ if there exists $\istyle{b}$ such that $\istyle{a} \leadsto_{i} \istyle{b}$ and $\istyle{b} \leadsto_{1} \istyle{c}$.
	The core of $\abox$, denoted $\core(\abox)$, is now defined as:
	\[
	\core(\abox) := \precore(\abox) \cup \{ \istyle{b} \mid \istyle{a} \leadsto_i \istyle{b}, \istyle{a} \in \precore(\abox), i \leq k + 1 \}.
	\]
	Notice that we stop the exploration of the neighbourhood of $\precore(\abox)$ at depth $k+1$.
	This is simply because the special behaviour of a pre-core element can only ``cascade'' to neighbours by enforcing a violation at each layer; thus, impacted elements cannot be further than $(k+1)$-away.
	
	We can now state our key technical lemma.
	Note that Points~5$_p$ and 5$_c$ are used respectively to handle the possible and certain semantics.
	Recall that Theorem~\ref{theorem:lower-bound-dllitepos-certain-cq-fixed-cost-1} established $\coNP$-hardness for CQA$_c^k$, which is why Point~5$_c$ only concerns the case where $q$ is an IQ.
	
	\begin{lemmarep}
		\label{lemma:main-lemma-for-dllite}
		Let $\kb = (\tbox, \abox)_\omega$ be a WKB, $q$ a BCQ, and $k$ a fixed cost.
		If there exists an interpretation $\Imc$ with cost $\leq k$, then there exists an interpretation $\Jmc$ such that:
		\begin{enumerate}[left= 5pt] 
			\item $\Delta^\Jmc = \individuals(\abox) \cup W$ for some $W \subseteq \{ w_t \mid t \text{ is a $1$-type} \}$; 
			\item $\Jmc\vert_{\precore(\abox)} = \Imc\vert_{\precore(\abox)}$;
			\item $\omega(\Jmc) = \omega(\Jmc\vert_{\core(\abox) \cup W})$;
			\item $\omega(\Jmc) \leq k$;
			\item[5$_p$.] If $\Imc \models q$, then $\Jmc\vert_{\core(\abox) \cup W} \models q$ (and thus $\Jmc \models q$);
			\item[5$_c$.] If $q$ is an IQ and $\Imc \not\models q$, then $\Jmc \not\models q$.
		\end{enumerate}
	\end{lemmarep}
	
\begin{toappendix}
	Figure~\ref{figure:small-interpretation-property} illustrates the construction underlying Lemma~\ref{lemma:main-lemma-for-dllite} and can be useful to follow its proof.
\end{toappendix}

	\begin{proof}
		Let $\kb = (\tbox, \abox)_\omega$ be a WKB, $q$ a Boolean instance query, and $k$ a fixed cost.
		Let $\Imc$ be an interpretation with cost $\leq k$. 
		We define the set $\criticals(\Imc)$ of  \emph{critical elements} of $\Imc$ 
		as the smallest set of elements from $\Delta^\Imc$ such that:
		\begin{itemize}
			\item $\precore(\abox) \subseteq \criticals(\Imc)$; and
			\item if $d \in \criticals(\Imc)$, $d \leadsto_1 e$ and $(d, e)$ is involved in a binary violation (\emph{i.e.}\ either a role inclusion or a role assertion), then $e \in \criticals(\Imc)$.
		\end{itemize}
		Notice that, as the cost of $\Imc$ has cost $\leq k$ and that each inductive step in the second item above involves at least one violation, it is clear that $\criticals(\Imc) \subseteq \core(\abox)$.
		We then set $$W := \{ w_t \mid t \in \type_\Imc(\Delta^\Imc \setminus \criticals(\Imc) ) \}$$ and define a mapping
		$\rho : \Delta^\Imc \rightarrow ~ \criticals(\Imc) \cup W$ as follows: 
		\begin{align*}
			d \mapsto ~ & \left\lbrace \begin{array}{ll}
				d & \text{if } d \in \criticals(\Imc)
				\\
				w_{\type_\Imc(d)} & \text{otherwise.}
			\end{array} \right.
		\end{align*}
		
		For each individual $\istyle{a} \in \individuals(\abox) \setminus \criticals(\Imc)$, notice that $\istyle{a} \notin \precore(\abox)$, thus the ABox type $t := \type_\abox(\istyle{a})$ is not rare.
		As the cost of $\Imc$ is $\leq k$, there must exist some individual $\istyle{b}$ such that $\type_\abox(\istyle{b}) = t$ and $\istyle{b}$ is not involved in any violation, that is all concept assertions and inclusions are satisfied by $\istyle{b}$ and for all $d \in \Delta^\Imc$, all role assertions and inclusions are satisfied by $(\istyle{b}, d)$ and at $(d, \istyle{b})$.
		For each such type $t$, we select one such `perfect' individual $p_t$.
		
		We define $\Jmc$ as the interpretation with domain $\individuals(\abox) \cup W$ 
		that interprets concept and role names 
		as follows: 
		\begin{align*}
			\cstyle{A}^\Jmc := ~ &
			\rho(\cstyle{A}^\Imc) \cup \{ \istyle{a} \in \individuals(\abox) \setminus \criticals(\Imc) \mid p_{\type_\abox(\istyle{a})} \in \cstyle{A}^\Imc \}
			\\
			\rstyle{r}^\Jmc := ~ &
			\rho(\rstyle{r}^\Imc) 
			\\
			&
			\cup \left\lbrace (\istyle{a}, \istyle{b}) \in \individuals(\abox)^2 \setminus \criticals(\Imc)^2 \middle| 
			\begin{array}{l} 
				\rstyle{s}(\istyle{a}, \istyle{b}) \in \abox \text{ for}
				\\
				\text{some } \rstyle{s} \sqsubseteq_\tbox \rstyle{r}
			\end{array}
			\right\rbrace
			\\
			&
			\cup \left\lbrace (\istyle{b}, \istyle{a}) \in \individuals(\abox)^2 \setminus \criticals(\Imc)^2 \middle| 
			\begin{array}{l} 
				\rstyle{s}(\istyle{a}, \istyle{b}) \in \abox \text{ for}
				\\
				\text{some } \rstyle{s} \sqsubseteq_\tbox \rstyle{r}^-
			\end{array}
			\right\rbrace
			\\
			&
			\cup \{ (\istyle{a}, \rho(e)) \mid \istyle{a} \in \individuals(\abox) \setminus \criticals(\Imc), (p_{\type_\abox(\istyle{a})}, e) \in \rstyle{r}^\Imc \}
			\\
			&
			\cup \{ (\rho(e), \istyle{a}) \mid \istyle{a} \in \individuals(\abox) \setminus \criticals(\Imc), (e, p_{\type_\abox(\istyle{a})}) \in \rstyle{r}^\Imc \}
		\end{align*}
		
			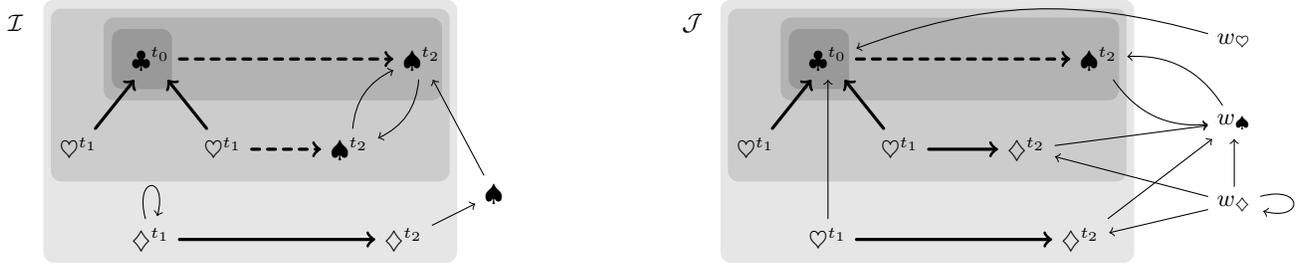
\begin{figure*}
			\centering
			\begin{tikzpicture}[every node/.append style={font=\small, scale=1}, xscale=1.2, yscale=1.2, line cap=round, line join=round]
				
				\node at (-1.5, .4) {$\Imc$};
				\node at (1.1, -.8) (rectind) [rectangle, fill=gray!20, minimum width=5.5cm, minimum height=3.5cm, rounded corners] {};
				\node at (1.1, -.4) (rectcore) [rectangle, fill=gray!40, minimum width=5.3cm, minimum height=2.3cm, rounded corners] {};
				\node at (1.35, 0) (rectcrit) [rectangle, fill=gray!60, minimum width=4.5cm, minimum height=1.1cm, rounded corners] {};
				\node at (-.1, 0) (rectpc) [rectangle, fill=gray!80, minimum width=.8cm, minimum height=.8cm, rounded corners] {};
				
				\node at ( 0, 0) (a) [] {$\clubsuit^{t_0}$};
				
				\node at (-.8, -1) (t11) [] {$\heartsuit^{t_1}$};
				\node at ( .8, -1) (t12) [] {$\heartsuit^{t_1}$}; 
				\node at ( 0, -2) (t13) [] {$\diamondsuit^{t_1}$};
				
				\node at (3, 0) (t21) [] {$\spadesuit^{t_2}$};
				\node at (2.2, -1) (t22) [] {$\spadesuit^{t_2}$};
				\node at (2.8, -2) (t23) [] {$\diamondsuit^{t_2}$}; 
				
				\node at (3.8, -1.5) (ano) [] {$\spadesuit$};

				\path[every edge/.append style={->}]
				(a) edge [dashed, very thick] (t21)
				(t12) edge [dashed, very thick] (t22)
				(t11) edge [very thick] (a)
				(t12) edge [very thick] (a)
				(t13) edge [very thick] (t23)
				(t13) edge [loop above] (t13)
				(t21) edge [bend left] (t22)
				(t22) edge [bend left] (t21)
				(t23) edge [] (ano)
				(ano) edge [] (t21)
				;

				\begin{scope}[shift={(7.5, 0)}]
					\node at (-1.5, .4) {$\Jmc$};
					\node at (1.1, -.8) (rectind) [rectangle, fill=gray!20, minimum width=5.5cm, minimum height=3.5cm, rounded corners] {};
					\node at (1.1, -.4) (rectcore) [rectangle, fill=gray!40, minimum width=5.3cm, minimum height=2.3cm, rounded corners] {};
					\node at (1.35, 0) (rectcrit) [rectangle, fill=gray!60, minimum width=4.5cm, minimum height=1.1cm, rounded corners] {};
					\node at (-.1, 0) (rectpc) [rectangle, fill=gray!80, minimum width=.8cm, minimum height=.8cm, rounded corners] {};
					
					\node at ( 0, 0) (a) [] {$\clubsuit^{t_0}$};
					
					\node at (-.8, -1) (t11) [] {$\heartsuit^{t_1}$};
					\node at ( .8, -1) (t12) [] {$\heartsuit^{t_1}$}; 
					\node at ( 0, -2) (t13) [] {$\heartsuit^{t_1}$};
					
					\node at (3, 0) (t21) [] {$\spadesuit^{t_2}$};
					\node at (2.2, -1) (t22) [] {$\diamondsuit^{t_2}$};
					\node at (2.8, -2) (t23) [] {$\diamondsuit^{t_2}$}; 
					
					\node at (4.5, .2) (wheart) [] {$w_\heartsuit$};
					\node at (4.5, -.7) (wspade) [] {$w_\spadesuit$};
					\node at (4.5, -1.6) (wdiamond) [] {$w_\diamondsuit$};

					\path[every edge/.append style={->}]
					(a) edge [dashed, very thick] (t21)
					(t12) edge [very thick] (t22)
					(t11) edge [very thick] (a)
					(t12) edge [very thick] (a)
					(t13) edge [very thick] (t23)
					(t13) edge [] (a)
					(wdiamond) edge [] (t22)
					(wdiamond) edge [] (t23)
					(wdiamond) edge [loop right] (wdiamond)
					(wdiamond) edge  (wspade)
					(wspade) edge [bend right] (t21)
					(t21) edge [bend right] (wspade)
					(t23) edge [] (wspade)
					(t22) edge [] (wspade)
					(wheart) edge [bend right=18] (a)
					;
				\end{scope}

			\end{tikzpicture}
			\caption{
				Two interpretations $\Imc$ and $\Jmc$ of the same WKB.
				Symbols $\heartsuit, \diamondsuit$ indicate $1$-types without unary violations; $\clubsuit, \spadesuit$ indicate $1$-types with unary violations.
				An arrow indicates the role $\rstyle{r}$; a bold arrow indicates a satisfied $\rstyle{r}$-role assertion; a dashed bold arrow indicates the violation of an $\rstyle{r}$-role assertion.
				Superscripts indicate the ABox type of an element; note that $\Imc$ features one anonymous element.
				We omit $k$ and pretend $t_0$ is rare, while $t_1$ and $t_2$ are not (alternatively, one could set $k \geq 6$ and consider sufficiently many other individuals with ABox types $t_1$ and $t_2$).
				Shades of gray, from darker to lighter, highlight the pre-core, the critical elements, the core, and the individuals.
				$\Jmc$ is obtained by applying Lemma~\ref{lemma:main-lemma-for-dllite} on $\Imc$, using the `perfect' individual with $1$-type $\heartsuit$ (resp.\ $\diamondsuit$) for ABox type $t_1$ (resp.\ $t_2$). 
				Note that the \emph{witness} $w_\spadesuit$ is given $1$-type $\spadesuit$ and thus carries some unary violations.
			}
			\label{figure:small-interpretation-property}
		\end{figure*}
		
		Most of the work is to prove that the realized $1$-types are indeed the expected ones, as shown by the next lemma: 
		\begin{lemma}
			\label{lemma:types-are-as-expected}
			For every $\istyle{a} \in \criticals(\Imc)$, we have $\type_\Jmc(\istyle{a}) = \type_\Imc(\istyle{a})$.
			For every $w_t \in W$, we have $\type_\Jmc(w_t) = t$.
			For every $\istyle{a} \in \individuals(\abox) \setminus \criticals(\Imc)$, we have $\type_\Jmc(\istyle{a}) = \type_\Imc(p_{\type_\abox(\istyle{a})})$.
		\end{lemma}
		\begin{nestedproof}
			Consider $\istyle{a} \in \criticals(\Imc)$.
			Since, for every concept name $\cstyle{A}$ and role name $\rstyle{r}$, we have $\rho(\cstyle{A}^\Imc) \subseteq \cstyle{A}^\Jmc$ and $\rho(\rstyle{r}^\Imc) \subseteq \rstyle{r}^\Jmc$, and $\rho$ is the identity on $\criticals(\Imc)$, it is clear that $\type_\Imc(\istyle{a}) \subseteq \type_\Jmc(\istyle{a})$.
			If $\cstyle{A} \in \type_\Jmc(\istyle{a})$, it is also immediate that $\cstyle{A} \in \type_\Imc(\istyle{a})$.
			It remains to treat the case of $\exists \rstyle{r} \in \type_\Jmc(\istyle{a})$.
			Assume there exists $(\istyle{a}, d) \in \rstyle{r}^\Jmc$.
			We distinguish five cases based on the definition of $\rstyle{r}^\Jmc$.
			Case~4 is not applicable as $\istyle{a} \in \criticals(\Imc)$.
			In both Cases~1 and 5, we directly obtain a witness for $\istyle{a} \in (\exists \rstyle{r})^\Imc$, ensuring $\exists \rstyle{r} \in \type_\Imc(\istyle{a})$.
			We now turn to Case~2, that is $\rstyle{s}(\istyle{a}, d) \in \abox$ for some $\rstyle{s} \sqsubseteq_\tbox \rstyle{r}$.
			Note that $d \notin \criticals(\Imc)$, as otherwise we would have $(\istyle{a}, d) \in \criticals(\Imc)^2$ and Case~2 would not be applicable.
			If $\istyle{a} \leadsto_{1} d$, then from $d \notin \criticals(\Imc)$ we obtain that there are no binary violations on $(\istyle{a}, d)$, thus $(\istyle{a}, d) \in \rstyle{s}^\Imc \subseteq \rstyle{\Imc}$ and we are done.
			Otherwise $\istyle{a} \leadsto_{1} d$ is false, and in particular $\existsmany \rstyle{s} \in \type_\abox(\istyle{a})$.
			Therefore, $\istyle{a} \in (\exists \rstyle{r})^\Imc$ by Lemma~\ref{lemma:many-succ-in-abox-mean-at-least-one-in-an-interpretation}.		
			Case~3 is treated as Case~2.
			
			We now turn to the case of $w_t \in W$.
			Since, for every concept name $\cstyle{A}$ and role name $\rstyle{r}$, we have $\rho(\cstyle{A}^\Imc) \subseteq \cstyle{A}^\Jmc$ and $\rho(\rstyle{r}^\Imc) \subseteq \rstyle{r}^\Jmc$, and that there exists a least one element $d \in \Delta^\Imc \setminus \criticals(\Imc)$ such that $\type_\Imc(d) = t$, it is clear that $t \subseteq \type_\Jmc(w_t)$.
			If $\cstyle{A} \in \type_\Jmc(w_t)$, it is also immediate that $\cstyle{A} \in t$.
			It remains to treat the case of $\exists \rstyle{r} \in \type_\Jmc(w_t)$.
			Assume there exists $(w_t, d) \in \rstyle{r}^\Jmc$.
			We distinguish 5 cases based on the definition of $\rstyle{r}^\Jmc$.
			Cases~2, 3 and 4 are not applicable.
			Case~1 immediately provides a witness $d \in \Delta^\Imc$ with $\type_\Imc(d) = t$ and such that $d \in (\exists \rstyle{r})^\Imc$.
			We turn to Case~5, that is $d \in \individuals(\abox) \setminus \criticals(\Imc)$, and there exists $(e, p_{\type_\abox(d)}) \in \rstyle{r}^\Imc$ with $\rho(e) = w_t$.
			In particular, $\type_\Imc(e) = t$, thus $e \in (\exists \rstyle{r})^\Imc$, as witnesses by $p_{\type_\abox(d)}$, proves $\exists \rstyle{r} \in t$.

			The case of $\istyle{a} \in \individuals(\abox) \setminus \criticals(\Imc)$ is trivial, it suffices to recall that $\type_\abox(\istyle{a}) = \type_\abox(p_{\type_\abox(\istyle{a})})$ and that $p_{\type_\abox(\istyle{a})}$ is not involved in any violations.
		\end{nestedproof}
		
		Points~1 and 2 are trivially satisfied, by construction of~$\Jmc$.
		For Point~3, we prove a slightly stronger statement, namely that 
		the cost of $\Jmc$ equals the cost of $\Jmc\vert_{\criticals(\Imc) \cup W}$.
		Since $\criticals(\Imc) \subseteq \core(\abox)$, Point~3 then follows.
		To prove the stronger statement, we need to guarantee that no element from $\individuals(\abox) \setminus \criticals(\Imc)$ is involved in any violations.
		By virtue of the interpretation of such elements being dictated by the interpretation of the chosen elements $p_t$, which are not involved in any violations in $\Imc$, the claim will follow.
		Let $\istyle{a} \in \individuals(\abox) \setminus \criticals(\Imc)$.
		For a concept assertion $\cstyle{A}(\istyle{a}) \in \abox$, recall that $\istyle{a}$ and $p_{\type_\abox(\istyle{a})}$ have the same ABox type, thus $\cstyle{A}(p_{\type_\abox(\istyle{a})}) \in \abox$.
		By definition of $p_{\type_\abox(\istyle{a})}$, we have $p_{\type_\abox(\istyle{a})} \in \cstyle{A}^\Imc$, thus, by definition of $\cstyle{A}^\Jmc$, we obtain $\istyle{a} \in \cstyle{A}^\Jmc$.
		For a concept inclusion, it simply follows from Lemma~\ref{lemma:types-are-as-expected} and the definition of $p_{\type_\abox(\istyle{a})}$ not being involved in any violation.
		For a role assertion $\rstyle{r}(\istyle{a}, \istyle{b}) \in \abox$, note that we can directly use the definition of $\rstyle{r}^\Jmc$ to obtain $(\istyle{a}, \istyle{b}) \in \rstyle{r}^\Jmc$.
		For a role inclusion $\rstyle{r} \sqsubseteq \rstyle{s}$, we treat the case of $(\istyle{a}, d) \in \rstyle{r}^\Jmc \setminus \rstyle{s}^\Jmc$, the case of $(d, \istyle{a}) \in \rstyle{r}^\Jmc \setminus \rstyle{s}^\Jmc$ being symmetric.
		By definition of $\rstyle{r}^\Jmc$, we have two options: either $d \in \individuals(\abox)$ with $\rstyle{t}(\istyle{a}, d) \in \abox$ for some $\rstyle{t} \sqsubseteq_\tbox \rstyle{r}$, or $d = \rho(e)$ and $(p_{\type_\abox(\istyle{a})}, e) \in \rstyle{r}^\Imc$.
		We start with $d \in \individuals(\abox)$ with $\rstyle{t}(\istyle{a}, d) \in \abox$ for some $\rstyle{t} \sqsubseteq_\tbox \rstyle{r}$.
		In that case, we have $\rstyle{t} \sqsubseteq_\tbox \rstyle{s}$ since $\rstyle{t} \sqsubseteq_\tbox \rstyle{r}$ and $\rstyle{r} \sqsubseteq \rstyle{s} \in \tbox$, thus the definition of $\rstyle{s}^\Jmc$ directly gives $(\istyle{a}, d) \in \rstyle{s}^\Jmc$, which contradicts our assumption that $(\istyle{a}, d) \in \rstyle{r}^\Jmc \setminus \rstyle{s}^\Jmc$.
		It remains to turn to the case of $d = \rho(e)$ and $(p_{\type_\abox(\istyle{a})}, e) \in \rstyle{r}^\Imc$.
		By definition of $p_{\type_\abox(\istyle{a})}$ not being involved in any violation, we have $(p_{\type_\abox(\istyle{a})}, e) \in \rstyle{s}^\Imc$.
		Then the definition of $\rstyle{s}^\Jmc$ ensures $(\istyle{a}, \rho(e)) \in \rstyle{s}^\Jmc$, which again contradicts our assumption that $(\istyle{a}, d) \in \rstyle{r}^\Jmc \setminus \rstyle{s}^\Jmc$.
		
		For Point~4, with the above observation that $\Jmc$ equals the cost of $\Jmc\vert_{\criticals(\Imc) \cup W}$, it thus suffices to prove that the cost of $\Jmc\vert_{\criticals(\Imc) \cup W}$ is $\leq k$.
		To achieve this, we associate each violation in $\Jmc$ to a corresponding violation in $\Imc$ in an injective manner.
		Unary violations in $\Jmc\vert_{\criticals(\Imc)}$ are identical to those in $\Imc$, by Lemma~\ref{lemma:types-are-as-expected}.
		A unary violation on an element $w_t \in W$ can only be a concept inclusion violation, and thus is present on any element $d$ from $\Imc$ such that $\type_\Imc(d) = t$.
		Binary violations in $\Jmc\vert_{\criticals(\Imc)} \times \Jmc\vert_{\criticals(\Imc)}$ are identical to those in $\Imc$, by construction of $\Jmc$.
		A binary violation on $(w_{t_1}, w_{t_2}) \in W \times W$ can only be a role inclusion $\rstyle{r} \sqsubseteq \rstyle{s}$ violation.
		By definition of $\rstyle{r}^\Jmc$, there exists $(d_1, d_2) \in \rstyle{r}^\Imc$ with $d_1, d_2 \notin \criticals(\Imc)$, $\type_\Imc(d_1) = t_1$ and $\type_\Imc(d_2) = t_2$.
		It follows that $(d_1, d_2) \notin \rstyle{s}^\Imc$, which provides the desired witness for the role inclusion being violated already in $\Imc$.
		We now turn to the case of a binary violation on $(\istyle{a}, w_{t}) \in \criticals(\Imc) \times W$.
		First notice that it can again only be a violation of a role inclusion $\rstyle{r} \sqsubseteq \rstyle{s}$.
		By construction of $\rstyle{r}^\Jmc$, there exists $d \in \Delta^\Imc \setminus \criticals(\Imc)$ such that $(\istyle{a}, d) \in \rstyle{r}^\Imc$ and $\type_\Imc(d) = t$.
		It follows that $(\istyle{a}, d) \notin \rstyle{s}^\Imc$, which provides the desired witness for the role inclusion being violated already in $\Imc$.
		The case of a binary violation in $W \times \criticals(\Imc)$ is symmetric.
		
		Point~5$_p$ is immediate from $\rho(\Imc) \subseteq \Jmc$ and the fact that $\rho$ preserves the interpretation of all individuals occurring in $q$ as $\individuals(q) \subseteq \precore(\abox) \subseteq \criticals(\Imc)$.
		Point~5$_c$ is a direct consequence of Lemma~\ref{lemma:types-are-as-expected} and of $\individuals(q) \subseteq \precore(\abox)$.
	\end{proof}
	
	\subsection{
	Construction of the $\fo$-Rewriting}
	\label{subsection:rewriting}
	\begin{toappendix}
	\subsection*{Proofs for Section~\ref{subsection:rewriting} (Construction of the $\fo$-Rewriting)}
	\end{toappendix}
	
	Consider a $\dlliteboolh$ TBox $\tbox$, a weight function $\omega_\tbox$ for $\tbox$, fixed cost $k$, and BCQ $q$.
	We now proceed to the actual construction of the rewritten query $q'$.
	Notice that the pre-core always has size at most $P_{0} := 2k \times 2^{\sizeof{\NC(\tbox)}+4\sizeof{\NR(\tbox)}}  + \sizeof{q}$ (at most $2k$ copies of each rare ABox type, plus the individuals from the IQ).
	For each individual $\istyle{a}$, there exist at most $2k\sizeof{\NR(\tbox)}$ distinct individuals $\istyle{b}$ such that $\istyle{a} \leadsto_1 \istyle{b}$.
	Therefore, the size of every core, regardless of the specific ABox, is bounded by $P_0 \times (2k\sizeof{\NR(\tbox)})^k$.
	We let $M_{\tbox, q, k}$ be the above quantity plus the number of possible $1$-types, that is:
	\[ 
	M_{\tbox, q, k} := P_0 \times (2k\sizeof{\NR(\tbox)})^k + 2^{\sizeof{\NC(\tbox)}+2\sizeof{\NR(\tbox)}} 
	\]
	The number $M_{\tbox, q, k}$ gives an upper bound on the maximal size of the domain $\core(\abox) \cup W$ involved in the construction of interpretation $\Jmc$ in Lemma~\ref{lemma:main-lemma-for-dllite}.
	Note that Points~3, 5$_p$ and 5$_c$ also guarantee that 
	$\core(\abox) \cup W$ contains all the relevant information regarding violations and query satisfaction.
	
	The rewritten query $q'$ 
	considers each relevant interpretation $\Mmc$, up to isomorphism, whose domain $\Delta^\Mmc$ has size at most $M_{\tbox, q, k}$ and tries to match the `individual part' of $\Mmc$ (that is, the $\core(\abox)$ part, as opposed to the $W$ part) in the input $\abox_{\omega_{\!\abox}}$. 
	The rewriting must also check that each remaining individual from $\abox_{\omega_{\!\abox}}$ can be interpreted in a manner that is compatible w.r.t.\ the considered $\Mmc$, in the sense that it shall not introduce any violations (nor satisfy the query, in the case of the certain semantics).
	To achieve this, we specify along with $\Mmc$ its intended individual part as a subdomain $\Gamma \subseteq \Delta^\Mmc$ and the allowed ABox violations as a weighted ABox $\Vmc_\nu$.
	We call such a triple $(\Mmc, \Gamma, \Vmc_\nu)$ a \emph{strategy} for $(\tbox_{\omega_\tbox}, q, k)$ 
	if $\individuals(\Vmc) \subseteq \Gamma$ and $\omega_\tbox(\Mmc) + \sum_{\alpha \in \Vmc} \nu(\alpha) \leq k$.
	
	We differentiate between p-strategies, used for CQA$_p^k$, and c-strategies, used for IQA$_c^k$.
	A \emph{p-strategy} 
	is a strategy $(\Mmc, \Gamma, \Vmc_\nu)$ that additionally satisfies $\Mmc \models q$. 
	An ABox type $t$ is p-safe for a p-strategy $\sigma=(\Mmc, \Gamma, \Vmc_\nu)$ if there exists a $1$-type $t'$ such that:
	\begin{enumerate}
		\item for every $\cstyle{A} \in \NC(\tbox)$, if $\cstyle{A} \in t$, then $\cstyle{A} \in t'$;
		\item for every $\rstyle{r} \in \NRpm(\tbox)$, if $\exists \rstyle{r} \in t$, then $\exists \rstyle{r} \in t'$;
		\item $t'$ does not violate any CIs from $\tbox$;
		\item for every $\rstyle{r} \in \NRpm(\tbox)$, if $\exists \rstyle{r} \in t'$, then $\exists \rstyle{s} \in t'$ for every $\rstyle{s} \in \NRpm(\tbox)$ with $\rstyle{r} \sqsubseteq_\tbox \rstyle{s}$;
		\item for every $\rstyle{r} \in \NRpm(\tbox)$, if $\exists \rstyle{r} \in t'$, then there exists $d \in \Delta^\Mmc$ such that for every $\rstyle{s} \in \NRpm(\tbox)$ with $\rstyle{r} \sqsubseteq_\tbox \rstyle{s}$, we have $\exists \rstyle{s}^- \in \type_\Mmc(d)$.
	\end{enumerate}	
	
\noindent Recall that we need only to define c-strategies for the case where $q$ is an IQ. 
	We may assume w.l.o.g.\ that $q$ is a concept IQ\footnote{IQs of the form $\exists y\ \rstyle{r}(\istyle{a}, y)$ (resp.\ $\exists xy\ \rstyle{r}(x, y)$) can be handled by adding infinite-weight CIs  
		$\exists \rstyle{r} \sqsubseteq \cstyle{B}, \cstyle{B} \sqsubseteq \exists \rstyle{r}$ for a fresh concept name $\cstyle{B}$, 
		and using the IQ 
		$\cstyle{B}(\istyle{a})$ (resp.\ $\exists x. \cstyle{B}(x)$) instead. 
	}.
	A \emph{c-strategy} for $(\tbox_{\omega_\tbox}, q, k)$ is a strategy $\sigma = (\Mmc, \Gamma, \Vmc_\nu)$ such that 
	$\individuals(q) \subseteq \Gamma$ and $\Mmc \not\models q$, and an
	ABox type $t$ is c-safe for $\sigma$ if there exists a $1$-type $t'$ that satisfies 
	Conditions~1--5 of p-safe types, plus the following 
	condition:
	\begin{enumerate}
		\item[6.] if $q = \exists y\ \cstyle{A}(y)$, then $\cstyle{A} \notin t'$.
	\end{enumerate}
	
	The rewritten FO-query $q'$ is now obtained as the disjunction of subqueries $q_\sigma$, where each $\sigma$ is a p-strategy (resp.\ c-strategy) for $(\tbox_{\omega_\tbox}, q, k)$.
	Now, for a given $\sigma := (\Mmc, \Gamma, \Vmc_\nu)$, the subquery $q_\sigma$ uses one existentially quantified variable $v_d$ for each $d \in \Gamma$, and attempts to identify a subpart of the input weighted ABox $\abox_{\omega_\abox}$ that could be interpreted as $\Mmc\vert_{\Gamma}$, up to the violations described in $\Vmc_\nu$.
	For example, if $d \notin \cstyle{A}^\Mmc$ and $\cstyle{A}(d) \in \Vmc$, then $q_\sigma$ contains the following subquery $\cstyle{A}(v_d) \rightarrow \cstyle{W}_\cstyle{A}^{\nu(\cstyle{A}(d))}(v_d)$, enforcing that variable $v_d$ can only be mapped on an individual $\istyle{a}$ such that $\cstyle{A}(\istyle{a}) \in \abox$ if ${\omega_\abox}$ and $\nu$ agree on the cost of the assertion $\cstyle{A}(\istyle{a})$.
	Using a universally quantified variable, $q_\sigma$ also makes sure that all other individuals in $\abox$ have an ABox type that is p-safe (resp. c-safe) w.r.t.~$\sigma$.
	It is indeed clear, examining the very local requirements defining safe types that `having a safe p- (or c-) type' can be verified using an appropriate FO-subquery.

	In this manner, we can construct FO-queries $q'_p$ and $q'_c$, respectively based on p- and c-strategies, such that the following properties hold, concluding the proof of Theorem~\ref{theorem:upper-bound-dlliteboolh-cq-fixed}:
	\begin{toappendix}
	In what follows, we first clarify the full construction of the rewritten query, and then prove the correctness of the rewriting, stated in Lemma~\ref{lemma:rewriting}, to conclude the proof of Theorem~\ref{theorem:upper-bound-dlliteboolh-cq-fixed}.
	
	For a given p-strategy (resp.\ c-strategy) $\sigma := (\Mmc, \Gamma, \Vmc_\nu)$, we denote by $\Smc_\sigma$ the set of \emph{p-safe} ABox types for $\sigma$.
	As already outlined in the main body of the paper, the rewritten FO-query $q'$ is obtained as the disjunction of queries $q_\sigma$, where each $\sigma$ is a p-strategy (resp.\ c-strategy) for $(\tbox_{\omega_\tbox}, q, k)$.
	Now, for a given $\sigma := (\Mmc, \Gamma, \Vmc_\nu)$, the subquery $q_\sigma$ uses one variable $v_d$ for each $d \in \Gamma$; we use $\tuplev$ to refer to the tuple formed by all these variables and $V_\Gamma$ to refer to the set $\{ v_d \mid d \in \Gamma \}$.
	Recall that if the query $q_\sigma$ is to be satisfied, it must identify a subpart of the input ABox that could be interpreted as $\Mmc\vert_{\Gamma}$, up to some violations described in $\Vmc_\nu$.
	We describe the different subqueries which are used to define 
	$q_\sigma$. 
	First, we must verify that, in this pinpointed part of the ABox, every concept assertion that is unsatisfied in $\Mmc$ is a violation from $\Vmc$ with the expected cost.
	For a given concept name $\cstyle{A} \in \NC(\tbox)$ and $d \in \Gamma$, we let $q_{\sigma, \cstyle{A}, d}(v_d)$ be the following formula:
	\[
	\cstyle{A}(v_d) \land d \notin \cstyle{A}^\Mmc \rightarrow \cstyle{A}(d) \in \Vmc \land \cstyle{W}_\cstyle{A}^{\nu(\cstyle{A}(d))}(v_d),
	\]
	where $d \notin \cstyle{A}^\Mmc$ and $\cstyle{A}(d) \in \Vmc$ are Boolean expressions that are replaced by truth constants $\top$ or $\bot$ 
	in the above depending on $\sigma$.
	We shall use similar expressions in later subqueries without further mention. 
	We proceed similarly for role assertions: for a given role name $\rstyle{r} \in \NR(\tbox)$ and $d, e \in \Gamma$, we let $q_{\sigma, \rstyle{r}, d, e}(v_d, v_e)$ be the following formula:
	\[
	\rstyle{r}(v_d, v_e) \land (d, e) \notin \rstyle{r}^\Mmc \rightarrow \rstyle{r}(d, e) \in \Vmc \land \rstyle{w}_\rstyle{r}^{\nu(\rstyle{r}(d, e))}(v_d, v_e).
	\]
	We also make sure that the other individuals from the input ABox have an ABox type that is p-safe (resp. c-safe) w.r.t.~$\sigma$.
	It is clear that `variable $v$ maps onto an individual with ABox type $t$' can be captured as follows: 
	\begin{align*}
		q_{t}(v) := 
		& \bigwedge_{\cstyle{A} \in \NC(\tbox)} \cstyle{A}(v) \leftrightarrow \cstyle{A} \in t 
		 \quad \land \quad \bigwedge_{\rstyle{r} \in \NRpm(\tbox)} (\exists v'\ \rstyle{r}(v, v')) \leftrightarrow \exists \rstyle{r} \in t
		 \quad \land \quad \bigwedge_{\rstyle{r} \in \NRpm(\tbox)} (\exists_{> k} v'\ \rstyle{r}(v, v')) \leftrightarrow \existsmany \rstyle{r} \in t,
	\end{align*}
	where $\exists_{> k} v\ \phi(v)$ is the usual counting quantifier.
	It remains to check whether the minimal completion of each role assertion involving an individual in $\Mmc$ and one not involved in $\Mmc$ is consistent w.r.t.\ $\Mmc$.
	For $\rstyle{r} \in \NRpm(\tbox)$ and $d \in \Delta^\Mmc$, we define $q_{\sigma, \rstyle{r}, d}(\tuplev)$ as the following formula:
	\[
	\forall u \big( \rstyle{r}(v_d, u) \land \bigwedge_{e \in \Gamma} u \neq v_e \rightarrow \bigwedge_{\rstyle{r} \sqsubseteq_\tbox \rstyle{s}} \exists \rstyle{s} \in \type_\Mmc(d) )
	\]
	We now combine these subqueries to obtain 
	the query $q_\sigma$:
	\[
	q_\sigma :=~ \exists \tuplev
	\bigg(
	\bigwedge_{\substack{d, e \in \Gamma \\ d \neq e}} v_d \neq v_e
	~\land~ \bigwedge_{\substack{\cstyle{A} \in \NC \\ d \in \Gamma}} q_{\sigma, \cstyle{A}, d}(v_d)
	~\land~ \bigwedge_{\substack{\rstyle{r} \in \NR \\ d, e \in \Gamma}} q_{\sigma, \rstyle{r}, d, e}(v_d, v_e)
	~\land~ \forall v\ \bigg( \bigwedge_{d \in \Gamma} v \neq v_d \rightarrow \bigvee_{t \in \Smc_\sigma} q_t(v) \bigg)
	~\land~ \bigwedge_{\substack{\rstyle{r} \in \NR \\ d \in \Gamma}} q_{\sigma, \rstyle{r}, d}(\tuplev)
	\bigg)
	\]
	Notice the 
	first conjunction of inequalities simply enforces the SNA on ABox individuals.
	Now, as announced, we define:
	\[ q'_p := \bigvee_{\substack{\sigma \text{ is a p-strategy}
			\\
			\text{for } (\tbox, \omega_\tbox, k, q)}} q_\sigma
	\qquad 
	q'_c := \bigvee_{\substack{\sigma \text{ is a c-strategy}
			\\
			\text{for } (\tbox, \omega_\tbox, k, q)}} q_\sigma.
	\]


\end{toappendix}
	
	\begin{lemmarep}
		\label{lemma:rewriting}
		Let $\tbox_{\omega_\tbox}$ be a weighted $\dlliteboolh$ TBox, $k$ an integer, and $q$ a BCQ.
		For every weighted ABox $\abox_{\omega_{\!\abox}}$:
		\[
		(\tbox, \abox)_{\omega_\tbox \cup \omega_{\!\abox}} \models^k_p q \text{~ iff ~} \Imc_{\abox_{\omega_{\!\abox}}}\models q'_p.
		\]
		Furthermore, if $q$ is an IQ, then:
		\[
		(\tbox, \abox)_{\omega_\tbox \cup \omega_{\!\abox}} \not\models^k_c q \text{~ iff ~} \Imc_{\abox_{\omega_{\!\abox}}} \models q'_c.
		\]
	\end{lemmarep}
	
	\begin{proof}
		Most of the proof is common to both $q'_p$ and $q'_c$, typically when it comes down to building an interpretation with cost $\leq k$.
		We only differentiate the two cases when it is about verifying query (non-)entailment.
		We voluntarily write $q'$ indifferently, and consider strategies without specifying whether they are p-strategies or c-strategies.\medskip
		
		\noindent$(\Leftarrow)$. 
		Assume $\abox^{\omega_{\!\abox}} \models q'$. 
		Therefore, there exists a strategy $\sigma := (\Mmc, \Gamma, \Vmc_\nu)$ such that $\abox^{\omega_{\!\abox}} \models q_\sigma$.
		We denote $W := \Delta^\Mmc \setminus \Gamma$.
		By definition of $\abox^{\omega_{\!\abox}} \models q_\sigma$, there exists a mapping $h_0 : V_\Gamma \mapsto \individuals(\abox)$ that makes the rest of $q_\sigma$ true on $\abox^{\omega_{\!\abox}}$.
		In particular, from the first conjunct in $q_\sigma$, $h_0$ is injective.
		We set:
		\[
		\begin{array}{rcl}
			h : \Delta^\Mmc & \rightarrow & h_0(V_\Gamma) \cup W
			\\
			d & \mapsto & \left\{ \begin{array}{ll}
				h_0(v_d) & \text{if } d \in \Gamma
				\\
				d & \text{otherwise}
			\end{array}\right.
		\end{array}
		\]
		Notice that $h$ is injective, since $h_0$ is, and is thus a bijection.
		For each $t$ in $\Smc_\sigma$, we choose a $1$-type $t'$ whose existence is guaranteed by the definition of $t$ being safe, and we denote $\rho : t \mapsto t'$ this choice function over $\Smc_\sigma$.
		By Point~5 of the definition of safe types, if $\exists \rstyle{r} \in \bigcup_{t \in \Smc_\sigma} \rho(t)$, then there exists $d \in \Delta^\Mmc$ such that $d \in \exists \rstyle{s}^-$ for every $\rstyle{s} \in \NRpm$ such that $\rstyle{r} \sqsubseteq_\tbox \rstyle{s}$.
		We choose one such element $d_\rstyle{r}$ per $\exists \rstyle{r} \in \bigcup_{t \in \Smc_\sigma} \rho(t)$ and denote $w_\rstyle{r} := h(d_\rstyle{r})$.
		
		We now set $\Imc$ as the interpretation with domain $\individuals(\abox) \cup W$ and that interprets every concept name $\cstyle{A}$ and role name $\rstyle{r}$ as:
		\begin{align*}
			\cstyle{A}^\Imc := ~ &
			h(\cstyle{A}^\Mmc) \cup \{ \istyle{a} \in \individuals(\abox) \setminus h_0(V_\Gamma) \mid \cstyle{A} \in \rho(\type_\abox(\istyle{a})) \}
			\\
			\rstyle{r}^\Imc := ~ &
			h(\rstyle{r}^\Mmc) 
			\\
			&
			\cup \left\lbrace (\istyle{a}, \istyle{b}) \in \individuals(\abox)^2 \setminus h_0(V_\Gamma)^2 \middle| 
			\begin{array}{l} 
				\rstyle{s}(\istyle{a}, \istyle{b}) \in \abox \text{ for}
				\\
				\text{some } \rstyle{s} \sqsubseteq_\tbox \rstyle{r}
			\end{array}
			\right\rbrace
			\\
			&
			\cup \left\lbrace (\istyle{b}, \istyle{a}) \in \individuals(\abox)^2 \setminus h_0(V_\Gamma)^2 \middle| 
			\begin{array}{l} 
				\rstyle{s}(\istyle{a}, \istyle{b}) \in \abox \text{ for}
				\\
				\text{some } \rstyle{s} \sqsubseteq_\tbox \rstyle{r}^-
			\end{array}
			\right\rbrace
			\\
			&
			\cup \{ (\istyle{a}, w_\rstyle{s}) \mid \istyle{a} \in \individuals(\abox) \setminus h_0(V_\Gamma), \exists \rstyle{s} \in \rho(\type_\abox(\istyle{a})), \rstyle{s} \sqsubseteq_\tbox \rstyle{r} \}
			\\
			&
			\cup \{ (w_\rstyle{s}, \istyle{a}) \mid \istyle{a} \in \individuals(\abox) \setminus h_0(V_\Gamma), \exists \rstyle{s} \in \rho(\type_\abox(\istyle{a})), \rstyle{s} \sqsubseteq_\tbox \rstyle{r}^- \}
		\end{align*}
		Notice that the two last lines are well defined as $q_\sigma$ being satisfied guarantees that every $\istyle{a} \in \individuals(\abox) \setminus h_0(V_\Gamma)$ has a safe ABox type in $\abox$, thus $\rho$ is indeed defined on $\type_\abox(\istyle{a})$.
		
		Before verifying that $\Imc$ has the desired properties, we first establish the following lemma:
		\begin{lemma}
			\label{lemma:types-are-again-as-expected}
			For every $d \in \Delta^\Mmc$, we have $\type_\Imc(h(d)) = \type_\Mmc(d)$.
			For every $\istyle{a} \in \individuals(\abox) \setminus h_0(V_\Gamma)$, we have $\type_\Imc(\istyle{a}) = \rho(\type_\abox(\istyle{a}))$.
		\end{lemma}
		\begin{nestedproof}
			Let $d \in \Delta^\Mmc$.
			The inclusion $\type_\Mmc(d) \subseteq \type_\Imc(h(d))$ is trivial since $h(\Mmc) \subseteq \Imc$.
			We turn to the inclusion $\type_\Imc(h(d)) \subseteq \type_\Mmc(d)$.
			Consider a concept name $\cstyle{A} \in \type_\Imc(h(d))$.
			If $h(d) \in h_0(V_\Gamma)$, then by definition of $\cstyle{A}^\Imc$, we have some $e \in \Delta^\Mmc$ such that $h(e) = h(d)$ and $e \in \cstyle{A}^\Mmc$.
			Recall that $h$ is injective, thus $d = e$.
			It follows that $\cstyle{A} \in \type_\Mmc(d)$.
			Consider a role $\rstyle{r} \in \NRpm(\tbox)$ such that $\exists \rstyle{r} \in \type_\Imc(h(d))$, that is we have some $e \in \Delta^\Imc$ such that $(h(d), e) \in \rstyle{r}^\Imc$.
			We distinguish based on the definition of $\rstyle{r}^\Imc$.
			\begin{itemize}
				\item
				If $(h(d), e) \in h(\rstyle{r}^\Mmc)$, then we are done.
				\item 
				If $(h(d), e) \in \individuals(\abox)^2 \setminus h_0(V_\Gamma)^2$ and $\rstyle{s}(h(d), e) \in \abox$ for some $\rstyle{s} \sqsubseteq_\tbox \rstyle{r}$, then $h(d) \in {h_0}(V_\Gamma)$ and $e \in \individuals(\abox) \setminus h_0(V_\Gamma)$.
				Thus from $q_{\sigma, \rstyle{s}, d}$ being true for $u := e$, we obtain $\exists \rstyle{r} \in \type_\Mmc(d)$ as desired.
				\item 
				If $(e, h(d)) \in \individuals(\abox)^2 \setminus h_0(V_\Gamma)^2$ and $\rstyle{s}(h(d), e) \in \abox$ for some $\rstyle{s} \sqsubseteq_\tbox \rstyle{r}^-$, then $h(d) \in {h_0}(V_\Gamma)$ and $e \in \individuals(\abox) \setminus h_0(V_\Gamma)$.
				Thus from $q_{\sigma, \rstyle{s}^-, d}$ being true for $u := e$, we again obtain $\exists \rstyle{r} \in \type_\Mmc(d)$ as desired.
				\item 
				Note that $h(d) \in \individuals(\abox) \setminus h_0(V_\Gamma)$ is impossible.
				\item 
				Otherwise $h(d) = w_\rstyle{s}$ for some $\rstyle{s} \sqsubseteq_\tbox \rstyle{r}^-$ such that there is $\istyle{a} \in \individuals(\abox) \setminus h_0(V_\Gamma)$ with $\exists \rstyle{s} \in \rho(\type_\abox(\istyle{a}))$.
				From the definition of $w_\rstyle{s}$ we obtain $\exists \rstyle{t}^- \in \type_\Mmc(d)$ for every $\rstyle{s} \sqsubseteq_\tbox \rstyle{t}$.
				In particular for $\rstyle{t} := \rstyle{r}^-$, we get $\exists \rstyle{r} \in \type_\Mmc(d)$.
			\end{itemize}
			In all the above cases, we obtained $\exists \rstyle{r} \in \type_\Mmc(d)$, thus $\type_\Imc(h(d)) \subseteq \type_\Mmc(d)$ which concludes this part of the proof.
			
			Let $\istyle{a} \in \individuals(\abox) \setminus h_0(V_\Gamma)$.
			The inclusion $\rho(\type_\abox(\istyle{a})) \subseteq \type_\Imc(\istyle{a})$ is trivial by construction of $\Imc$ (notice that, for the case of $\exists \rstyle{r} \in \rho(\type_\abox(\istyle{a}))$, the interpretation provides the $\rstyle{r}$-edge $(\istyle{a}, w_\rstyle{r})$).
			We turn to the inclusion $\type_\Imc(\istyle{a}) \subseteq \rho(\type_\abox(\istyle{a}))$.
			Consider a concept name $\cstyle{A} \in \type_\Imc(\istyle{a})$.
			Then by definition of $\cstyle{A}^\Imc$, we have $\cstyle{A} \in \rho(\type_\abox(\istyle{a}))$ as desired.
			Consider a role $\rstyle{r} \in \NRpm(\tbox)$ such that $\exists \rstyle{r} \in \type_\Imc(\istyle{a})$, that is, we have some $e \in \Delta^\Imc$ such that $(\istyle{a}, e) \in \rstyle{r}^\Imc$.
			We distinguish based on the definition of $\rstyle{r}^\Imc$.
			\begin{itemize}
				\item
				From $\istyle{a} \notin {h_0}(V_\Gamma)$, we cannot have $(\istyle{a}, e) \in h(\rstyle{\Mmc})$. 
				\item 
				If $(\istyle{a} , e) \in \individuals(\abox)^2 \setminus h_0(V_\Gamma)^2$ and $\rstyle{s}(\istyle{a} , e) \in \abox$ for some $\rstyle{s} \sqsubseteq_\tbox \rstyle{r}$, then $\exists \rstyle{s} \in \type_\abox(\istyle{a})$.
				Since $\type_\abox(\istyle{a})$ is a safe type (from $q_\sigma$ being satisfied), Points~2 and 4 in the definition of safe-types, guarantees that $\exists \rstyle{r} \in \rho(\type_\abox(\istyle{a}))$ as desired.
				\item 
				If $(e, \istyle{a}) \in \individuals(\abox)^2 \setminus h_0(V_\Gamma)^2$ and $\rstyle{s}(e, \istyle{a}) \in \abox$ for some $\rstyle{s} \sqsubseteq_\tbox \rstyle{r}^-$, then $\exists \rstyle{s}^- \in \type_\abox(\istyle{a})$.
				Since $\type_\abox(\istyle{a})$ is a safe type (from $q_\sigma$ being satisfied), Points~2 and 4 in the definition of safe-types, guarantees that $\exists \rstyle{r} \in \rho(\type_\abox(\istyle{a}))$ as desired.
				\item 
				If $e = w_\rstyle{s}$ for some $\rstyle{s} \sqsubseteq_\tbox \rstyle{r}$ such that $\exists \rstyle{s} \in \rho(\type_\abox(\istyle{a}))$, then Point~4 in the definition of safe types yields $\exists \rstyle{r} \in \rho(\type_\abox(\istyle{a}))$.
				\item 
				Note that $\istyle{a}$ cannot be a witness $w_\rstyle{s}$ as those are not individuals.
			\end{itemize}
			This concludes the proof of $\type_\Imc(\istyle{a}) \subseteq \rho(\type_\abox(\istyle{a}))$.
		\end{nestedproof}
		
		We now verify that $\Imc$ is an interpretation with cost $\leq k$.
		First consider a concept assertion $\cstyle{C}(\istyle{a}) \in \abox$ violated in $\Imc$.
		If $\istyle{a} \in \individuals(\abox) \setminus h_0(V_\Gamma)$, then $\type_\abox(\istyle{a})$ is a safe type, from $q_\sigma$ being satisfied, and thus $\istyle{a} \in \cstyle{A}^\Imc$ by definition of $\cstyle{A}^\Imc$ and Point~1 in the definition of safe types; this contradicts $\cstyle{C}(\istyle{a})$ being violated in $\Imc$.
		Otherwise $\istyle{a} \in h_0(V_\Gamma)$, so there exists $d \in \Gamma$ such that $h_0(v_d) = \istyle{a}$.
		From $\cstyle{A}(\istyle{a})$ being violated, we have $d \notin \cstyle{A}^\Mmc$, thus from $q_{\sigma, \cstyle{A}, d}$ being satisfied we obtain $\cstyle{A}(\istyle{d}) \in \Vmc$ and $\cstyle{W}_\cstyle{A}^{\nu(\cstyle{A}(d))}(\istyle{a})$, thus the weight of $\cstyle{A}(\istyle{a})$ is accounted for in $\Vmc_\nu$.
		
		We proceed similarly for a role assertion $\rstyle{r}(\istyle{a}, \istyle{b}) \in \abox$.
		It is satisfied if $(\istyle{a}, \istyle{b}) \in \individuals(\abox)^2 \setminus h_0(V_\Gamma)^2$ by definition of $\rstyle{r}^\Imc$.
		If $(\istyle{a}, \istyle{b}) \in h_0(V_\Gamma)^2$, then the cost is accounted for in $\Vmc_\nu$, this time using the corresponding subquery $q_{\sigma, \rstyle{r}, h^{-1}(\istyle{a}), h^{-1}(\istyle{b})}$.
		Therefore, the cost of ABox violations in $\Imc$ is bounded by $\sum_{\alpha \in \Vmc} \nu(\alpha)$.
		
		The case of concept inclusions is rather straightforward with Lemma~\ref{lemma:types-are-again-as-expected} at hand.
		If a violation of $\cstyle{C} \sqsubseteq \cstyle{D}$ occurs at $d \in \Delta^\Imc$, then we distinguish two cases.
		If $d \in h(\Delta^\Mmc)$, then Lemma~\ref{lemma:types-are-again-as-expected} guarantees that $h^{-1}(d)$ also violates $\cstyle{C} \sqsubseteq \cstyle{D}$ in $\Mmc$, thus the cost of this violation is accounted for in the TBox violations of $\Mmc$.
		If $d \notin h(\Delta^\Mmc)$, then Lemma~\ref{lemma:types-are-again-as-expected} yields $\type_\Imc(d) = \rho(\type_\abox(d))$.
		The latter, using Point~3 from the definition of safe types, does not violate any CIs, which contradicts $d$ violating $\cstyle{C} \sqsubseteq \cstyle{D}$.
		
		For a role inclusion $\rstyle{r} \sqsubseteq \rstyle{s}$, we trivially have $\rstyle{r} \sqsubseteq_\tbox \rstyle{s}$.
		Therefore, in the definition of $\rstyle{r}^\Imc$, it is clear that only the case of $h(\rstyle{r}^\Mmc)$ may be violated (that is, not-matched with a pair from $h(\rstyle{s}^\Mmc)$), as the remaining cases are immediately treated by $\sqsubseteq_\tbox$ being transitive.
		However, every violation of a role inclusion in $h(\Mmc)$ is a violation of the same role inclusion in $\Mmc$, thus its cost is already accounted for in the TBox violations of $\Mmc$.
		
		Overall, we obtained that $(\omega_\tbox \cup \omega_{\!\abox})(\Imc) \leq \omega_\tbox(\Mmc) + \sum_{\alpha \in \Vmc} \nu(\alpha)$.
		Recalling the definition of a strategy, we obtain $(\omega_\tbox \cup \omega_{\!\abox})(\Imc) \leq k$.
		It remains to verify that $\Imc$ behaves as expected w.r.t.\ the query $q$.
		\begin{itemize}
			\item
			If working with $q$ a CQ under the possible semantics, then the definition of p-strategies guarantees that $\Mmc \models q$.
			Therefore, from $h(\Mmc) \subseteq \Imc$ and $\individuals(q) \subseteq \Gamma \subseteq \Delta^\Mmc$, we obtain immediately that $\Imc \models q$ as desired.
			\item
			If working with $q$ an IQ under the certain semantics, then the definition of c-strategies guarantees that $\individuals(q) \subseteq \Gamma$ and $\Mmc \not\models q$.
			Recall that we restricted w.l.o.g.\ to the case of $q$ being a concept IQ, we thus distinguish between two possible cases.
			If $q = \cstyle{A}(\istyle{a})$, then Lemma~\ref{lemma:types-are-again-as-expected} joint with the above directly concludes.
			If $q = \exists y\ \cstyle{A}(y)$, then we use Lemma~\ref{lemma:types-are-again-as-expected} to obtain $\cstyle{A}^\Imc \cap {h(\Delta^\Mmc)} = \emptyset$, and use Point~6 in the definition of a c-safe type to also ensure $\cstyle{A}^\Imc \cap (\Delta^\Mmc \setminus h(\Delta^\Mmc)) = \emptyset$.
		\end{itemize}

		\noindent$(\Rightarrow)$. 
		Assume that there exists an interpretation $\Imc$ of $(\tbox, \abox)_{\omega_\tbox \cup \omega_{\!\abox}}$ whose cost is $\leq k$ and that either satisfies the CQ $q$ (for the possible semantics), or does {not} satisfy the IQ $q$ (for 
		the certain semantics).
		We apply Lemma~\ref{lemma:main-lemma-for-dllite} and denote $\Jmc$ the obtained interpretation whose domain is $\individuals(\abox) \cup W$ for some set of $1$-types $W$.
		Let $\Mmc := \Jmc\vert_{\core(\abox) \cup W}$, $\Gamma := \core(\abox)$, $\Vmc := \{ \alpha \in \abox \mid \Jmc \not\models \alpha \}$ and $\nu : \alpha \mapsto \omega_{\!\abox}(\alpha)$ the cost function over $\Vmc$.
		Denote $\sigma := (\Mmc, \Gamma, \Vmc_\nu)$ the corresponding strategy.
		It is direct to verify that $\sigma$ is indeed a p-strategy, resp.\ a c-strategy, depending on the considered case.
		We claim that $\abox^\omega \models q_\sigma$, and thus $\abox^\omega \models q'$.
		It is not hard to verify it, by considering the following mapping $h$ that defines how to assign the existentially quantified variables $\tuplev$ of $q_\sigma$:
		\[
		\begin{array}{rcl}
			h : V_\Gamma & \rightarrow & \Gamma
			\\
			v_d & \mapsto & d.
		\end{array}
		\]
		We highlight two key ingredients to verify the above claim:
		\begin{itemize}
			\item 
			To check that the part of $q_\sigma$ interested in safe types is satisfied, one needs to exhibit fitting 1-types $t'$ for an ABox type $t$ found on some individual $\istyle{a}$ that is not reached by the variables $\tuplev$.
			This is achieved by setting $t' := \type_\Imc(p_{\type_\abox(\istyle{a}})$.
			From Point~3 of Lemma~\ref{lemma:main-lemma-for-dllite}, guaranteeing that non-core elements are not involved in any violations, it is clear that $t'$ satisfies Points~1-4 from the definition of safe types.
			To check that $t'$ satisfies Point~5 follows from Lemma~\ref{lemma:types-are-as-expected} and the construction of $\Jmc$.
			Points~6 and 7 (in the case of c-safe types) both follow from Point~5$_c$ of Lemma~\ref{lemma:main-lemma-for-dllite}.
			\item
			Using Lemma~\ref{lemma:types-are-as-expected}, one can verify that such $t'$ can always be found already in $\Jmc$ .
			To check that each subquery $q_{\sigma, \rstyle{r}, d}$ is indeed satisfied, we rely again on Point~3 from Lemma~\ref{lemma:main-lemma-for-dllite} that guarantees non-core elements are not involved in any violations.\qedhere
		\end{itemize}
	\end{proof}

	\section{Conclusion}
	Our results significantly improve our understanding of the data complexity 
	of query entailment under recently introduced cost-based semantics. 
	In particular, we have proved a $\deltaptwo$ upper bound for the optimal-cost certain and possible semantics, yielding tight 
	complexity bounds for a wide range of lightweight and expressive DLs, up to $\mathcal{ALCHIO}$ and covering also prominent DL-Lite dialects. 
	Moreover, we obtained surprising tractability results, showing 
	that fixed-cost possible semantics (for CQs)
	and certain semantics (for IQs) in $\dlliteboolh$ enjoy the same low $\tczero$ complexity as classical CQ answering in DL-Lite. 
	We expect our upper bounds can be adapted to also handle negative role inclusions (to cover also $\dlliter$).
	For DLs with functionality or number restrictions, it does not suffice to work with finite interpretations, so wholly different methods are required. 
	Developing a practical implementation of the FO-rewritings for the identified tractable cases is another interesting direction.

	\section*{Acknowledgments}
	The authors acknowledge the financial support of the ANR AI Chair INTENDED (ANR-19-CHIA-0014)
and the Federal Ministry of Research, Technology and Space of Germany and by Sächsische Staatsministerium für Wissenschaft, Kultur und Tourismus in the programme Center of Excellence for AI-research ``Center for Scalable Data Analytics and Artificial Intelligence Dresden/Leipzig'', project identification number: ScaDS.AI.

	\bibliographystyle{aaai2026}
	\bibliography{references}
	

\end{document}